\documentclass{article}

\usepackage[final]{neurips_2023}

\usepackage[utf8]{inputenc} % allow utf-8 input
\usepackage[T1]{fontenc}    % use 8-bit T1 fonts
\usepackage{hyperref}       % hyperlinks
\usepackage{url}            % simple URL typesetting
\usepackage{booktabs}       % professional-quality tables
\usepackage{amsfonts}       % blackboard math symbols
\usepackage{nicefrac}       % compact symbols for 1/2, etc.
\usepackage{microtype}      % microtypography
\usepackage{xcolor}         % colors

\usepackage{amsfonts, amsmath, amssymb, amsthm}
\usepackage{xfrac}
\usepackage{graphicx, xcolor, colortbl}
\usepackage{footnote}
\usepackage{tablefootnote}
\usepackage{natbib}
\usepackage{dsfont}
\usepackage{bbm}
\usepackage{algorithm}
\usepackage{algorithmic}
\usepackage{import}
\usepackage{mathtools}
\usepackage{bm}
\usepackage{xspace}
\usepackage{thmtools,thm-restate} %to restate lemma
\usepackage{accents}
\usepackage{framed}
\usepackage[inline]{enumitem}
\usepackage{booktabs}
\usepackage{color, colortbl}
\usepackage{natbib}
\usepackage{multirow}
\usepackage{multicol}
\usepackage{wrapfig}

\usepackage{notation}
\usepackage{todonotes}

\newtheorem{lemma}{Lemma}
\newtheorem{theorem}{Theorem}
\newtheorem*{theorem*}{Theorem}
\newtheorem{corollary}{Corollary}

\theoremstyle{definition}

\newtheorem{proposition}{Proposition}
\newtheorem{definition}{Definition}
\newtheorem{assumption}{Assumption}
\newtheorem{remark}{Remark}

\usepackage{minitoc}

\title{Model-free Posterior Sampling via Learning Rate Randomization}

\author{%
Daniil Tiapkin \\
  CMAP, \' Ecole Polytechnique\\
  HSE University\\
\texttt{daniil.tiapkin@polytechnique.edu} \\
  \And
  Denis Belomestny \\
  Duisburg-Essen University\\
  HSE University \\
  \texttt{denis.belomestny@uni-due.de} \\
  \AND
  Daniele Calandriello \\
  Google DeepMind \\
\texttt{dcalandriello@google.com}
  \And
  \' Eric Moulines \\
  CMAP, \' Ecole Polytechnique \\
  Mohamed Bin Zayed University of AI, UAE \\
\texttt{eric.moulines@polytechnique.edu} \\
  \And
  Remi Munos \\
  Google DeepMind \\
  \texttt{munos@google.com} \\
  \And
  Alexey Naumov \\
  HSE University \\
  \texttt{anaumov@hse.ru} \\
  \AND 
  Pierre Perrault \\
    IDEMIA \\   \texttt{pierre.perrault@outlook.com} \\
  \And 
  Michal Valko \\
  Google DeepMind \\
  \texttt{valkom@google.com} \\
  \And 
  Pierre M\' enard \\
  ENS Lyon\\
  \texttt{pierre.menard@ens-lyon.fr} \\
}

\author{%
Daniil Tiapkin$^{1,2}$ \quad Denis Belomestny$^{3,2}$ \quad Daniele Calandriello$^4$ \quad \' Eric Moulines$^{1,5}$ \\ \textbf{Remi Munos}$^4$ \quad
\textbf{Alexey Naumov}$^2$ \quad \textbf{Pierre Perrault}$^6$ \quad \textbf{Michal Valko}$^4$
\quad \textbf{Pierre M\' enard}$^7$
\\
$^1$CMAP, École Polytechnique 
\quad $^2$HSE University 
\quad $^3$Duisburg-Essen University
\\ $^4$Google DeepMind
\quad $^5$Mohamed Bin Zayed University of AI, UAE
\quad $^6$IDEMIA
\quad $^7$ENS Lyon\\
\texttt{\{daniil.tiapkin,eric.moulines\}@polytechnique.edu}\\
\texttt{denis.belomestny@uni-due.de}\\
\texttt{\{dcalandriello,munos,valkom\}@google.com}\quad
\texttt{anaumov@hse.ru}
\\
\texttt{pierre.perrault@outlook.com}
\quad
\texttt{pierre.menard@ens-lyon.fr}
}

\begin{document}

\maketitle

\begin{abstract}
    In this paper, we introduce Randomized Q-learning (\RandQL), a novel randomized model-free algorithm for regret minimization in episodic Markov Decision Processes (MDPs). To the best of our knowledge, \RandQL is the first tractable model-free posterior sampling-based algorithm. 
    We analyze the performance of \RandQL in both tabular and non-tabular metric space settings. In tabular MDPs, \RandQL achieves a regret bound of order $\tcO(\sqrt{H^{5}SAT})$, where $H$ is the planning horizon, $S$ is the number of states, $A$ is the number of actions, and $T$ is the number of episodes. For a metric state-action space, \RandQL enjoys a regret bound of order $\tcO(H^{5/2} T^{(d_z+1)/(d_z+2)})$, where $d_z$ denotes the zooming dimension. Notably, \RandQL achieves optimistic exploration \textit{without using bonuses}, relying instead on a novel idea of \textit{learning rate randomization}. Our empirical study shows that \RandQL outperforms existing approaches on baseline exploration environments.
\end{abstract}

% For TOC in appendix (https://tex.stackexchange.com/a/419290)
\doparttoc % Tell to minitoc to generate a toc for the parts
\faketableofcontents % Run a fake tableofcontents command for the partocs

\section{Introduction}

In reinforcement learning (RL, \citealt{SuttonBarto98}), an agent learns to interact with an unknown environment by acting, observing the next state, and receiving a reward. The agent's goal is to maximize the sum of the collected rewards. To achieve this, the agent can choose to use model-based or model-free algorithms. In model-based algorithms, the agent builds a model of the environment by inferring the reward function and the transition kernel that produces the next state. The agent then plans in this model to find the optimal policy. In contrast, model-free algorithms directly learn the optimal policy, which is the mapping of a state to an optimal action, or equivalently, the optimal Q-values, which are the mapping of a state-action pair to the expected return of an optimal policy starting by taking the given action at the given state.
\par
Although empirical evidence suggests that model-based algorithms are more sample efficient than model-free algorithms \citep{deisenroth2011pilco, schulman2015trpo}; model-free approaches offer several advantages. These include smaller time and space complexity, the absence of a need to learn an explicit model, and often simpler algorithms. As a result, most of the recent breakthroughs in deep RL, such as those reported by \citet{mnih2013playing,schulman2015trpo,schulman2017ppo,haarnoja2018actorcritic}, have been based on model-free algorithms, with a few notable exceptions, such as 
\citet{schrittwieser2019mastering,hessel2021muesli}. Many of these model-free algorithms \citep{mnih2013playing,vanhasselt2015reinforcement, lillicrap2015continuous} are rooted in the well-known Q-learning algorithm of \citet{watkins1992q}. Q-learning is an off-policy learning technique where the agent follows a behavioral policy while simultaneously incrementally learning the optimal Q-values by combining asynchronous dynamic programming and stochastic approximation. Until recently, little was known about the sample complexity of Q-learning in the setting where the agent has no access to a simulator allowing to sample an arbitrary state-action pair. In this work, we consider such challenging setting where the environment is modelled by an episodic Markov Decision Process (MDP) of horizon $H$. After $T$ episodes, the performance of an agent is measured through regret which is the difference between the cumulative
reward the agent could have obtained by acting optimally and what the agent really obtained during the interaction with the MDP.
\par
This framework poses the famous exploration-exploitation dilemma where the agent must balance the need to try new state-action pairs to learn an optimal policy against exploiting the current observations to collect the rewards. One effective approach to resolving this dilemma is to adopt the principle of optimism in the face of uncertainty. In finite MDPs, this principle has been successfully implemented in the model-based algorithm using bonuses \citep{jaksch2010near,azar2017minimax,fruit2018efficient,dann2017unifying,zanette2019tighter}. Specifically, the upper confidence bounds (UCBs) on the optimal Q-value are built by adding bonuses and then used for planning. Building on this approach, \citet{jin2018is} proposed the \OptQL algorithm, which applies a similar bonus-based technique to Q-learning, achieving efficient exploration. Recently, \citet{zhang2020advantage} introduced a simple modification to \OptQL that achieves optimal sample complexity, making it competitive with model-based algorithms.
\par
Another class of methods for optimistic exploration is Bayesian-based approaches.  An iconic example among this class is the posterior sampling for reinforcement learning (\PSRL,\citealt{strens2000bayesian,osband2013more}) algorithm. This model-based algorithm maintains a \textit{surrogate Bayesian model} of the MDP, for instance, a Dirichlet posterior on the transition
probability distribution if the rewards are known. At each episode, a new MDP is sampled (i.e., a
transition probability for each state-action pair) according to the posterior distribution of the Bayesian
model. Then, the agent plans in this sampled MDP and uses the resulting policy to interact with the environment. Notably, an optimistic variant of \PSRL, named optimistic posterior sampling for reinforcement learning (\OPSRL,\, \citealp{agrawal2020posterior,tiapkin2022optimistic}) also enjoys an optimal sample complexity \citep{tiapkin2022optimistic}. The random least square value iteration (\RLSVI, \citet{osband2013more}) is another well-known model-based algorithm that leverages a Bayesian-based technique for exploration. Precisely, \RLSVI directly sets a Gaussian prior on the optimal Q-values and then updates the associated posterior trough value iteration in a model \citep{osband2013more,russo2019worst}. A close variant of \RLSVI proposed by \citet{xiong2021nearoptimal}, using a more sophisticated prior/posterior couple, is also proven to be near-optimal. 
\par
It is noteworthy that Bayesian-based exploration techniques have shown superior empirical performance compared to bonus-based exploration, at least in the tabular setting \citep{osband2013more,osband2017why}. Furthermore, these techniques have also been successfully applied to the deep RL setting \citep{osband2016deep, azizzadenesheli2018efficient, fortunato2018noisy, li2022hyperdqn, sasso2023posterior}. Finally, Bayesian methods allow for the incorporation of apriori information into exploration (e.g. by giving more weight to important states).  However, most of the theoretical studies on Bayesian-based exploration have focused on model-based algorithms, raising the natural question of whether the \PSRL approach can be extended to a provably efficient model-free algorithm that matches the good empirical performance of its model-based counterparts. Recently, \cite{dann2021modelfree} proposed a model-free posterior sampling algorithm for structured MDPs, however, it is not computationally tractable. Therefore, a provably tractable model-free posterior sampling algorithm has remained a challenge.
\par
In this paper, we aim to resolve this challenge. We propose the randomized Q-learning (\RandQL) algorithm that achieves exploration without bonuses, relying instead on a novel idea of learning rate randomization. \RandQL is a tractable model-free algorithm that updates an ensemble of Q-values via Q-learning with Beta distributed step-sizes. If tuned appropriately, the noise introduced by the random learning rates is similar to the one obtained by sampling from the posterior of the \PSRL algorithm. Thus, one can see the ensemble of Q-values as posterior samples from the same induced posterior on the optimal Q-values as in \PSRL. Then, \RandQL chooses among these samples in the same optimistic fashion as \OPSRL. We prove that for tabular MDPs, a staged version \citep{zhang2020advantage} of \RandQL, named \StagedRandQL enjoys the same regret bound as the \OptQL algorithm, that is, $\tcO(\sqrt{H^5SAT})$ where $S$ is the number of states and $A$ the number of actions. Furthermore, we extend \StagedRandQL beyond the tabular setting into the \NetStagedRandQL algorithm to deal with metric state-action spaces \citep{domingues2021kernel,sinclair2019adaptive}. \NetStagedRandQL operates similarly to \StagedRandQL but over a fixed discretization of the state-action space and uses a specific prior tuning to handle the effect of discretization. We prove that \NetStagedRandQL enjoys a regret bound of order $\tcO(H^{5/2} T^{(d_c+1)/(d_c+2)})$, where $d_c$ denotes the covering dimension. This rate is of the same order as the one of \AdaptiveQL by \citet{sinclair2019adaptive,sinclair2022adaptive}, an adaptation of \OptQL to metric state-action space and has a better dependence on the budget $T$ than one of the model-based kernel algorithms such that \KernelUCBVI by \citet{domingues2021kernel}. We also explain how to adapt \NetStagedRandQL and its analysis to work with an adaptive discretization as by \citet{sinclair2019adaptive,sinclair2022adaptive}. Finally, we provide preliminary experiments to illustrate the good performance of \RandQL against several baselines in finite and continuous environments.

We highlight our main contributions:
\vspace{-0.1cm}
\begin{itemize}[itemsep=-2pt,leftmargin=6pt]
    \item The \RandQL algorithm, a new tractable (provably efficient) model-free Q-learning adaptation of the \PSRL algorithm that explores through randomization of the learning rates. 
    \item A regret bound of order $\tcO(\sqrt{H^5SAT})$ for a staged version of the \RandQL algorithm in finite MDPs where $S$ is the number of states and $A$ the number of actions, $H$ the horizon and $T$ the budget.
    \item A regret bound of order  $\tcO(H^{5/2} T^{(d_c+1)/(d_c+2)})$ for an adaptation of \RandQL to metric spaces where $d_c$ denotes the covering dimension.
    \item Adaptive version of metric space extension of \RandQL algorithm that achieves a regret bound of order $\tcO(H^{5/2} T^{(d_z+1)/(d_z+2)})$, where $d_z$ is a \textit{zooming} dimension.
    \item Experiments in finite and continuous MDPs that show that \RandQL is competitive with model-based and model-free baselines while keeping a low time-complexity.
\end{itemize}

\section{Setting}
\label{sec:setting}

 We consider an episodic MDP $\left(\cS, \cA, H, \{p_h\}_{h\in[H]},\{r_h\}_{h\in[H]}\right)$,  where $\cS$ is the set of states, $\cA$ is the set of actions,
 $H$ is the number of steps in one episode, $p_h(s'|s,a)$ is the probability transition from state~$s$ to state~$s'$ upon taking action $a$ at step $h,$ and $r_h(s,a)\in[0,1]$ is the bounded deterministic reward received after taking the action $a$ in state $s$ at step $h$. Note that we consider the general case of rewards and transition functions that are possibly non-stationary, i.e., that are allowed to depend on the decision step $h$ in the episode.

\vspace{-0.3cm}
\paragraph{Policy \& value functions} A \emph{deterministic} policy $\pi$ is a collection of functions $\pi_h : \cS \to \cA$ for all $h\in [H]$, where every $\pi_h$  maps each state to a \emph{single} action. The value functions of $\pi$, denoted by $V_h^\pi$, as well as the optimal value functions, denoted by $\Vstar_h$ are given by the Bellman and the optimal Bellman equations,
\begin{align*}
	Q_h^{\pi}(s,a) &= r_h(s,a) + p_h V_{h+1}^\pi(s,a) & V_h^\pi(s) &= \pi_h Q_h^\pi (s)\\
  Q_h^\star(s,a) &=  r_h(s,a) + p_h V_{h+1}^\star(s,a) & V_h^\star(s) &= \max_a Q_h^\star (s, a),
\end{align*}
where by definition, $V_{H+1}^\star \triangleq V_{H+1}^\pi \triangleq 0$. Furthermore, $p_{h} f>(s, a) \triangleq \E_{s' \sim p_h(\cdot | s, a)} \left[f(s')\right]$   denotes the expectation operator with respect to the transition probabilities $p_h$ and
$\pi_h g(s) \triangleq  g(s,\pi_h(s))$ denotes the composition with the policy~$\pi$ at step $h$.

\vspace{-0.3cm}
\paragraph{Learning problem} The agent, to which the transitions are \emph{unknown} (the rewards are assumed to be known\footnote{Our work can be extended without too much difficulty to the case of random rewards.} for simplicity), interacts with the environment during $T$ episodes of length $H$, with a \emph{fixed} initial state $s_1$.\footnote{As explained by \citet{fiechter1994efficient} if the first state is sampled randomly as $s_1\sim p,$ we can simply add an artificial first state $s_{1'}$ such that for  any action $a$, the transition probability is defined as the distribution $p_{1'}(s_{1'},a) \triangleq p.$} Before each episode $t$ the agent selects a policy $\pi^t$ based only on the past observed transitions up to episode $t-1$. At each step $h\in[H]$ in episode $t$, the agent observes a state $s_h^t\in\cS$, takes an action $\pi_h^t(s_h^t) = a_h^t\in\cA$ and  makes a transition to a new state $s_{h+1}^t$ according to the probability distribution $p_h(s_h^t,a_h^t)$ and receives a deterministic reward $r_h(s_h^t,a_h^t)$.

\vspace{-0.3cm}
\paragraph{Regret} The quality of an agent is measured through its regret, that is the difference between what it could obtain (in expectation) by acting optimally and what it really gets,
%\begin{small}
\[
\regret^T \triangleq  \sum_{t=1}^T \Vstar_1(s_1)- V_1^{\pi^t}(s_1)\,.
\]
%\end{small}

\vspace{-0.3cm}
\paragraph{Additional notation} For $N\in\N_{++},$ we define the set $[N]\triangleq \{1,\ldots,N\}$. We denote the uniform distribution over this set by $\Unif[N]$. We define the beta distribution with parameters $\alpha,\beta$ as $\Beta(\alpha,\beta)$.
%The vector of dimension $N$ with all entries one is $\bOne^N \triangleq  (1,\ldots,1)^\top$.%\markcomment{transpose?}. 
 %The empirical probability distribution $\hp^{\,t}_h(s,a)$ is defined as $\hp^{\,t}_h(s'|s,a) = n^{\,t}_h(s'|s,a) / n^{\,t}_h(s,a)$ if $n^t_h(s, a) > 0 $ and $\hp^{\,0}_h(s'|s,a)=1/S$ otherwise. 
 Appendix~\ref{app:notations} references all the notation used.

\section{Randomized Q-learning for Tabular Environments}

In this section, we assume that the state space $\cS$ is finite of size $S$ as well as the action space $\cA$ of size $A$. We first provide some intuitions for \RandQL algorithm. 

\subsection{Concept}
\label{sec:concept}
The main idea of \RandQL is to perform the usual Q-learning updates but instead of adding bonuses to the targets as \OptQL to drive exploration, \RandQL injects noise into the updates of the Q-values through \emph{noisy learning rates}. Precisely, for $J\in\N$, we maintain an ensemble of size $J$ of Q-values\footnote{We index the quantities by $n$ in this section where $n$ is the
number of times the state-action pair $(s, a)$ is visited. In particular
this is different from the global time $t$ since, in our setting, all the state-
action pair are not visited at each episode. See Section~\ref{sec:algorithm} and Appendix~\ref{app:description_randQL}
precise notations.} $(\uQ^{n,j})_{j\in[J]}$ updated with random independent Beta-distributed step-sizes $(w_{n,j})_{j\in[J]}$ where $w_{n,j} \sim \Beta(H, n)$. Then, policy Q-values $\uQ^{n}$ are  obtained by taking the maximum among the Q-values of the ensemble
\begin{align*}
    \uQ^{n+1,j}_h(s,a) &= 
        (1 - w_{n,j}) \uQ^{n,j}_h(s,a) + w_{n,j} [r_h(s,a) + \uV^{n}_{h+1}(s^n_{h+1})] \\
    \uQ^{n+1}_h(s,a) &= \max_{j \in [J]} \uQ^{n+1,j}_h(s,a), \quad \uV^{n+1}_h(s) = \max_{a \in \cA} \uQ^{n+1}_h(s,a),
\end{align*}
where \(s^n_{h+1}\) stands for the next (in time) state after \(n\)-th visitation of  \((s,a)\) at step \(h\).

Note that the policy Q-values $\uQ^{n}$ are designed to be upper confidence bound on the optimal Q-values. The policy used to interact with the environment is greedy with respect to the policy Q-values $\pi_h^n(s) \in\argmax_a \uQ_h^n(s,a)$. We provide a formal description of \RandQL in Appendix~\ref{app:description_randQL}.

\vspace{-0.3cm}
\paragraph{Connection with \OptQL} We observe that the learning rates of \RandQL are in expectation of the same order $\E[w_{n,j} ]= H/(n+H)$ as the ones used by the \OptQL algorithm. Thus, we can view our randomized Q-learning as a noisy version of the \OptQL algorithm that doesn't use bonuses.

\vspace{-0.3cm}
\paragraph{Connection with \PSRL} If we unfold the recursive formula above we can express the Q-values $\uQ^{n+1,j}$ as a weighted sum  

\vspace{-0.4cm}
\[
    \uQ^{n+1,j}_h(s,a) = W^0_{n,j} \uQ^{1,j}_h(s,a) + \sum_{k=1}^{n} W^k_{n,j} [r_h(s,a) + \uV^{k}_{h+1}(s^{k}_{h+1})],
\]

\vspace{-0.35cm}
\!where we define 
$W^0_{n,j} = \prod_{\ell=0}^{n-1} (1 - w_{\ell,j})$ and $W^k_{n,j} = w_{k-1,j}  \prod_{\ell=k}^{n-1} (1 - w_{\ell,j}).$

To compare, we can unfold the corresponding formula for \PSRL algorithm using the aggregation properties of the Dirichlet distribution (see e.g. Section 4 of \cite{tiapkin2022dirichlet} or Appendix~\ref{app:weights_randql})
\begin{equation}
\label{eq:qpsrl}
    \uQ^{n+1}_h(s,a) = \tW^0_{n} \uQ^{1}_h(s,a) + \sum_{k=1}^n \tW^{k}_{n} [r_h(s,a) + \uV^{n+1}_{h+1}(s^k_{h+1})],
\end{equation}
where weights $(\tW^0_{n},\ldots, \tW^{n}_n)$ follows Dirichlet distribution $\Dir(n_0, 1, \ldots, 1)$ and $n_0$ is a weight for the prior distribution. In particular, one can represent these weights as partial products of other weights $w_{n} \sim \Beta(1, n+n_0)$. If we use \eqref{eq:qpsrl} to construct a model-free algorithm, this would require recomputing the targets $r_h(s,a) + \uV^{n+1}(s^k_{h+1})$ in each iteration. To make algorithm more efficient and model-free, we  approximate $\uV^{n+1}$ by $\uV^k$, and, as a result, obtain \RandQL algorithm with weight distribution $w_{n,j} \sim \Beta(1, n+n_0)$. 
\par
Note that  in expectation this algorithm is  equivalent to \OptQL with the uniform step-sizes which are known to be sub-optimal due to a high bias (see discussion in Section 3 of \citep{jin2018is}). There are two known ways to overcome this sub-optimality for Q-learning: to introduce more aggressive learning rates $w_{n,j} \sim \Beta(H, n+n_0)$ leading to \RandQL algorithm, or to use  stage-dependent framework by \cite{bai2019provably,zhang2020advantage} resulting in \StagedRandQL algorithm.
\par
The aforementioned  transition from \PSRL to \RandQL is similar to the transition from \UCBVI \citep{azar2017minimax} to Q-learning. To make \UCBVI model-free, one has to to keep old targets in Q-values. This, however,  introduces a bias that could be eliminated either by more aggressive step-size \citep{jin2018is} or by splitting on stages \citep{bai2019provably}. Our algorithms (\RandQL and \StagedRandQL) perform the similar tricks for \PSRL and thus could be viewed as model-free versions of it. Additionally, \RandQL shares some similarities with the \OPSRL algorithm \citep{agrawal2020posterior, tiapkin2022optimistic} in the way of introducing optimism (taking maximum over $J$ independent ensembles of Q-values). 
Let us also mention a close connection to the theory of Dirichlet processes in the proof of optimism for the  case of metric spaces (see Remark~\ref{rm:dirichlet_process} in Appendix~\ref{app:optimism_metric}).

\paragraph{Prior} As remarked above, in expectation, \RandQL has a learning rate of the same order as \OptQL. In particular, it implies that the first $(1-1/H)$ fraction of the the target will be forgotten exponentially fast in the estimation of the Q-values, see \citet{jin2018is,menard2021ucb}. Thus we need to re-inject prior targets, as explained in Appendix~\ref{app:description_randQL}, in order to not forget too quickly the prior and thus replicate the same exploration mechanism as in the \PSRL algorithm.

\subsection{Algorithm}
\label{sec:algorithm}

In this section, following \citet{bai2019provably,zhang2020advantage}, we present the \StagedRandQL algorithm a scheduled version of \RandQL that is simpler to analyse. The main idea is that instead of using a carefully tuned learning rate to keep only the last $1/H$ fraction of the targets we split the learning of the Q-values in stages of  exponentially increasing size with growth rate of order $1+1/H$. At a given stage, the estimate of the Q-value relies only on the targets within this stage and resets at the beginning of the next stage. Notice that the two procedures are almost equivalent. A detail description of \StagedRandQL is provided in Algorithm~\ref{alg:StagedRandQL}.

\textbf{Counts and stages} Let $n^t_h(s,a) \triangleq \sum_{i=1}^{t-1} \ind\{ (s^i_h, a^i_h) = (s,a) \}$ be the number of visits of state-action pair $(s,a)$ at step $h$ before episode $t$.  We say that a triple $(s,a,h)$  belongs to the $k$-th stage at the beginning of episode $t$ if $n^t_h(s,a) \in [\sum_{i=0}^{k-1} e_i, \sum_{i=0}^k e_i )$. Here  $e_k = \lfloor (1 + 1/H)^k \cdot H \rfloor$ is the length of the stage $k \geq 0$ and, by convention, $e_{-1} = 0$. Let $\tn^t_h(s,a)\triangleq n^t_h(s,a) - \sum_{i=0}^{k-1} e_i$ be the number of visits of state-action pair $(s,a)$ at step $h$ during the current stage $k$.

\textbf{Temporary Q-values} At the beginning of a stage, let say time $t$, we initialize $J$ \textit{temporary} Q-values as $\tQ^{t,j}_h(s,a) = r_h(s,a) +  \ur (H-h-1)$ for $j\in[J]$ and $r_0$ some pseudo-reward.
Then as long as $(s^t_h, a^t_h,h)$ remains  within a stage we update recursively the \textit{temporary} Q-values
\[
    \tQ^{t+1,j}_h(s,a) = \begin{cases}
        (1- w_{j, \tn}) \tQ^{t,j}_h(s,a) + w_{j, \tn} [r_h(s,a) + \uV^{t}_{h+1}(s^t_{h+1})], & (s,a) = (s^t_h, a^t_h) \\
        \tQ^{t,j}_h(s,a) & \text{otherwise},
    \end{cases}
\]
where $\tn=\tn^t_h(s,a)$ is the number of visits, $w_{j, \tn}$ is a sequence of i.i.d. random variables $w_{j,\tn} \sim \Beta(1/\kappa, (\tn + n_0) / \kappa)$ with $\kappa >0$ being some posterior inflation coefficient and $n_0$ a number of pseudo-transitions.

\textbf{Policy Q-values} Next we define the policy Q-values that is updated at the end of a stage. Let say for state-action pair $(s,a)$ at step $h$ an stage ends at time $t$. This policy Q-values is then given by the maximum of temporary Q-values $\uQ_h^{t+1}=\max_{j\in[J]}  \tQ^{t+1,j}_h(s,a)$. Then the policy Q-values is constant within a stage. The value used to defined the targets is
$\uV^{t+1}_h(s) = \max_{a \in \cA} \uQ^{t+1}_h(s,a)$. The policy used to interact with the environment is greedy with respect to the policy Q-values $\pi^{t+1}_h(s) \in \argmax_{a \in \cA} \uQ^{t+1}_h(s,a)$ (we break ties arbitrarily).

\begin{figure}[t]
\vspace{-0.5cm}
\begin{algorithm}[H]
\centering
\caption{Tabular \StagedRandQL}
\label{alg:StagedRandQL}
\begin{algorithmic}[1]
  \STATE {\bfseries Input:}  inflation coefficient $\kappa$, $J$ ensemble size, number of prior transitions $n_0$, prior reward $r_0$.
  \STATE {\bfseries Initialize: }  $\uV_h(s) = \uQ_h(s,a) = \tQ^j_h(s,a) = r(s,a) + \ur(H-h-1),$ initialize counters $\tn_h(s,a) = 0$ for $j,h,s,a\in[J]\times[H]\times\cS\times\cA$ and stage $q_h(s,a) = 0$.
      \FOR{$t \in[T]$}
      \FOR{$h \in [H]$}
        \STATE Play $a_h \in \argmax_a \uQ_h(s_h,a)$.
        \STATE Observe reward and next state $s_{h+1}\sim p_h(s_h,a_h)$.
        \STATE Sample learning rates $w_j \sim \Beta(1/\kappa, (\tn+n_0)/\kappa)$ for $\tn = \tn_h(s_h,a_h)$.
        \STATE Update temporary $Q$-values for all $j \in [J]$
        {\small\[
            \tQ^{j}_h(s_h,a_h) := (1 - w_j) \tQ^{j}_h(s_h,a_h) + w_j \left( r_h(s_h,a_h) + \uV_{h+1}(s_{h+1})\right)\,.
        \]}
        \vspace{-0.4cm}
        \STATE Update counter $\tn_h(s_h, a_h) := \tn_h(s_h, a_h) + 1$
        \IF{$\tn_h(s_h, a_h) = \lfloor (1 + 1/H)^{q} H \rfloor$ for $q=q_h(s_h, a_h)$ being the current stage}
            \STATE Update policy $Q$-values $\uQ_h(s_h,a_h) := \max_{j \in [J]} \tQ^{j}_h(s_h,a_h)$.
            \STATE Update value function $\uV_h(s_h) := \max_{a \in \cA} \uQ_h(s_h,a)$ 
            \STATE Reset temporary $Q$-values $\tQ^j_h(s_h,a_h) := r_h(s_h,a_h) + \ur(H-h-1)$.
            \STATE Reset counter $\tn_h(s_h,a_h) := 0$ and change stage $q_h(s_h, a_h) := q_h(s_h, a_h) + 1$.
        \ENDIF
      \ENDFOR
  \ENDFOR
\end{algorithmic}
\end{algorithm}
\vspace{-1cm}
\end{figure}

\subsection{Regret bound}

We fix $\delta\in(0,1)$ and the number of posterior samples $J \triangleq \lceil c_J \cdot \log(2SAHT/\delta) \rceil$, where $c_J = 1/\log(2/(1 + \Phi(1)))$ and $\Phi(\cdot)$ is the  cumulative distribution function (CDF) of a normal distribution. Note that $J$ has a logarithmic dependence on $S,A,H,T,$ and $1/\delta$.

We now state the regret bound of \StagedRandQL with a full proof in Appendix~\ref{app:randql_tabular_proof}.
\begin{theorem}\label{th:regret_randql_tabular}
    Consider a parameter $\delta \in (0,1)$.    
    Let $\kappa \triangleq 2(\log(8 SAH/\delta) + 3\log(\rme\pi(2T+1)))$, $n_0 \triangleq \lceil \kappa(c_{0} + \log_{17/16}(T)) \rceil$, $\ur \triangleq 2$, where  $c_{0}$ is an absolute constant defined in \eqref{eq:constant_c0}; see Appendix~\ref{app:optimism_tabular}. Then for \StagedRandQL, with probability at least $1-\delta$, 
    \[
        \regret^T = \tcO\left( \sqrt{H^5 SAT}  + H^3 S A \right).
    \]
\end{theorem}
\vspace{-0.3cm}
\paragraph{Discussion} The regret bound of Theorem~\ref{th:regret_randql_tabular}  coincides (up to a logarithmic factor) with the bound of the \OptQL algorithm with Hoeffding-type bonuses from \citet{jin2018is}.  Up to a $H$ factor, our regret matches the information-theoretic lower bound $\Omega(\sqrt{H^3SAT})$ \citep{jin2018is,domingues2021episodic}. This bound could be achieved (up to logarithmic terms) in model-free algorithms by using Bernstein-type bonuses and variance reduction \citep{zhang2020advantage}. We keep these refinements for future research as the main focus of our paper is on the novel randomization technique and its use to construct computationally tractable model-free algorithms. 

\vspace{-0.3cm}
\paragraph{Computational complexity} \StagedRandQL is a model-free algorithm, and thus gets the $\tcO(HSA)$ space complexity as \OptQL, recall that we set $J=\tcO(1)$. The per-episode time-complexity is also similar and of order $\tcO(H)$ .

\section{Randomized Q-learning for Metric Spaces}
\label{sec:metric_space}
In this section we present a way to extend \RandQL to general state-action spaces. We start from the simplest approach with predefined $\varepsilon$-net type discretization of the state-action space $\cS \times \cA$ (see \citealt{song2019efficient}), and then discuss an adaptive version of the algorithm, similar to one presented by \citet{sinclair2019adaptive}.

\subsection{Assumptions}

To pose the first assumption, we start from a general definition of covering numbers.
\begin{definition}[Covering number and covering dimension]
    Let $(M, \rho)$ be a metric space. A set $\cM$ of open balls of radius $\varepsilon$ is called an $\varepsilon$-cover  of $M$ if $M \subseteq \bigcup_{B \in \cM} B$. The cardinality of the minimal $\varepsilon$-cover is called covering number $N_{\varepsilon}$ of $(M,\rho)$. We denote the corresponding minimal $\varepsilon$-covering by $\cN_{\varepsilon}$. A metric space $(M, \rho)$ has a covering dimension $d_c$ if $\forall \varepsilon > 0 : N_{\varepsilon} \leq C_N \varepsilon^{-d_c}$, where $C_N$ is a constant.
\end{definition}

The last definition extends the definition of dimension beyond vector spaces. For example, is case of $M = [0,1]^d$ the covering dimension of $M$ is equal to $d$. For more details and examples see e.g. \citet[Section 4.2]{vershynin2018high}.

Next we are ready to introduce the first assumption.
\begin{assumption}[Metric Assumption]\label{ass:metric}
    Spaces $\cS$ and $\cA$ are separable compact metric spaces with the corresponding metrics $\rho_{\cS}$ and $\rho_{\cA}$. The joint space $\cS \times \cA$ endowed with a product metric $\rho$ that satisfies $\rho((s,a),(s',a')) \leq \rho_{\cS}(s,s') + \rho_{\cA}(a,a')$. Moreover, the diameter of $\cS \times \cA$ is bounded by $d_{\max}$, and $\cS \times \cA$ has  covering dimension $d_c$ with a constant $C_N$.
\end{assumption}

This assumption is, for example, satisfied for the finite state and action spaces endowed with discrete metrics $\rho_{\cS}(s,s') = \ind\{s \not = s'\}, \rho_\cA(a,a') = \ind\{a \not = a'\}$ with $d_c = 0,$ $C_N = SA$ and $S$ and $A$ being the cardinalities of the state and action spaces respectively. The above assumption also holds in the case $\cS \subseteq [0,1]^{d_{\cS}}$ and $\cA \subseteq [0,1]^{d_{\cA}}$ with $d_c = d_{\cS} + d_{\cA}$.
\par
The next two assumptions describe the regularity conditions of transition kernel and rewards.
\begin{assumption}[Reparametrization Assumption]\label{ass:reparametrization}
    The Markov transition kernel could be represented as an iterated random function. In other words, there exists a measurable space $(\Xi, \cF_\Xi)$ and a measurable function $F_h \colon (\cS \times \cA) \times \Xi \to \cS$, such that $s_{h+1} \sim p_h(s_h,a_h) \iff s_{h+1} = F_h(s_h,a_h, \xi_h)$ for a sequence of independent random variables $\{\xi_h\}_{h\in[H]}$. 
\end{assumption}
This assumption is naturally satisfied for a large family of probabilistic model, see \citet{kingma2014autoencoding}. Moreover, it has been utilized by the RL community both in theory \citep{ye2015information} and practice \citep{heess2015learning,liu2018action-dependent}.
Essentially, this assumption holds for Markov transition kernels over a separable metric space, see Theorem 1.3.6 by \cite{douc2018markov}. However, the function $F_h$ could be ill-behaved. To avoid this behaviour, we need the following assumption.
\begin{assumption}[Lipschitz Assumption]\label{ass:lipschitz}
    The function $F_h(\cdot, \xi_h)$ is $L_F$-Lipschitz in the first argument for almost every value of $\xi_h$. Additionally, the reward function $r_h \colon \cS \times \cA \to [0,1]$ is $L_r$-Lipschitz.
\end{assumption}
This assumption is commonly used in studies of the Markov processes corresponding to iterated random functions, see \citet{diaconis1999iterated,ghosh2022iterated}.
%\dan{Maybe remove the next "intuition" section}
%In the setting of reinforcement learning problem, the main intuition behind this assumption is following. Any program implementation of a simulator should be reproducible, i.e. all its randomness depends on initial random seed. In this terms, we may interpret $\xi_h$ as a random seed and $F_h(s,a, \xi_h)$ is a computation by a simulator of the next transition. If this simulator is well-behaved (e.g. Lipschitz) with respect to any random seed, then we say that Assumption~\ref{ass:lipschitz} is satisfied.
Moreover, this assumption holds for many  cases of interest. As main example, it trivially holds in tabular and Lipschitz continuous deterministic MDPs \citep{ni2019learning}. Notably, this observation demonstrates that Assumption~\ref{ass:lipschitz} does not necessitate Lipschitz continuity of the transition kernels in total variation distance, since deterministic Lipschitz MDPs are not continuous in that sense. Additionally, incorporation of an additive noise to deterministic Lipschitz MDPs will lead to  Assumption~\ref{ass:lipschitz} with  $L_F = 1$.

Furthermore, it is possible to show that   Assumption~\ref{ass:lipschitz} implies other assumptions stated in the literature. For example, it implies that the transition kernel is Lipschitz continuous in \(1\)-Wasserstein metric, and that $\Qstar$ and $\Vstar$ are both Lipschitz continuous.

\begin{lemma}\label{lem:transition_w_from_reparam}
    Let Assumption~\ref{ass:metric},\ref{ass:reparametrization},\ref{ass:lipschitz} hold. Then the transition kernels $p_h(s,a)$ are $L_F$-Lipschitz continuous in \(1\)-Wasserstein distance
    \[
        \cW_1(p_h(s,a), p_h(s',a')) \leq L_F \cdot \rho((s,a), (s',a')),
    \]
    where $1$-Wasserstein distance between two probability measures on the metric space $(M,\rho)$ is defined as $\cW_1(\nu, \eta) = \sup_{f \text{ is } 1-\text{Lipschitz}} \int_{M} f \rmd \nu - \int_{M} f \rmd \eta$.
\end{lemma}

\begin{lemma}\label{lem:v_lipschitz_from_reparam}
    Let Assumption~\ref{ass:metric},\ref{ass:reparametrization},\ref{ass:lipschitz} hold. Then $\Qstar_h$ and $\Vstar_h$ are Lipschitz continuous with Lipschitz constant $L_{V,h} \leq \sum_{h'=h}^H L_F^{h'-h} L_r$. 
\end{lemma}

The proof of these lemmas is postponed to Appendix~\ref{app:randql_metric_proof}.  For a more detailed exposition on 1-Wasserstein distance we refer to the book by \citet{peyre2019ot}. The first assumption was studied by \citet{domingues2021kernel,sinclair2022adaptive} in the setting  of model-based algorithms in metric spaces. We are not aware of any natural examples of MDPs with a compact state-action space where the transition kernels are Lipschitz in $\cW_1$ but fail to satisfy  Assumption~\ref{ass:lipschitz}.

\subsection{Algorithms}

In this section, following \citet{song2019efficient}, we present \NetStagedRandQL algorithm that combines a simple non-adaptive discretization and an idea of stages by \citet{bai2019provably,zhang2020advantage}.

We assume that we have an access to all Lipschitz constants $L_r, L_F, L_{V}\triangleq L_{V,1}$. Additionally,  we have  access to the oracle that computes $\varepsilon$-cover $\cN_{\varepsilon}$ of the space $\cS \times \cA$ for any predefined $\varepsilon > 0$\footnote{Remark that the simple greedy algorithm can generate $\varepsilon$-cover of size $N_{\varepsilon/2}$, that will not affect the asymptotic behavior of our regret bounds, see \citet{song2019efficient}.}. 

\textbf{Counts and stages} Let $n^t_h(B) \triangleq \sum_{i=1}^{t-1} \ind\{ (s^i_h, a^i_h) \in B \}$ be the number of visits of the ball $B \in \cN_{\varepsilon}$ at step $h$ before episode $t$. Let $e_k = \lfloor (1 + 1/H)^k \cdot H \rfloor$ be length of the stage $k \geq 0$ and, by convention, $e_{-1} = 0$. We say that $(B,h)$  belongs to the $k$-th stage at the beginning of episode $t$ if $n^t_h(B) \in [\sum_{i=0}^{k-1} e_i, \sum_{i=0}^k e_i )$. Let $\tn^t_h(B)\triangleq n^t_h(s,a) - \sum_{i=0}^{k-1} e_i$ be the number of visits of the ball $B$ at step $h$ during the current stage $k$.

\textbf{Temporary Q-values} At the beginning of a stage, let say time $t$, we initialize $J$ \textit{temporary} Q-values as $\tQ^{t,j}_h(B) = \ur H $ for $j\in[J]$ and $\ur$ some pseudo-reward. Then within a stage $k$ we update recursively the \textit{temporary} Q-values
\[
    \tQ^{t+1,j}_h(B) = \begin{cases}
        (1- w_{j, \tn}) \tQ^{t,j}_h(B) + w_{j, \tn} [r_h(s^t_h,a^t_h) + \uV^{t}_{h+1}(s^t_{h+1})], & (s,a) = (s^t_h, a^t_h) \\
        \tQ^{1,j}_h(B) & \text{otherwise},
    \end{cases}
\]
where $\tn=\tn^t_h(B)$ is the number of visits, $w_{j, \tn}$ is a sequence of i.i.d random variables $w_{j,\tn} \sim \Beta(1/\kappa, (\tn + n_0(k)) / \kappa)$ with $\kappa >0$ some posterior inflation coefficient and $n_0(k)$ a number of pseudo-transitions. The important difference between tabular and metric settings  is the dependence on the pseudo-count $n_0(k)$ on $k$ in the latter case, since here the  prior  is  used to eliminate  the approximation error.

\textbf{Policy Q-values} Next, we define the policy Q-values that are updated at the end of a stage. Let us fix a ball $B$ at step $h$ and suppose that the currents stage ends at time $t$. Then the policy Q-values are given by the maximum of the temporary Q-values $\uQ_h^{t+1}(B) =\max_{j\in[J]}  \tQ^{t+1,j}_h(B)$. The policy Q-values are constant within a stage. The value used to define the targets is computed on-flight using the formula $\uV^{t}_h(s) = \max_{a \in \cA} \uQ^{t}_h(\psi_{\varepsilon}(s,a))$, where $\psi_{\varepsilon} \colon \cS \times \cA \to \cN_{\varepsilon}$ is a quantization map, that assigns each state-action pair $(s,a)$ to a ball $B \ni (s,a)$. The policy used to interact with the environment is greedy with respect to the policy Q-values and also computed on-flight $\pi^{t}_h(s) \in \argmax_{a \in \cA} \uQ^t_h(\psi_{\varepsilon}(s,a))$ (we break ties arbitrarily).

A detail description of \NetStagedRandQL is provided in Algorithm~\ref{alg:NetStagedRandQL} in Appendix~\ref{app:net_staged_randql_description}.

\subsection{Regret Bound}
 We fix $\delta\in(0,1),$ the discretization level $\varepsilon > 0$ and the number of posterior samples 
\begin{equation*}
\label{eq:def_J_metric}
    J \triangleq \lceil \tilde{c}_J \cdot ( \log(2C_NHT/\delta) + d_c \log(1/\varepsilon) ) \rceil,
\end{equation*}
where $\tilde{c}_J = 1/\log(4/(3 + \Phi(1)))$ and $\Phi(\cdot)$ is the  cumulative distribution function (CDF) of a normal distribution. Note that $J$ has a logarithmic dependence on $H,T,1/\varepsilon$ and $1/\delta$. For the regret-optimal discretization level $\varepsilon = T^{-1/(d_c + 2)}$, the number $J$ is almost  independent of $d_{c}$ .
Let us note that the role of prior in metric spaces is much higher than in the tabular setting. Another important  difference is dependence of the prior count on the stage index. In particular, we have
\[
    n_0(k) = \left\lceil \tn_0 + \kappa + \frac{\varepsilon L}{H-1} \cdot (e_k + \tn_0 + \kappa) \right\rceil, \qquad \tn_0 = (c_0 + 1 + \log_{17/16}(T)) \cdot \kappa
\]
where  $c_{0}$ is an absolute constant defined in \eqref{eq:constant_c0} ( see Appendix~\ref{app:optimism_tabular}), $\kappa$ is the posterior inflation coefficient and $L = L_r + (1+L_F)L_V$ is a  constant.
We now state the regret bound of \NetStagedRandQL with a full proof being postponed  to Appendix~\ref{app:randql_metric_proof}.
\begin{theorem}\label{th:regret_randql_metric}
    Suppose that \(N_{\varepsilon} \leq C_N \varepsilon^{-d_c}\) for all $\varepsilon>0$ and some constant $C_N>0.$
    Consider a parameter $\delta \in (0,1)$ and take an optimal level of discretization $\varepsilon = T^{-1/(d_c+ 2)}$.
    Let $\kappa \triangleq 2(\log(8HC_N/\delta) + d_c \log(1/\varepsilon) + 3\log(\rme\pi(2T+1)))$, $\ur \triangleq 2$. Then it holds for \NetStagedRandQL, with probability at least $1-\delta$, 
    \[
        \regret^T = \tcO\biggl( H^{5/2} C_N^{1/2} T^{\frac{d_c+1}{d_c+2}} + H^3 C_N T^{\frac{d_c}{d_c+2}} + L T^{\frac{d_c+1}{d_c+2}}\biggl).
    \]
\end{theorem}
We can restore the regret bound in the tabular setting by letting $d_c = 0$ and $C_N = SA$, where $S$ is the cardinality of the state-space, and $A$ is the cardinality of the action-space. 

%\vspace{-0.3cm}
\paragraph{Discussion}
From the point of view of instance-independent bounds, our algorithm achieves the same result as \NetQL \citep{song2019efficient} and \AdaptiveQL \citep{sinclair2019adaptive}, that matches the lower bound $\Omega(H T^{\frac{d_c+1}{d_c+2}})$ by \cite{sinclair2022adaptive} in dependence on budget $T$ and covering dimension $d_c$. Notably, as discussed by \cite{sinclair2022adaptive}, the model-based algorithm such as \KernelUCBVI \citep{domingues2021kernel} does not achieves optimal dependence in $T$ due to hardness of the transition estimation problem. 

\vspace{-0.3cm}
\paragraph{Computational complexity}
For a fixed level of discretization $\varepsilon$, our algorithm has a space complexity of order $\tcO(H\cN_{\varepsilon})$. Assuming that the computation of a quantization map $\psi_{\varepsilon}$ has $\tcO(1)$ time complexity, we achieve a per-episode time complexity of $\tcO(HA)$ for a finite action space and $\cO(H N_{\varepsilon})$ for an infinite action space in the worst case due to computation of $\argmax_{a \in \cA} \uQ_h(\psi_{\varepsilon}(s,a))$. However, this can be improved to $\tcO(H)$ if we consider adaptive discretization \citep{sinclair2019adaptive}.

\vspace{-0.3cm}
\paragraph{Adaptive discretization} Additionally, we propose a way to combine \RandQL with adaptive discretization by  \citet{cao2020provably,sinclair2022adaptive}. 
This combination results in two algorithms: \AdaptiveRandQL and \AdaptiveStagedRandQL.
The second one could achieve the instance-dependent regret bound that scales with a \textit{zooming} dimension, the instance-dependent measure of dimension. We will follow \cite{sinclair2022adaptive} in the exposition of the required notation.

\begin{definition}\label{def:value_gap}
    For any $(s,a) \in \cS \times \cA,$ the stage-dependent sub-optimality gap is defined as $\gap_h(s,a) = \Vstar_h(s) - \Qstar_h(s,a)$.
\end{definition}

This quantity is widely used in the theoretical instance-dependent analysis of reinforcement learning  and contextual bandit algorithms.   
\begin{definition}\label{def:near_optimal_set}
    The near-optimal set of $\cS \times \cA$ for a given value $\varepsilon$ defined as
    $
        Z^{\varepsilon}_h = \{ (s,a) \in \cS \times \cA  \mid \gap_h(s,a) \leq (H+1) \varepsilon \}.
    $
    %where $L = L_r + (1 + L_F) L_V$ is a constant depending on the Lipschitz constants of the problem (see Assumption~\ref{ass:lipschitz}).
\end{definition}
The main insight of this definition is that essentially we are interested in a detailed discretization  of the near-optimal set $Z^{\varepsilon}_h$ for small $\varepsilon$, whereas all other state-action pairs could be discretized in a more rough manner.  Interestingly enough, $Z^\varepsilon_h$ could be a lower dimensional manifold, leading to the following definition.
\begin{definition}\label{def:zooming_dim}
    The step-$h$ zooming dimension $d_{z,h}$ with a constant $C_{N,h}$ and a scaling factor $\rho > 0$ is given by
    \[
        d_{z,h} = \inf\left\{ d>0 : \forall \varepsilon > 0\ N_{\varepsilon}(Z^{\rho \cdot \varepsilon}_h) \leq C_{N,h} \varepsilon^{-d} \right\}.
    \]
\end{definition}

Under some additional structural assumptions on $\Qstar_h$, it is possible to show that the zooming dimension could be significantly smaller  than the covering dimension, see, e.g., Lemma~2.8 in \cite{sinclair2022adaptive}. However, at the same time, it has been shown that $d_{z,h} \geq d_{\cS} - 1$, where $d_{\cS}$ is a covering dimension of the state space. Thus, the zooming dimension allows adaptation to a rich action space but not a rich state space.

Given this definition, it is possible to define define an adaptive algorithm \AdaptiveStagedRandQL that attains the following regret guarantees
\begin{theorem}\label{th:regret_randql_adaptive}
    Consider a parameter $\delta \in (0,1)$. For a value $\kappa$ that depends on $T, d_c$ ad $\delta$, for \AdaptiveStagedRandQL the following holds with probability at least $1-\delta$, 
    \[
        \regret^T = \tcO\biggl(H^3 +  H^{3/2} \sum_{h=1}^H  T^{\frac{d_{z,h}+1}{d_{z,h}+2}}\biggl),
    \]
    where $d_{z,h}$ is the step-$h$ zooming dimension and we ignore all multiplicative factors in the covering dimension $d_c$, $\log(C_N)$, and Lipschitz constants.
\end{theorem}

We refer to Appendix~\ref{app:adaptive_randql_proof} to a formal statement and a proof.

\section{Experiments}\label{sec:experiments}

\vspace{-0.2cm}
In this section we present the experiments we conducted for tabular environments using \texttt{rlberry} library \citep{rlberry}. We also provide experiments in non-tabular environment in  Appendix~\ref{app:experiment_details}.

\vspace{-0.2cm}
\paragraph{Environment} We use a grid-world environment with $100$ states $(i, j) \in [10]\times[10]$ and $4$ actions (left, right, up and down). The horizon is set to $H=50$. When taking an action, the agent moves in the corresponding direction with probability $1-\epsilon$, and moves to a neighbor state at random with probability $\epsilon=0.2$. The agent starts at position $(1, 1)$. The reward equals to $1$ at the state $(10, 10)$ and is zero elsewhere. 

\begin{wrapfigure}{r}{0.5\textwidth}

    \vspace{-1cm}
    
    \centering
    \includegraphics[width=1.0\linewidth]{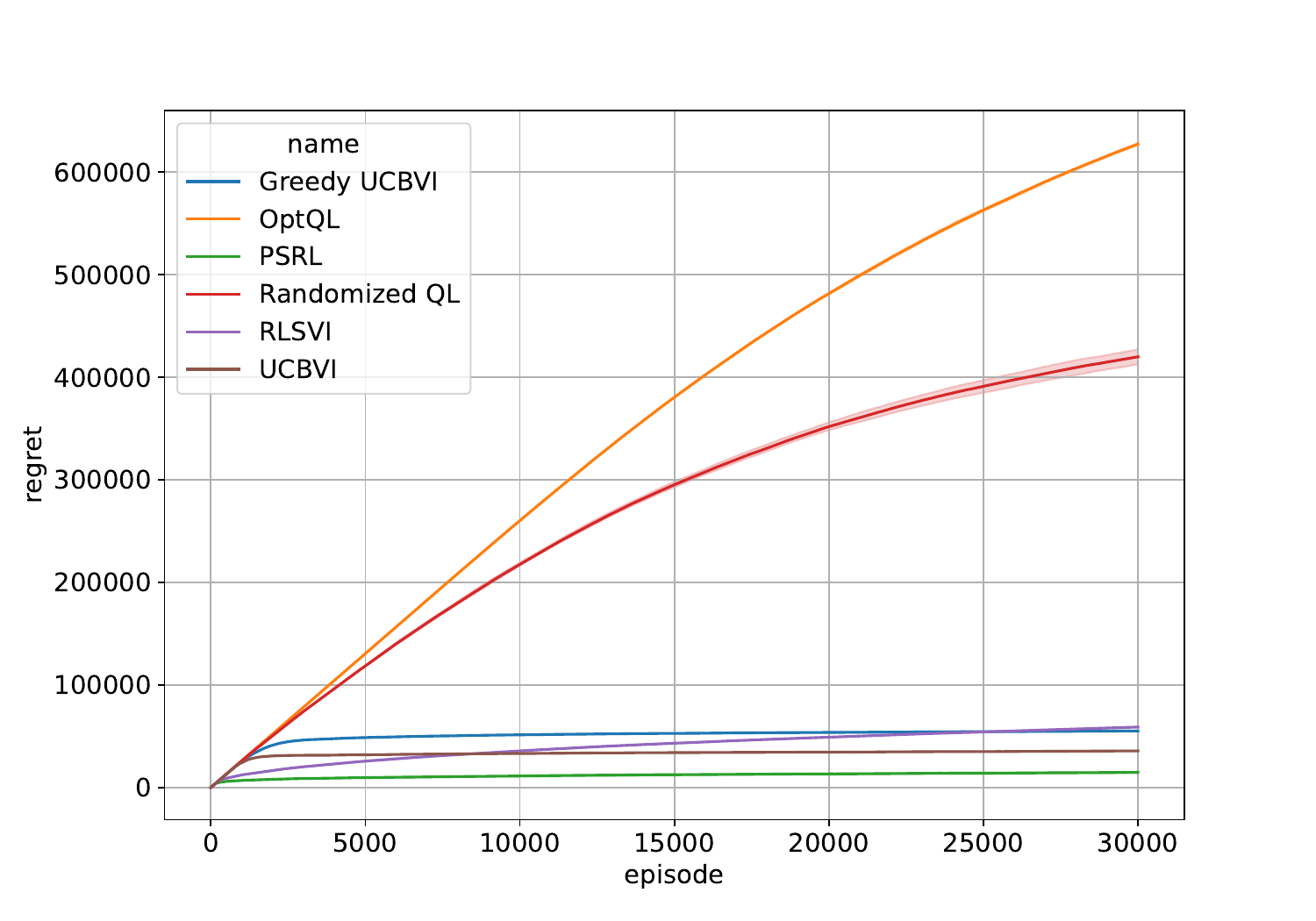}

    \caption{Regret curves of \RandQL and baselines in a grid-world environment for $H=50$
and transition noise $\epsilon = 0.2$. The average is over 4 seeds.}
    \label{fig:main_figure}
    
    \vspace{-1cm}
\end{wrapfigure}

\vspace{-0.2cm}
\paragraph{Variations of randomized Q-learning} For the tabular experiment we use the \RandQL algorithm, described in Appendix~\ref{app:description_randQL} as it is the version of randomized Q-learning that is the closest to the baseline \OptQL. Note that, we compare the different versions of randomized Q-learning in Appendix~\ref{app:description_randQL}.

\vspace{-0.2cm}
\paragraph{Baselines}
We compare \RandQL algorithm to the following baselines: (i) \OptQL \citep{jin2018is} (ii) \UCBVI \citep{azar2017minimax}  (iii) \GreedyUCBVI, a version of \UCBVI using real–
time dynamic programming \citep{efroni2019tight} (iv) \PSRL \citep{osband2013more} and (v) \RLSVI \citep{russo2019worst}.
For the hyper-parameters used for these baselines refer to Appendix~\ref{app:experiment_details}.

\vspace{-0.2cm}
\paragraph{Results} Figure~\ref{fig:main_figure} shows the result of the experiments. Overall, we see that \RandQL outperforms \OptQL algorithm on tabular environment, but still degrades in comparison to model-based approaches, that is usual for model-free algorithms in tabular environments. Indeed, using a model and backward induction allows new information to be more quickly propagated. But as counterpart, \RandQL has a better time-complexity and space-complexity than model-based algorithm, see Table~\ref{tab:time_space_complexity} in Appendix~\ref{app:experiment_details}.

\section{Conclusion}\label{sec:conclusion}

This paper introduced the \RandQL algorithm, a new model-free algorithm that achieves exploration without bonuses. It utilizes a novel idea of learning rate randomization, resulting in provable sample efficiency with regret of order $\tcO(\sqrt{H^5SAT})$ in the tabular case. We also extend \RandQL to the case of metric state-action space by using proper discretization techniques. The proposed algorithms inherit the good empirical performance of model-based Bayesian algorithm such that \PSRL while keeping the small space and time complexity of model-free algorithm. Our result rises following interesting open questions for a further research.

\vspace{-0.3cm}
\paragraph{Optimal rate for \RandQL} We conjecture that \RandQL could get optimal regret in the tabular setting if coupled with variance reductions techniques as used by \citet{zhang2020advantage}.
However, obtaining such improvements is not straightforward due to the intricate statistical dependencies involved in the analysis of \RandQL.

\vspace{-0.3cm}
\paragraph{Beyond one-step learning} We observe a large gap in the experiments between Q-learning type algorithm that do one-step planning and e.g. \UCBVI algorithm that does full planning or \GreedyUCBVI that does one-step planning with full back-up (expectation under transition of the model)  for all actions. Therefore, it would interesting to study also algorithms that range between these two extremes \citep{efroni2018multiple, efroni2019tight}.

\section*{Acknowledgments}

The work of D. Tiapkin, A. Naumov, and D. Belomestny were supported by the grant for research centers in the field of AI provided by the Analytical Center for the Government of the Russian Federation (ACRF) in accordance with the agreement on the provision of subsidies (identifier of the agreement 000000D730321P5Q0002) and the agreement with HSE University No. 70-2021-00139.
E. Moulines received support from the grant ANR-19-CHIA-002 SCAI and parts of his work has been done under the auspices of Lagrange Center for maths and computing. P. M\'enard acknowledges the Chaire SeqALO (ANR-20-CHIA-0020-01). This research was supported in part through computational resources of HPC facilities at HSE University.

\bibliographystyle{plainnat}
\bibliography{ref.bib}

\appendix
\newpage
\part{Appendix}
\parttoc
\newpage

\section{Notation}
\label{app:notations}

\begin{table}[h]
	\centering
	\caption{Table of notation use throughout the paper for the tabular setting}
	\begin{tabular}{@{}l|l@{}}
		\toprule
		\thead{Notation} & \thead{Meaning} \\ \midrule
	$\cS$ & state space of size $S$\\
	$\cA$ & action space of size $A$\\
	$H$ & length of one episode\\
	$T$ & number of episodes\\
	$J$ & number of posterior samples\\
	\hline
	$r_h(s,a)$ & reward \\
	$p_h(s'|s,a)$ & probability transition \\
	$Q^{\pi}_h(s,a)$ & Q-function of a given policy $\pi$ at step $h$\\
	$V^{\pi}_h(s)$ & V-function of a given policy $\pi$ at step $h$\\
	$\Qstar_h(s,a)$ & optimal Q-function at step $h$\\
	$\Vstar_h(s)$ & optimal V-function at step $h$ \\
	$\regret^T $ & regret \\
	\hline
	$n_0$ and $n_0(k)$ & number of pseudo-transitions \\
	$s_0$ & optimistic pseudo-state \\
	$\ur$ & pseudo-reward \\
	$\kappa$ & posterior inflation parameter \\
	\hline
	$s^{\,t}_h$ & state that was visited at $h$ step during $t$ episode \\
	$a^{\,t}_h$ & action that was picked at $h$ step during $t$ episode \\
    $B^{\,t}_h$ & a ball that contains a pair $(s^t_h, a^t_h)$ \\
	$n_h^t(s,a)$ & number of visits of state-action at the beginning of episode $t$ \\
    & $n_h^t(s,a) = \sum_{k = 1}^{t-1}  \ind{\left\{(s_h^k,a_h^k) = (s,a)\right\}}$\\
    $n_h^t(B)$ & number of visits of a ball $B$ at the beginning of episode $t$ \\
%    $n_h^{t}(s'|s,a)$ & number of transition to $s'$ from state-action at the beginning of episode $t$\\
%    & $n_h^t(s'|s,a) = \sum_{k = 1}^{t-1}  \ind{\left\{(s_h^k,a_h^k, s_{h+1}^k) = (s,a,s')\right\}}$. \\
    \hline 
    $e_k$ & length of $k$-th stage $e_k = \lfloor (1+1/H)^k H \rfloor$ for $k \geq 0$ and $e_{-1} = 0$ \\
    $k^t_h(s,a)$ & index of stage previous to time $t$ at step $h$ and state-action pair $(s,a)$: \\
    & $k^t_h(s,a) = \max\{ k: n^{t}_h(s,a) \geq \sum_{i=0}^k e_i \}$\\
    $\tn_h^t(s,a)$ & number of visits of state-action during the current stage: \\
    & $\tn^t_h(s,a) = n^t_h(s,a) - \sum_{i=0}^{k^t_h(s,a) }e_{i}$ \\
    $\tn_h^t(B)$ & number of visits of a ball $B$ during the current stage: \\
    \hline 
    $\uV^t_h(s)$ & upper approximation of the optimal V-value \\
    $\uQ^t_h(s,a)$ & upper approximation of the optimal Q-value \\
    $\uQ^t_h(B)$ & upper approximation of the optimal Q-value for all $(s,a) \in B$ \\
    $\tQ^{t,j}_h(s,a)$ & temporary estimate of the optimal Q-value \\
    $\tQ^{t,j}_h(B)$ & temporary estimate of the optimal Q-value for all $(s,a) \in B$ \\
    $w_{n,j}$ & random learning rates \\ 
    \hline 
    $\rho_\cS, \rho_\cA, \rho$ & metrics on $\cS, \cA$ and $\cS \times \cA$ correspondingly \\
    $\cN_\varepsilon$ & minimal $\varepsilon$-cover if $\cS \times \cA$ of size $N_{\varepsilon}$ \\
    $d_c$ & covering dimension of space $\cS \times \cA$: $\forall \varepsilon > 0 : N_\varepsilon \leq C_N \varepsilon^{-d_c}$ \\
    $d_{\max}$ & diameter of $\cS \times \cA$ \\
    $F_h(s,a,\xi_h)$ & reparametrization function $s_{h+1} \sim p_h(s,a) \iff s_{h+1} = F_h(s,a, \xi_h)$ \\
    $L_r, L_F$ & Lipschitz constants of rewards and reparametrization function \\
    $L_V$ & Lipschitz constants of $\Qstar_h$ and $\Vstar_h$ \\
     \bottomrule
	\end{tabular}
\end{table}

Let $(\Xset,\Xsigma)$ be a measurable space and $\Pens(\Xset)$ be the set of all probability measures on this space. For $p \in \Pens(\Xset)$ we denote by $\E_p$ the expectation w.r.t.\,$p$. For random variable $\xi: \Xset \to \R$ notation $\xi \sim p$ means $\operatorname{Law}(\xi) = p$. We also write $\E_{\xi \sim p}$ instead of $\E_{p}$. For independent (resp.\,i.i.d.) random variables $\xi_\ell \mysim p_\ell$ (resp.\,$\xi_\ell \mysimiid p$), $\ell = 1, \ldots, d$, we will write $\E_{\xi_\ell \mysim p_\ell}$ (resp.\,$\E_{\xi_\ell \mysimiid p}$), to denote expectation w.r.t.\,product measure on $(\Xset^d, \Xsigma^{\otimes d})$. For any $x \in \Xset$ we denote $\delta_x$ a Dirac measure supported at point $x$.

For any $p, q \in \Pens(\Xset)$ the Kullback-Leibler divergence $\KL(p, q)$ is given by
$$
\KL(p, q) \triangleq \begin{cases}
\E_{p}\left[\log \frac{\rmd p}{\rmd q}\right], & p \ll q, \\
+ \infty, & \text{otherwise.}
\end{cases} 
$$
For any $p \in \Pens(\Xset)$ and $f: \Xset \to \R$, $p f = \E_p[f]$. In particular, for any $p \in \simplex_d$ and $f: \{0, \ldots, d\}   \to  \R$, $pf =  \sum_{\ell = 0}^d f(\ell) p(\ell)$. Define $\Var_{p}(f) = \E_{s' \sim p} \big[(f(s')-p f)^2\big] = p[f^2] - (pf)^2$. For any $(s,a) \in \cS$, transition kernel $p(s,a) \in \Pens(\cS)$ and $f \colon \cS \to \R$ define $pf(s,a) = \E_{p(s,a)}[f]$ and $\Var_{p}[f](s,a) = \Var_{p(s,a)}[f]$.

Let $(\Xset, \rho)$ be a metric space, then the 1-Wasserstein distance between $p,q \in \Pens(\Xset)$ is defined as $\cW_1(p,q) = \sup_{f \text{ is } 1 \text{-Lipschitz}} \E_p[f] - \E_q[f]$.

We write $f(S,A,H,T) = \cO(g(S,A,H,T,\delta))$ if there exist $ S_0, A_0, H_0, T_0, \delta_0$ and constant $C_{f,g}$ such that for any $S \geq S_0, A \geq A_0, H \geq H_0, T \geq T_0, \delta < \delta_0, f(S,A,H,T,\delta) \leq C_{f,g} \cdot g(S,A,H,T,\delta)$. We write $f(S,A,H,T,\delta) = \tcO(g(S,A,H,T,\delta))$ if $C_{f,g}$ in the previous definition is poly-logarithmic in $S,A,H,T,1/\delta$.

For $\alpha, \beta > 0,$ we define $\Beta(\alpha, \beta)$ as a beta distribution with parameters $\alpha, \beta$. For set $\Xset$ such that $\vert \Xset \vert < \infty$ define $\Unif(\Xset)$ as a uniform distribution over this set. In particular, $\Unif[N]$ is a uniform distribution over a set $[N]$.

For a measure $p \in \Pens([0,b])$ supported on a segment $[0,b]$ (equipped with a Borel $\sigma$-algebra) and a number $\mu \in [0,b]$ we define
 \begin{equation*}
    \Kinf(p, \mu) \triangleq \inf\left\{  \KL(p,q): q \in \Pens([0,b]), p \ll q, \E_{X \sim q}[X] \geq \mu \right\}\,.
 \end{equation*}
As the Kullback-Leibler divergence this quantity admits a variational formula by Lemma~18 of \citet{garivier2018kl} up to rescaling for any $u \in (0, b)$
\begin{equation*}
    \Kinf(p, \mu) = \max_{\lambda \in[0,1/(b-\mu)]} \E_{X\sim p}\left[ \log\left( 1-\lambda (X-\mu)\right)\right] \,.
\end{equation*}
\section{Description of \texorpdfstring{\RandQL}{RandQL}}
\label{app:description_randQL}

In this appendix we describe \RandQL and \SampledRandQL algorithms.

\subsection{\RandQL algorithm}
 We recall that $n^t_h(s,a) = \sum_{i=1}^{t-1} \ind\{ (s^i_h, a^i_h) = (s,a) \}$ is the number of visits of state-action pair $(s,a)$ at step $h$ before episode $t$.

We start by initializing the ensemble of Q-values, the policy Q-values, and values to an optimistic value $\tQ_h^{t,j}(s,a) = \uQ_h^{1}(s,a) = \uV^1_h(s,a) = r_h(s,a)+r_0 (H-h)$ for all $(j,h,s,a)\in[J]\times[H]\times\cS\times\cA$ and $r_0>0$ some pseudo-rewards.

At episode $t$ we update the ensemble of Q-values as follows, denoting by $n=n^t_h(s,a)$ the count, $w_{j,n} \sim \Beta(H, n)$ the independent learning rates, 
\[
\tQ^{t+1,j}_h(s,a) = \begin{cases}
        (1- w_{j,n}) \tQ^{t,j}_h(s,a) + w_{j,n} \rQ_{h}^{t,j}(s,a), & (s,a) = (s^t_h, a^t_h) \\
        \tQ^{t,j}_h(s,a) & \text{otherwise},
    \end{cases}
\]
where we defined the target $\rQ_{h}^{t,j}(s,a)$ as a mixture between the usual target and some prior target with mixture coefficient $\rw_{n,j}\sim\Beta(n, n_0)$ and $n_0$ the number of prior samples,
\[
\rQ_{h}^{t,j}(s,a) = \rw_{j,n} [r_h(s,a) + \uV^{t}_{h+1}(s^t_{h+1})] + (1-\rw_{j,n}) [ r_h(s,a) + r_0 (H-h-1)]\,.
\]
It is important to note that in our approach, we need to re-inject prior targets to avoid forgetting their effects too quickly due to the aggressive learning rate. Indeed, the exponential decay of the prior effect can hurt exploration. We observe that the ensemble Q-value only averages uniformly over the last $1/H$ fraction of the targets, as the expected value of the learning rate is $\E[w_{j,n}] = H/(n+H)$. Since $\E[1-\rw_{j,n}] = n_0(n+n_0)$ the weight put on the prior sample in expectation, when we unfold the definition of $\tQ_h^{t+1,j}$, is of order $H/n \cdot n/H \cdot n_0/(n+n_0) =  n_0/(n+n_0)$,  which is consistent with the usual prior forgetting in Bayesian learning.
In \StagedRandQL, we avoid forgetting the prior too quickly by resetting the temporary Q-value to a prior value at the beginning of each stage. 

The policy Q-values are obtained by taking the maximum among the ensemble of Q-values
\[
\uQ_h^{t+1}(s,a) = \max_{j\in[J]} \tQ_h^{t+1,j}(s,a)\,.
\]
The policy is then greedy with respect to the policy Q-values  $\pi_h^{t+1}(s) \in\argmax_{a\in\cA} \uQ_h^{t+1}(s,a)$ and the value is $\uV^{t+1}_h(s)=\max_{a\in\cA}  \uQ_h^{t+1}(s,a)$. The complete \RandQL procedure is detailed in Algorithm~\ref{alg:RandQL}.

\begin{algorithm}[h!]
\centering
\caption{\RandQL}
\label{alg:RandQL}
\begin{algorithmic}[1]
  \STATE {\bfseries Input:} $J$ ensemble size, number of prior transitions $n_0$, prior reward $r_0$.
  \STATE {\bfseries Initialize: }  $\uV_h(s) = \uQ_h(s,a) = \tQ^j_h(s,a) = r(s,a) + \ur(H-h),$ initialize counters $n_h(s,a) = 0$ for $h,s,a\in[H]\times\cS\times\cA$.
      \FOR{$t \in[T]$}
      \FOR{$h \in [H]$}
        \STATE Play $a_h \in \argmax_a \uQ_h(s_h,a)$.
        \STATE Observe reward and next state $s_{h+1}\sim p_h(s_h,a_h)$.
        \STATE Sample $\rw_j \sim \Beta(n, n_0)$ for $n = n_h(s_h,a_h)$.
        \STATE Build targets for all $j \in [J]$
        \[
        \rQ_{h}^{j} = \rw_{j} [r_h(s_h,a_h) + \uV_{h+1}(s_{h+1})] + (1-\rw_{j}) [ r_h(s_h,a_h) + r_0 (H-h)]
        \,.\]
        \vspace{-0.4cm}
        \STATE Sample learning rates $w_j \sim \Beta(H, n)$.
        \STATE Update ensemble $Q$-functions for all $j \in [J]$
        \[
            \tQ^{j}_h(s_h,a_h) := (1 - w_j) \tQ^{j}_h(s_h,a_h) + w_j \rQ_{h}^{j}\,.
        \]
        \vspace{-0.4cm}
        \STATE Update policy $Q$-function $\uQ_h(s_h,a_h) := \max_{j \in [J]} \tQ^{j}_h(s_h,a_h)$.
        \STATE Update value function $\uV_h(s_h) := \max_{a \in \cA} \uQ_h(s_h,a)$\,.
      \ENDFOR
  \ENDFOR
\end{algorithmic}
\end{algorithm}

%\newpage
\subsection{\SampledRandQL algorithm}

To create an algorithm that is more similar to \PSRL, it is possible to select a Q-value at random from the ensemble of Q-values, rather than using the maximum Q-value
\[
\uQ_h^{t}(s,a) = \tQ_h^{t,j_t}(s,a)\qquad \text{with } j_t \sim\Unif[J].
\]
In this case we also need to update each Q-value in the ensemble with its corresponding target, see \citet{osband2015bootstrap}, 
\[
\rQ_{h}^{t,j}(s,a) = \rw_{j,n} [r_h(s,a) + \tV^{t,j}_{h+1}(s^t_{h+1})] + (1-\rw_{j,n}) [ r_h(s,a) + r_0 (H-h-1)]
\]
where $\tV^{t,j}_h(s)=\max_{a\in\cA}  \tQ_h^{t,j}(s,a)$. We name this new procedure \SampledRandQL and detail it in Algorithm~\ref{alg:SampledRandQL}.

\begin{algorithm}[h!]
\centering
\caption{\SampledRandQL}
\label{alg:SampledRandQL}
\begin{algorithmic}[1]
  \STATE {\bfseries Input:} $J$ ensemble size, number of prior transitions $n_0$, prior reward $r_0$.
  \STATE {\bfseries Initialize: }  $\uV_h(s) = \uQ_h(s,a) = \tQ^j_h(s,a) = r(s,a) + \ur(H-h),$ initialize counters $n_h(s,a) = 0$ for $h,s,a\in[H]\times\cS\times\cA$.
      \FOR{$t \in[T]$}
      \STATE Sample ensemble index $i\sim \Unif[J]$
      \FOR{$h \in [H]$}
        \STATE Play $a_h \in \argmax_a \uQ_h(s_h,a)$.
        \STATE Observe reward and next state $s_{h+1}\sim p_h(s_h,a_h)$.
        \STATE Sample $\rw_j \sim \Beta(n, n_0)$ for $n = n_h(s_h,a_h)$.
        \STATE Build targets for all $j \in [J]$
        \[
        \rQ_{h}^{j} = \rw_{j} [r_h(s_h,a_h) + \tV^j_{h+1}(s_{h+1})] + (1-\rw_{j}) [ r_h(s_h,a_h) + r_0 (H-h)]
        \,.\]
        \vspace{-0.4cm}
        \STATE Sample learning rates $w_j \sim \Beta(H, n)$.
        \STATE Update ensemble $Q$-functions for all $j \in [J]$
        \[
            \tQ^{j}_h(s_h,a_h) := (1 - w_j) \tQ^{j}_h(s_h,a_h) + w_j \rQ^j_h\,.
        \]
        \vspace{-0.4cm}
        \STATE Update value function $\uV_h(s_h) := \max_{a \in \cA} \tQ_h^j(s_h,a)$ for all $j\in[J]$.
        \STATE Update policy $Q$-function $\uQ_h(s_h,a_h) :=\tQ^{i}_h(s_h,a_h)$.

      \ENDFOR
  \ENDFOR
\end{algorithmic}
\end{algorithm}

%\newpage

\newpage
\section{Weight Distribution in \RandQL}
\label{app:weights_randql}

In this section we study the joint distribution of weights over all targets in \RandQL algorithm, described in details in Appendix~\ref{app:description_randQL}. To do it, we describe a very useful distribution, defined by \cite{wong1998generalized}.

\begin{definition}
    We say that a random vector $(X_1,\ldots,X_n,X_{n+1})$ has a \textit{generalized Dirichlet distribution} $\GD(\alpha_1,\ldots,\alpha_n;\beta_1,\ldots,\beta_n) $ if $X_{n+1} = 1 - (X_1 + \ldots+ X_n)$ and $(X_1,\ldots,X_n)$ it has the following density over the simplex $\{x_1,\ldots,x_{n} : x_1 + \ldots + x_{n} \leq 1\}$, 
    \[
        p(x) = \prod_{i=1}^n \frac{1}{B(\alpha_i, \beta_i)} x_i^{\alpha_i-1} (1 - x_1 - \ldots - x_i)^{\gamma_i}
    \]
    for $x_1 + \ldots + x_n \leq 1, x_j \geq 0$ for $j=1,\ldots,n,$ and $\gamma_j = \beta_j - \alpha_{j+1} - \beta_{j+1}$ for $j=1,\ldots,n-1$ and $\gamma_n = \beta_n - 1$. If we set $x_{n+1} = 1-(x_1 + \ldots + x_n)$ then we  obtain a homogeneous formula
    \[
        p(x) = \prod_{i=1}^n \frac{1}{B(\alpha_i, \beta_i)} x_i^{\alpha_i-1} \left( \sum_{j=i+1}^{n+1} x_j \right)^{\gamma_i}
    \]
\end{definition}
Alternative characterization of generalized Dirichlet distribution could be given using independent beta-distributed random variables $Z_1,\ldots,Z_n$ with $Z_i \sim \Beta(\alpha_i,\beta_i)$ as follows
\begin{align*}
    X_1 &= Z_1, \\
    X_j &= Z_j(1 - X_1 - \ldots - X_{j-1}) = Z_j \prod_{i=1}^{j-1}(1-Z_i) \quad \text{for } j = 2,3,\ldots,n \\
    X_{n+1} &= 1 - X_1 - \ldots - X_n = \prod_{i=1}^n (1 - Z_i)
\end{align*}

Therefore, for \RandQL algorithm without prior re-injection we have the following formula
\[
    \tQ^{t,j}_h(s,a) =  \sum_{i=0}^{n^t_h(s,a)} W^{i}_{j, n} \left( r_h(s^{\ell^i}_h,a^{\ell^i}_h) +  \uV^{\ell^i}_{h+1}(s^{\ell^i}_{h+1}) \right),
\]
for $n = n^t_h(s,a)$ and  weights are defined as follows
\[
    W^{0}_{j, n} = \prod_{q=0}^{n-1} ( 1 - w_{j,q}), \quad W^i_{j,n} = w_{j,i-1} \cdot \prod_{q=i}^{n-1} (1 - w_{j,q}),\ i \geq 1.
\]
And, moreover, we have that this vector of weights has the \textit{generalized Dirichlet distribution}
\[
    (W^n_{n,j}, W^{n-1}_{n,j}, \ldots, W^1_{n,j}, W^0_{n,j}) \sim \GD(H, H, \ldots, H; n+n_0,\ldots, n_0+1, n_0).
\]
That is, weights generated by the \RandQL procedure is an inverted generalized Dirichlet random vector, that induces additional similarities with a usual posterior sampling approaches. Notably, that for $H=1$ we recover exactly usual Dirichlet distribution, as in the setting of \StagedRandQL.

In the setting of the analysis, the main feature of this distribution is asymmetry in attitude to the order of components. In particular, the expectation of the prior weight $W^0_{n,j}$ is $\prod_{i=1}^n \left( 1 - \frac{H}{i+H} \right) \sim n^{-H}$ that leads to too rapid forgetting of the prior information.

% Algorithm proofs part
\newpage
\section{Proofs for Tabular algorithm}\label{app:randql_tabular_proof}

\subsection{Algorithm}
In this section we describe in detail the tabular algorithms and the ways we will analyze them. We also provide some notations that will be used in the sequel. 

Let $n^t_h(s,a)$ be the number of visits of $(s,a,h)$ (i.e., of the state-action pair $(s,a)$ at step $h$) at the beginning of episode $t$: $n^t_h(s,a) = \sum_{i=1}^{t-1} \ind\{ (s^i_h, a^i_h) = (s,a) \}$. In particular, $n^{T+1}_h(s,a)$ is the number of visits of $(s,a,h)$ after all episodes.

Let $e_k = \lfloor (1 + 1/H)^k \cdot H \rfloor$ be the length of each stage for any $k \geq 0$ and, by convention, $e_{-1} = 0$. %Let $\cL = \{ \sum_{i=0}^{k-1} e_{i} \}_{k\in\N}$ be the sequence of stages starting points. 
We will say that at the beginning of episode $t$ a triple $(s,a,h)$ is in $k$-th stage if $n^t_h(s,a) \in [\sum_{i=0}^{k-1} e_i, \sum_{i=0}^k e_i )$. %For an episode $t+1$ we will call $k^t_h(s,a)$ the index of the previous stage.

Let $\tn^t_h(s,a)$ be the number of visits of state-action pair during the current stage at the beginning of episode $t$. Formally, it holds $\tn^t_h(s,a) = n^t_h(s,a) - \sum_{i=0}^{k-1} e_i$, where $k$ is the index of current stage. 

%Also we define $\upn^t_h(s,a) = \tn^t_h(s,a) + n_0$ as a pseudo-counts.

Let $\kappa > 0$ be the posterior inflation coefficient, $n_0$ be the number of prior transitions, and $J$ be the number of temporary $Q$-functions. Let $\tQ^{t,j}_h$ be the $j$-th \textit{temporary} Q-function and $\uQ^t_h$ be the \textit{policy} Q-function at the beginning of episode $t$. We initialize them as follows
\[
    \uQ^1_h(s,a) = r_h(s,a) + \ur (H - h - 1), \quad \tQ^{1,j}_h(s,a) = r_h(s,a) +  \ur (H-h-1),
\]
We can treat this initialization as a setting prior over $n_0$ pseudo-transitions to artificial state $s_0$ with $\ur > 1$ reward for each interaction.

For each transition we perform the following update of temporary Q-functions
% \pier{It is $w_{\tn+1}$ or $w_{\tn}$? (In general I think it is a bad idea to give count at the beginning of an episode it always lead to complicated notations)}
% \dan{The problem here is that the temporary counter reset in the end of stage, so formally $\tn^{\tau^t_h(s,a)}_h(s,a) = 0$ always, not $e_k$ as we want, so I suggest numeration from zero for weights. And also it is  the reason why I introduced this $\tQ^{t+1/2}$. This stage-based framework is a disaster from the "notational" point of view. }
\begin{equation}\label{eq:tilde_Q_update}
    \tQ^{t+1/2,j}_h(s,a) = \begin{cases}
        (1- w^k_{j, \tn}) \cdot \tQ^{t,j}_h(s,a) + w^k_{j, \tn} [r_h(s,a) + \uV^{t}_{h+1}(s^t_{h+1})], & (s,a) = (s^t_h, a^t_h) \\
        \tQ^{t,j}_h(s,a) & \text{otherwise},
    \end{cases}
\end{equation}
where $\tn=\tn^t_h(s,a)$ is the number of visits of $(s,a,h)$ during the current stage at the beginning of episode $t$, $k$ is the index of the current stage, and $w^k_{j, \tn}$ is a sequence of independent beta-distribution random variables $w^k_{j,\tn} \sim \Beta(1/\kappa, (\tn + n_0) / \kappa)$. Here we slightly abuse the notation by dropping the dependence of weights $w^k_{j,\tn}$ on the triple $(h,s,a)$ in order to simplify the exposition. In the case that the explicit dependence is required, we will call these weights as $w^{k,h}_{j,\tn}(s,a)$.

Next we define the stage update as follows
\begin{align*}
    \uQ^{t+1}_h(s,a) &= \begin{cases}
        \max_{j\in[J]}  \tQ^{t+1/2,j}_h(s,a) & \tn^t_h(s,a) = \lfloor (1 + 1/H)^k H \rfloor \\
        \uQ^{t}_h(s,a) & \text{otherwise}
    \end{cases} \\
    \tQ^{t+1,j}_h(s,a) &= \begin{cases}
         r_h(s,a) + \ur (H-h+1)  & \tn^t_h(s,a) = \lfloor (1 + 1/H)^k H \rfloor \\
         \tQ^{t+1/2,j}_h(s,a) & \text{otherwise}
    \end{cases} \\
    \uV^{t+1}_h(s) &= \max_{a \in \cA} \uQ^{t+1}_h(s,a) \\
    \pi^{t+1}_h(s) &\in \argmax_{a \in \cA} \uQ^{t+1}_h(s,a),
\end{align*}
where $k$ is the current stage. In other words, we update $\uQ^{t+1}$ with temporary values of $\tQ^{t+1/2,j}$, and then, if the change of stage is triggered, reinitialize $\tQ^{t+1,j}_h(s,a)$ for all $j$. For episode $t$ we will call $k^t_h(s,a)$ the index of stage where $\uQ^t_h(s,a)$ was updated (and $k^t_h(s,a) = -1$ if there was no update). For all $t$ we define $\tau^t_h(s,a) \leq t$ as an episode when the stage update happens. In other words, for any $t$ the following holds
\[
    \uQ^{t+1}_h(s,a) = \max_{j\in[J]}  \tQ^{\tau^t_h(s,a)+1/2,j}_h(s,a),
\]
where $\tau^t_h(s,a) = 0$ and $e_k = 0$ if there was no updates. To simplify the notation we will omit dependence on $(s,a,h)$ where it is deducible from the context.

To simplify the notation, we can extend the state space $\cS$ by an additional state $s_0$ that will be purely technical and used in the proofs. This state has the prescribed value function $\Vstar_h(s_0) = \ur(H-h)$ and could be treated as a absorbing pseudo-state with reward $\ur$.

We notice that in this case we use $e_k$ samples to compute $\tQ^{\tau^t_h(s,a)+1/2,j}$ for $k = k^t_h(s,a)$. For this $k$ we define $\ell^i_{k,h}(s,a)$ as a time of $i$-th visit of state-action pair $(s,a)$ during $k$-th stage. Then we have the following decomposition
\begin{equation}\label{eq:temp_q_rollout_}
    \tQ^{\tau^t+1/2,j}_h(s,a) = r_h(s,a) + \sum_{i=0}^{e_k} W^{i}_{j, e_k, k} \uV^{\ell^i}_{h+1}(s^{\ell^i}_{h+1}),
\end{equation}
where we drop dependence on $k$ and $(s,a,h)$ in $\ell^i$ to simplify notations, and use the convention $s^{\ell^0_{k,h}(s,a)}_{h+1} = s_0$, and the following aggregated weights
\[
    W^{0}_{j, n, k} = \prod_{q=0}^{n-1} ( 1 - w^k_{j,q}), \quad W^i_{j,n,k} = w^k_{j,i-1} \cdot \prod_{q=i}^{n-1} (1 - w^k_{j,q}),\ i \geq 1.
\]
We will omit the dependence on the stage index $k$ when it is not needed for the statement. However, we notice that these vectors, for different stages $k$, will be independent.

By the properties of the generalized Dirichlet distribution, it is possible to show the following result
\begin{lemma}\label{lem:weights_distribution}
    For any fixed $n > 0$, the random vector $(W^{0}_{j,n},W^{1}_{j,n}, \ldots, W^{n}_{j,n})$ has a Dirichlet distribution $\Dir(n_0/\kappa, 1/\kappa, \ldots, 1/\kappa)$.
\end{lemma}
\begin{proof}
Using the Dirichlet random variate generation from marginal beta distributions, it is sufficient to prove that
for all $i\in \{0,\dots,n\}$, $W^{n-i}_{j,n,k} = (1-W^{n}_{j,n,k} -\dots- W^{n-i+1}_{j,n,k}) w_{j,n-i-1}^k,$ with the convention $w_{j,-1}^k = 1$. This is trivial for $i=0$, as $W^{n}_{j,n,k} = w_{j,n-1}^k$.
Now, if this is true for some $i$, then, for $i+1\in \{0,\dots,n\}$, we have
\begin{align*}W^{n-i-1}_{j,n,k} &= w_{j,n-i-2}^k \prod_{q=n-i-1}^{n-1}(1-w_{j,q}^k)\\&=w_{j,n-i-2}^k(1-w_{j,n-i-1}^k)  
(1-W^{n}_{j,n,k} -\dots- W^{n-i+1}_{j,n,k}) 
\\&= w_{j,n-i-2}^k (1-W^{n}_{j,n,k} -\dots- W^{n-i+1}_{j,n,k} - \underbrace{w_{j,n-i-1}^k(1-W^{n}_{j,n,k} -\dots- W^{n-i+1}_{j,n,k})}_{=W^{n-i}_{j,n,k}}),
\end{align*}
which finishes the proof.
\end{proof}

Notably, the expression \eqref{eq:temp_q_rollout_} shows a significant similarity between our method and \OPSRL. It is the reason why we can call this method a model-free posterior sampling, where posterior sampling is performed over the model in a lazy and model-free fashion.

\subsection{Concentration}\label{app:concentration_tabular}

Let $\betastar\colon (0,1) \times \N \to \R_{+}$ and $\beta^{B}, \beta^{\conc}, \beta\colon (0,1) \to \R_{+}$ be some function defined later on in Lemma \ref{lem:proba_master_event}. We define the following favorable events

\begin{align*}
  \cE^\star(\delta) &\triangleq \Bigg\{\forall t \in \N, \forall h \in [H], \forall (s,a)\in\cS\times\cA, k = k^t_h(s,a): \\
  &\qquad
    \Kinf\left( \frac{1}{e_k}\sum_{i=1}^{e_k} \delta_{\Vstar_{h+1}(s^{\ell^i}_{h+1})} ,p_h \Vstar_{h+1}(s,a) \right) \leq  \frac{\betastar(\delta,e_k)}{e_k}\Bigg\}\,,\\
    \cE^{B}(\delta) &\triangleq \Bigg\{ \forall t \in [T], \forall h \in [H], \forall (s,a) \in \cS \times \cA, \forall j \in [J], k = k^t_h(s,a):  \\
    &\qquad \left| \sum_{i=0}^{e_k} \left( W^i_{j, e_k, k} - \E[W^i_{j, e_k,k}] \right) \uV^{\ell^i}_{h+1}(s^{\ell^i}_{h+1}) \right| \leq 60 \rme^2 \sqrt{\frac{\ur^2 H^2 \kappa \beta^B(\delta)}{e_k + n_0}} \\
    &\qquad\qquad\qquad\qquad\qquad\qquad\qquad\qquad\quad + 1200 \rme \frac{\ur H \kappa \log(e_k) (\beta^B(\delta))^2}{e_k + n_0} \bigg\}\,,\\
\cE^{\conc}(\delta) &\triangleq \Bigg\{\forall t \in [T], \forall h \in [H], \forall (s,a)\in\cS\times\cA, k = k^t_h(s,a): \\
&\qquad\left|\frac{1}{e_k} \sum_{i=1}^{e_k} \Vstar_{h+1}(s^{\ell^i_{k,h}(s,a)}_{h+1})  - p_h\Vstar_{h+1}(s,a) \right| \leq \sqrt{\frac{2\ur^2 H^2 \beta^{\conc}(\delta)}{e_k}}\Bigg\}\\
\cE(\delta) &\triangleq \Bigg\{  \sum_{t=1}^T \sum_{h=1}^H (1+1/H)^{H-h}\left| p_h[\Vstar_{h+1} - V^{\pi_t}_{h+1}](s^t_h, a^t_h) - [\Vstar_{h+1} - V^{\pi_t}_{h+1}](s^t_{h+1})\right| \\
&\qquad\qquad\qquad\qquad\qquad\qquad\qquad\qquad\qquad\qquad\qquad\quad\leq 2\rme \ur H\sqrt{2HT \beta(\delta)}.
\Bigg\}.
\end{align*}
We also introduce the intersection of these events, $\cG(\delta) \triangleq \cE^\star(\delta) \cap \cE^{B}(\delta) \cap \cE^{\conc}(\delta) \cap \cE(\delta)$. We  prove that for the right choice of the functions $\betastar,  \beta^{\KL}, \beta^{\conc}, \beta, \beta^{\Var}$ the above events hold with high probability.
\begin{lemma}
\label{lem:proba_master_event}
For any $\delta \in (0,1)$ and for the following choices of functions $\beta,$
\begin{align*}
    \betastar(\delta,n) &\triangleq \log(8SAH/\delta) + 3\log\left(\rme\pi(2n+1)\right)\,,\\
    \beta^B(\delta) &\triangleq  \log(8SAH/\delta) + \log(TJ)\,,\\
    \beta^{\conc}(\delta) &\triangleq \log(8SAH/\delta) + \log(2T) ,\\
    \beta(\delta) &\triangleq \log\left(16/\delta\right),
\end{align*}
it holds that
\begin{align*}
\P[\cE^\star(\delta)]&\geq 1-\delta/8, \qquad \P[\cE^{B}(\delta)]\geq 1-\delta/8,  \\\
\P[\cE^\conc(\delta)] &\geq 1-\delta/8, \qquad \P[\cE(\delta)]\geq 1-\delta/8.
\end{align*}
In particular, $\P[\cG(\delta)] \geq 1-\delta/2$.
\end{lemma}
\begin{proof}
From the fact that $s^{\ell^i}_{h+1}$ are i.i.d. generated from $p_h(s,a)$, Theorem \ref{th:max_ineq_kinf}, and union bound $\cS \times \cA \times [H]$ it holds $\P[\cE^\star(\delta)]\geq 1-\delta/8$. 

Next we fix all $t,h,s,a,j$, and denote $n = e_{k^t_h(s,a)}$. First, we define a filtration of $\sigma$-algebras $\cF_{\tau}$ that is sigma-algebra generated by all random variables appeared untill the update \eqref{eq:tilde_Q_update} in the episode $t$ and step $h$, before newly generated random weights but after receiving new state $s^t_{h+1}$. Formally, we can define it as follows
\begin{align*}
    \cF_{t,h} &= \sigma\bigg( \left\{ (s^\tau_{h'}, a^\tau_{h'}, w^{k^{\tau}_{h'}+1, h'}_{j, \tn^\tau_{h'}}(s^\tau_{h'}, a^\tau_{h'}) ), \forall \tau < t, (h',j) \in [H] \times [J] \right\} \\
    &\qquad \cup \{ (s^t_{h'}, a^t_{h'}, s^t_{h'+1}), \forall h' \leq h   \} \cup \{ w^{k^t_{h'}+1, h'}_{j, \tn^{t}_{h'}}(s^t_{h'}, a^t_{h'}), \forall h' < h, j \in [J] \}\bigg),
\end{align*}
where we drop dependence on state-action pairs everywhere where it is deducible from the context.

Consider a sequence $\ell^1 < \ldots < \ell^n$ be an excursion of the state-action pair $(s,a)$ at the step $h$. Each $\ell^i$ is a stopping time w.r.t $\cF_{t,h}$, so we can consider a stopped filtration (with a shift by 1 in indices) $\widetilde{\cF}_{i-1} = \cF_{\ell^i,h}$.  In other words, this filtration at time-steps $i-1$ contains all the information that is available just before the generation of random weights for the $i$-th update of temporary Q-functions inside the last stage. We notice that under this definition, we have
\begin{align*}
    \E[\uV^{\ell_i}_{h+1}(s^{\ell_i}_{h+1}) | \widetilde{\cF}_{i-1}] &= \uV^{\ell_i}_{h+1}(s^{\ell_i}_{h+1})\,, \qquad
    \E[w_{j,n}^i | \widetilde{\cF}_{i-1}] =  \E[w_{j,n}^i]\,
\end{align*}
and, moreover, $w_{j,n}^i$ is $\widetilde{\cF}_i$-measurable. To simplify the notation, we define $Y_i = \uV^{\ell_i}_{h+1}(s^{\ell_i}_{h+1}) / (r_0 \cdot H)$, and then we notice that we can apply the recursion on the aggregated weights back: define two Q-value-style sequences
\[
    X_i = (1-w_{j,n}^i) X_{i-1} + w_{j,n}^i Y_i\,, \quad \bar{X}_i = (1-\E[w_{j,n}^i]) \bar{X}_{i-1} + \E[w_{j,n}^i] Y_i\,.
\]
Then, by the aggregation property, we have
\[
    (r_0 \cdot H) \cdot (X_n - \bar{X}_n) = \sum_{i=0}^{n} \left( W^i_{j, n,k} - \E[W^i_{j, n,k}] \right) \uV^{\ell^i}_{h+1}(s^{\ell^i}_{h+1})\,.
\]
Given this reformulation, Proposition~\ref{prop:beta_martingale_bound} and union bound implies $\P[\cE^B(\delta)] \geq 1 -\delta/8$.

To show that $\PP{\cE^{\conc}(\delta)} > 1-\delta/8$, it is enough to apply Hoeffding inequality for a fixed number of samples $e_k$ used in empirical mean, and then use union bound of all possible values of $(s,a,h) \in \cS \times \cA \times [H]$ and $e_k \in [T]$.

Next, define the following sequence
\begin{align*}
Z_{t,h} &\triangleq (1+1/H)^{H-h}\left([\Vstar_{h+1}-V^{\pi^t}_{h+1}](s_{h+1}^t)-p_h [\Vstar_{h+1}- V^{\pi^t}_{h+1}](s^t_h,a^t_h)\right),& t\in [T], h \in [H],
\end{align*}
It is easy to see that these sequences form a martingale-difference w.r.t filtration $\cF_{t,h} = \sigma\left\{ \{ (s^{\ell}_{h'}, a^{\ell}_{h'}, \pi^{\ell}), \ell < t, h' \in [H] \} \cup \{ (s^{t}_{h'}, a^t_{h'}, \pi^{t}), h' \leq h \} \right\}$. Moreover,  \(|Z_{t,h}|\leq 2\rme \ur H\) for all \(t\in [T]\) and \(h\in [H].\)   Hence, the Azuma-Hoeffding inequality implies
 \begin{align*}
    \P\Bigl(\Bigl|\sum_{t=1}^T \sum_{h=1}^H Z_{t,h}\Bigr|> 2\rme \ur H\sqrt{2 t H \cdot \beta(\delta)}\Bigr)&\leq 2\exp(-\beta(\delta))=\delta/8,
\end{align*}
therefore $\P[\cE(\delta)] \geq 1 - \delta/8$.
\end{proof}

\subsection{Optimism}\label{app:optimism_tabular}

 In this section we prove that our estimate of $Q$-function $\uQ^{\,t}_h(s,a)$ is optimistic, that is the event
\begin{equation}\label{eq:opt_event}
    \cE_{\opt} \triangleq \left\{ \forall t \in [T], h \in [H], (s,a) \in \cS \times \cA:  \uQ^t_h(s,a) \geq \Qstar_h(s,a) \right\}.
\end{equation}
holds with high probability on the event $\cE^\star(\delta)$.

Define constants
\begin{equation}\label{eq:constant_c0}
    c_0 \triangleq \frac{8}{\pi} \left( \frac{4}{\sqrt{\log(17/16)}} + 8 + \frac{49\cdot 4\sqrt{6}}{9} \right)^2 + 1.
\end{equation}
and
\begin{equation}\label{eq:constant_cJ}
    c_J \triangleq \frac{1}{\log\left( \frac{2}{1 + \Phi(1)} \right)},
\end{equation}
where $\Phi(\cdot)$ is a CDF of a normal distribution.

\begin{proposition}\label{prop:anticonc}
    Assume that $J = \lceil  c_J \cdot \log(2SAHT/\delta)  \rceil$, $\kappa = 2\beta^\star(\delta, T)$, $\ur = 2$, and $n_0 = \lceil (c_0 + 1 + \log_{17/16}(T))  \cdot \kappa \rceil$. Then conditionally on $\cE^\star(\delta)$ the  event
    \begin{align*}
        \cE_{\anticonc} \triangleq \biggl\{ &\forall t \in [T], \ \forall h \in [H], \ \forall (s,a) \in \cS \times \cA: \\
        &\max_{j \in [J]} \left\{ \sum_{i=0}^{e_k} W^i_{j,e_k,k} \Vstar_{h+1}(s^{\ell^i_{t,h}(s,a)}_{h+1})\right\} \geq p_h \Vstar_{h+1}(s,a), k = k^t_h(s,a) \biggl\}
    \end{align*}
    holds with probability at least $1-\delta/2$.
\end{proposition}
\begin{proof}
    Let us fix $t \in [T], h \in [H], (s,a) \in \cS \times \cA$, and $j \in [J]$. 
    By Lemma~\ref{lem:weights_distribution}, we have that the vector $(W^i_{j,e_k,k})_{i=0,\ldots,e_k}$ has Dirichlet distribution.  Note that $\Vstar_{h+1}(s^{\ell^0}_{h+1}) = \ur(H-h-1)$ is an upper bound on $V$-function and  the weight of the first atom is $\alpha_0 \triangleq n_0/\kappa \geq  c_0 + \log_{17/16}(T)$ for $c_0$ defined in \eqref{eq:constant_c0}.
    Define a measure $\bnu_{e_k} = \frac{n_0 - 1}{e_k + n_0 - 1} \delta_{\Vstar_{h+1}(s_0)} + \sum_{i=1}^{e_k} \frac{1}{e_k + n_0 - 1} \delta_{\Vstar_{h+1}(s^{\ell^i}_{h+1})}$. Since $p_h \Vstar_{h+1}(s,a) \leq H-h-1$, we can apply Lemma~\ref{lem:lower_bound_dbc_relaxed} with a fixed $\varepsilon = 1/2$ conditioned on independent samples $\{ s^{\ell_i}_{h+1} \}_{i=1}^{e_k}$ from $p_h(s,a)$
    \begin{align}\label{eq:anticonc_1}
        \begin{split}
        \P\biggl[ \sum_{i=0}^{e_k} W^i_{j,e_k,k} & \Vstar_{h+1}(s^{\ell^i_{t,h}(s,a)}_{h+1}) \geq p_h \Vstar_{h+1}(s,a) \mid \{ s^{\ell_i}_{h+1} \}_{i=1}^{e_k} \biggl] \\
        &\geq \frac{1}{2}\left( 1 - \Phi\left(\sqrt{\frac{2 (e_k + n_0 -\kappa) \Kinf\left(\bnu_{e_k}, p_h \Vstar_{h+1}(s,a)\right) }{\kappa}}\right)\right),
        \end{split}
    \end{align}
    where $\Phi$ is a CDF of a normal distribution. Combining Lemma~\ref{lem:kinf_prior_remove} and the event $\cE^\star(\delta)$
    \begin{align*}
        (e_k + n_0 - \kappa) \Kinf\left( \bnu_{e_k}, p_h \Vstar_{h+1}(s,a)\right) 
        \leq e_k \Kinf\left( \hnu_{e_k}, p_h \Vstar_{h+1}(s,a)\right) \leq \beta^\star(\delta, T),
    \end{align*}
    where $\hnu_{e_k} = \frac{1}{e_k} \sum_{i=1}^{e_k} \delta_{\Vstar_{h+1}(s^{\ell^i}_{h+1})}$, and, as a corollary
    \begin{small}
    \[
        \P\left[ \sum_{i=0}^{e_k} W^i_{j,e_k,k} \Vstar_{h+1}(s^{\ell^i_{t,h}(s,a)}_{h+1}) \geq p_h \Vstar_{h+1}(s,a) \mid \cE^\star(\delta), \{ s^{\ell^i}_{h+1} \}_{i=1}^{e_k}  \right] \geq \frac{1}{2} \left( 1 - \Phi\left( \sqrt{\frac{2\beta^\star(\delta, T) }{\kappa}} \right) \right).
    \]
    \end{small}
    \!\!By taking $\kappa = 2\beta^\star(\delta, T)$ we have a constant probability of being optimistic
    \[
        \PP{ \sum_{i=0}^{e_k} W^i_{j,e_k,k} \Vstar_{h+1}(s^{\ell^i_{t,h}(s,a)}_{h+1}) \geq p_h \Vstar_{h+1}(s,a) \mid \cE^\star(\delta) } \geq \frac{1 - \Phi(1)}{2} \triangleq \gamma.
    \]
    Next, using a choice $J = \lceil \log(2SAHT/\delta) / \log(1/(1-\gamma)) \rceil = \lceil c_J \cdot \log(2SAHT/\delta)  \rceil$ 
    \[
        \P\left[ \max_{j \in [J]}\left\{ \sum_{i=0}^{e_k} W^i_{j,e_k,k} \Vstar_{h+1}(s^{\ell^i_{t,h}(s,a)}_{h+1}) \right\} \geq p_h \Vstar_{h+1}(s,a) \mid \cE^\star(\delta) \right] \geq 1 - (1 - \gamma)^{J} \geq 1 - \frac{\delta}{2SAHT}\cdot
    \]
    By a union bound we conclude the statement.
\end{proof}
Next we provide a connection between $\cE^{\anticonc}$ and $\cE^{\opt}$.
\begin{proposition}
    It holds that $\cE^{\opt} \subseteq \cE^{\anticonc}$.
\end{proposition}

\begin{proof}
    We proceed by a backward induction over $h$. Base of induction $h = H+1$ is trivial. Next by Bellman equations for $\uQ^t_h$ and $\Qstar_h$
    \[
        [\uQ^{t}_h - \Qstar_h](s,a) = \max_{j \in [J]} \left\{ \sum_{i=0}^{n} W^i_{j,n} \uV^{\ell^i}_{h+1}(s^{\ell^i}_{h+1})\right\} - p_h \Vstar_{h+1}(s,a),
    \]
    where $n = e_{k^t_h(s,a)}$ and we drop dependence on $k,t,h,s,a$ in $\ell^i$. By induction hypothesis we have $\uV^{\ell^i}_{h+1}(s') \geq \uQ^{\ell^i}_{h+1}(s', \pi^\star(s')) \geq \Qstar_{h+1}(s', \pi^\star(s')) = \Vstar_{h+1}(s')$ for any $i$, thus
    \[
        [\uQ^{t}_h - \Qstar_h](s,a) \geq \max_{j \in [J]} \left\{ \sum_{i=0}^{n} W^i_{j,n} \Vstar_{h+1}(s^{\ell^i}_{h+1})\right\} - p_h \Vstar_{h+1}(s,a).
    \]
    By the definition of event $\cE^{\anticonc}(\delta)$ we conclude the statement.
\end{proof}

\begin{proposition}[Optimism]\label{prop:optimism}

    Assume that $J = \lceil  c_J \cdot \log(2SAHT/\delta)  \rceil$, $\kappa = 2\beta^\star(\delta, T)$, $\ur = 2$, and $n_0 = \lceil (c_0 + 1 + \log_{17/16}(T)) \cdot \kappa \rceil$, where $c_0$ is defined in \eqref{eq:constant_c0} and $c_J$ is defined in \eqref{eq:constant_cJ}. Then $\PP{\cE^{\opt} \mid  \cE^\star(\delta)} \geq 1-\delta/2$. 
\end{proposition}

\subsection{Regret Bound}\label{app:regret_bound_tabular}

Let us define the main event $\cG'(\delta) = \cG(\delta) \cap \cE^\opt$. On this event we have the following corollary that connects \RandQL with \OptQL with Hoeffding bonuses.

Define the following quantity
\[
    \beta^{\max}(\delta) = \max\left\{ \kappa, n_0/\kappa, \beta^B(\delta), \beta^{\conc}(\delta), \beta(\delta), \log(T+n_0)  \right\} = \cO(\log(SATH/\delta)).
\]

\begin{corollary}\label{cor:delta_q_bounds}
    Assume conditions of Proposition~\ref{prop:optimism} hold. Let $t \in [T], h\in[H], (s,a) \in \cS \times \cA$. Define $k = k^t_h(s,a)$ and let $\ell^1 < \ldots < \ell^{e_k}$ be a excursions of $(s,a,h)$ until the previous stage. Then on the event $\cG'(\delta)$ the following bound holds for $k \geq 0$
    \[
        0\leq \uQ^t_h(s,a) - \Qstar_h(s,a) \leq \frac{1}{n} \sum_{i=1}^n [\uV^{\ell^i}_{h+1}(s^{\ell^i}_{h+1}) - \Vstar_{h+1}(s^{\ell^i}_{h+1}) ]  + \cB^t_h(k),
    \]
    where
    \[
        \cB^t_h(k) = 61\rme^2 \frac{ \ur H  (\beta^{\max}(\delta))}{\sqrt{e_k}} + 1201\rme \frac{\ur H (\beta^{\max}(\delta))^4}{e_k}.
    \]
\end{corollary}
\begin{proof}
    The lower bound follows from the definition of the event $\cE^{\opt}$. For the upper bound we first apply the decomposition for $\uQ^t_h(s,a)$ and the definition of event $\cE^B(\delta)$ from Lemma~\ref{lem:proba_master_event}
    \begin{align*}
        \uQ^t_h(s,a) &= r_h(s,a) + \max_{j \in [J]} \left\{ \sum_{i=0}^{e_k} W^i_{j,e_k} \uV^{\ell^i}_{h+1}(s^{\ell^i}_{h+1}) \right\} \\
        &\leq r_h(s,a) + \frac{1}{e_k + n_0} \sum_{i=1}^{e_k} \uV^{\ell^i}_{h+1}(s^{\ell^i}_{h+1}) + \frac{n_0 \kappa \cdot \ur H}{e_k + n_0} + 60 \rme^2 \sqrt{\frac{\ur^2 H^2 \kappa \beta^B(\delta)}{e_k + n_0}} \\
        &+ 1200\rme \frac{\ur H \kappa \log(e_k + n_0) (\beta^B(\delta))^2}{e_k + n_0}.
    \end{align*}
    Then, by Bellman equations, 
    \begin{align*}
        \uQ^t_h(s,a) - \Qstar_h(s,a) &\leq \frac{1}{e_k} \sum_{i=1}^{e_k}\left[ \uV^{\ell^i}_{h+1} - \Vstar_{h+1} \right](s^{\ell^i}_{h+1}) +  \frac{1}{e_k} \sum_{i=1}^{e_k} \left[ \Vstar_{h+1} (s^{\ell^i}_{h+1})  - p_h \Vstar_{h+1}(s,a)\right] \\
        &+ (1200\rme + 1) \frac{\ur H (\beta^{\max}(\delta))^4}{e_k + n_0} + 60\rme^2 \cdot \frac{\ur H \beta^{\max}(\delta)}{\sqrt{e_k +n_0}}
    \end{align*}
    By the definition of event $\cE^{\conc}(\delta)$  we conclude the statement.
\end{proof}

Let us define $\delta^t_h = \uV^{t}_h(s^t_h) - V^{\pi^{t}}_h(s^t_h)$ and $\zeta^{t}_h = \uV^{t}_h(s^t_h) - \Vstar_h(s^t_h)$. 
\begin{lemma}\label{lem:regret_decomp}
    Assume conditions of Proposition~\ref{prop:optimism} hold. Then on event $\cG'(\delta) = \cG(\delta) \cap \cE^{\opt}$, where $\cG(\delta)$ is defined in Lemma~\ref{lem:proba_master_event}, the following upper bound on regret holds
    \[
        \regret^T \leq \rme H \sum_{t=1}^T \sum_{h=1}^H \ind\{k^t_h(s^t_h, a^t_h) = -1\} + \sum_{t=1}^T \sum_{h=1}^H (1+1/H)^{H-h} \xi^t_h + \rme \sum_{t=1}^T \sum_{h=1}^H \cB^t_h,
    \]
    where $\xi^t_h = p_h [\Vstar_{h+1} - V^{\pi^{t}}_{h+1}](s^t_h,a^t_h) - [\Vstar_{h+1} - V^{\pi^{t}}_{h+1}](s^t_{h+1})$ and $\cB^t_h = \cB^t_h(s^t_h, a^t_h) \cdot \ind\{k^t_h(s^t_h, a^t_h) \geq 0\}$ for $\cB^t_h$ defined in Corollary~\ref{cor:delta_q_bounds}.
\end{lemma}
\begin{proof}
    We notice that on the event $\cE^\opt$ the following upper bound holds
    \begin{equation}\label{eq:regret_bound_deltas}
        \regret^T \leq \sum_{t=1}^T \delta^t_1.
    \end{equation}
    Next we analyze $\delta^t_h$. By the choice of $a^t_h = \argmax_{a\in \cA} \uQ^{t}_h(s^t_h, a)$, Corollary~\ref{cor:delta_q_bounds}, and Bellman equations, we have
    \begin{align*}
        \delta^t_h &= \uV^{t}_h(s^t_h) - V^{\pi^{t}}_h(s^t_h) =  \uQ^{t}_h(s^t_h, a^t_h) - Q^{\pi^{t}}_h(s^t_h, a^t_h) \\
        &= \uQ^{t}_h(s^t_h, a^t_h) - \Qstar_h(s^t_h, a^t_h) + \Qstar_h(s^t_h, a^t_h) - Q^{\pi^{t}}_h(s^t_h, a^t_h) \\
        &\leq H \ind\{N^t_h = 0\} + \ind\{ N^t_h > 0\}\left( \frac{1}{N^t_h} \sum_{i=1}^{N^t_h} \zeta^{\ell^i_{t,h}}_{h+1} + \cB^t_h(s^t_h, a^t_h) + p_h [\Vstar_{h+1} - V^{\pi^{t}}_{h+1}](s^t_h,a^t_h) \right).
    \end{align*}
    where $k^t_h = k^t_h(s^t_h, a^t_h)$, $N^t_h = e_{k^t_h}$, $\ell^i_{t,h}$ is  episode of  the $i$-th visitation of the state-action pair $(s^t_h, a^t_h)$ during the stage $k^t_h$, and additionally by the convention $0/0 = 0$. 
    Let $\xi^t_h = p_h [\Vstar_{h+1} - V^{\pi^{t}}_{h+1}](s^t_h,a^t_h) - [\Vstar_{h+1} - V^{\pi^{t}}_{h+1}](s^t_{h+1})$ be a martingale-difference sequence, and $\cB^t_h = \cB^t_h(s^t_h, a^t_h) \ind\{N^t_h > 0\}$ 
    then 
    \[
        \delta^t_h \leq H \ind\{N^t_h = 0\} + \frac{\ind\{N^t_h > 0\}}{N^t_h} \sum_{i=1}^{N^t_h} \zeta^{\ell^i_{t,h}}_{h+1}  - \zeta^t_{h+1} + \delta^t_{h+1} + \xi^t_h + \cB^t_h.
    \]
    and, as a result
    \begin{align*}
        \sum_{t=1}^T \delta^t_h &\leq H \sum_{t=1}^T \ind\{N^t_h = 0\} + \sum_{t=1}^T \frac{\ind\{ N^t_h > 0\}}{N^t_h}\sum_{i=1}^{N^t_h} \zeta^{\ell^i_{t,h}}_{h+1} \\
        &-\sum_{t=1}^T\zeta^t_{h+1} + \sum_{t=1}^T \delta^t_{h+1} + \sum_{t=1}^T \xi^t_h + \sum_{t=1}^T \cB^t_h.
    \end{align*}
    Next we have to analyze the second term, following the approach by \cite{zhang2020advantage},
    \begin{align*}
        \sum_{t=1}^T \frac{\ind\{N^t_h > 0\}}{N^t_h}\sum_{i=1}^{N^t_h} \zeta^{\ell^i_{t,h}}_{h+1} &= \sum_{q=1}^T \sum_{t=1}^T \frac{\ind\{N^t_h > 0\}}{N^t_h}\sum_{i=1}^{N^t_h} \zeta^{\ell^i_{t,h}}_{h+1}  \ind\{\ell^i_{t,h} = q\} \\
        &= \sum_{q=1}^T \zeta^q_{h+1}  \cdot \sum_{t=1}^T \frac{\ind\{k^t_h \geq 0\}}{N^t_h} \sum_{i=1}^{N^t_h} \ind\{\ell^i_{t,h} = q\}.
    \end{align*}
    Notice that $\sum_{i=1}^{N^t_h} \ind\{\ell^i_{t,h} = q\} \leq 1$ since all visitations are increasing in $i$, and, moreover, it turns to equality if and only if $(s^q_h ,a^q_h) = (s^t_h ,a^t_h)$ and this visitation happens in stage $k^t_h$, where $k^t_h$ is equal to the stage of episode $q$ with respect to $(s^q_h, a^q_h, h)$.
    Since the sum is over all the next episodes with respect to stage of $q$, we have that the number of non-zero elements in the sum over $t$ is bounded by $(1+1/H) N^t_h$. Thus
    \[
        \sum_{q=1}^T \zeta^q_{h+1}  \cdot \sum_{t=1}^T \frac{\ind\{k^t_h \geq 0\}}{N^t_h} \sum_{i=1}^{N^t_h} \ind\{\ell^i_{t,h} = q\} \leq \left( 1 + \frac{1}{H} \right)\sum_{q=1}^T \zeta^q_{h+1}.
    \]

     After a simple algebraic manipulations and using the fact that $\zeta^t_h \leq \delta^t_h$,
    \begin{align*}
        \sum_{t=1}^T \delta^t_h &\leq H \sum_{t=1}^T \ind\{N^t_h = 0\} + \sum_{t=1}^T (1 + 1/H) \zeta^t_{h+1} - \sum_{t=1}^T\zeta^t_{h+1} + \sum_{t=1}^T \delta^t_{h+1} + \sum_{t=1}^T \xi^t_h + \sum_{t=1}^T \cB^t_h \\
        &\leq  H \sum_{t=1}^T \ind\{N^t_h = 0\} + \left(1 + \frac{1}{H} \right) \sum_{t=1}^T \delta^t_{h+1} + \sum_{t=1}^T \xi^t_h + \sum_{t=1}^T \cB^t_h.
    \end{align*}
    By rolling out the upper bound on regret \eqref{eq:regret_bound_deltas} and using inequality $(1+1/H)^{H-h} \leq \rme$ we have
    \[
        \regret^T \leq \rme H \sum_{t=1}^T \sum_{h=1}^H \ind\{N^t_h = 0\} + \sum_{t=1}^T \sum_{h=1}^H (1+1/H)^{H-h} \xi^t_h + \rme \sum_{t=1}^T \sum_{h=1}^H \cB^t_h.
    \]
\end{proof}

\begin{proof}[Proof of Theorem~\ref{th:regret_randql_tabular}]
    First, we notice that the event $\cG'(\delta)$ defined in Lemma~\ref{lem:regret_decomp}, holds with probability at least $1-\delta$ by Lemma~\ref{lem:proba_master_event} and Proposition~\ref{prop:optimism}. Thus, we may assume that $\cG'(\delta)$ holds.

    We start from the decomposition given by Lemma~\ref{lem:regret_decomp}
    \[
        \regret^T \leq \rme H \sum_{t=1}^T \sum_{h=1}^H \ind\{k^t_h(s^t_h, a^t_h) = -1\} + \sum_{t=1}^T \sum_{h=1}^H (1+1/H)^{H-h} \xi^t_h + \rme \sum_{t=1}^T \sum_{h=1}^H \cB^t_h.
    \]
    The first term is upper bounded by $\rme SAH^3$, since there is no more than $H$ visits of each state-action-step triple before the update for the first stage. The second term is bounded by $\tcO(\sqrt{H^3 T})$ by a definition of the event $\cE(\delta)$ in Lemma~\ref{lem:proba_master_event}. To upper bound the last term we have to analyze the following sum
    \[
        \sum_{t=1}^T \sum_{h=1}^H \frac{\ind\{ e_{k^t_h(s^t_h,a^t_h)} > 0 \}}{\sqrt{e_{k^t_h(s^t_h, a^t_h)}}} \leq \sum_{(s,a,h) \in \cS \times \cA \times [H]} \sum_{k=0}^{k^{T+1}_h(s,a)} \frac{e_{k+1}}{\sqrt{e_k}},
    \]
    where
    \[
        e_{k} = \left\lfloor \left(1+\frac{1}{H} \right)^k H \right\rfloor \Rightarrow \frac{e_{k+1}}{\sqrt{e_k}} \leq 2 \sqrt{e_k},
    \]
    therefore by Cauchy inequality
    \begin{small}
    \[
        \sum_{k=0}^{k^{T+1}_h(s,a)} \frac{e_{k+1}}{\sqrt{e_k}} \leq 2 \sum_{k=0}^{k^{T+1}_h(s,a) } \sqrt{e_k} \leq 2 \sqrt{k^{T+1}_h(s,a)} \sqrt{\sum_{k=0}^{k^{T+1}_h(s,a)} e_k} \leq 2 \sqrt{\frac{\log(T)}{\log(1+1/H)}} \sqrt{n^{T+1}_h(s,a)},
    \]
    \end{small}
    \!\!where we used the definition of the previous stage $k^{T+1}_h(s,a)$
    \[
        N^{T+1}_h(s,a) \geq \sum_{k=0}^{k^{T+1}_h(s,a)} e^{k}, 
    \]
    thus by Cauchy inequality and inequality $\log(1+1/H) \geq 1/(4H)$ for $H \geq 1$
    \begin{align*}
        \sum_{t=1}^T \sum_{h=1}^H \frac{\ind\{ e_{k^t_h(s^t_h,a^t_h) > 0} \}}{\sqrt{e_{k^t_h(s^t_h, a^t_h)}}} &\leq 2\sqrt{H \log(T)} \sum_{(s,a,h) \in \cS \times \cA \times [H]} \sqrt{N^{T+1}_h(s,a) + 1} \\
        &\leq 4\sqrt{SAH^2 \log(T)} \sqrt{\sum_{(s,a,h)} (N^{T+1}_h(s,a) + 1)} \\
        &\leq 4 \sqrt{SAH^3T \log(T)} + 4 SAH^2 \log(T).
    \end{align*}
    Using this upper bound, we have
    \[
        \sum_{t=1}^T \sum_{h=1}^H \cB^t_h = \tcO\left( H \sum_{t=1}^T \sum_{h=1}^H \frac{\ind\{ e_{k^t_h(s^t_h,a^t_h)} > 0 \}}{\sqrt{e_{k^t_h(s^t_h, a^t_h)}}}  \right) = \tcO\left( \sqrt{H^5 SA T} + SAH^3\right).
    \]
    Combining this upper bound with the previous ones, we conclude the statement.
\end{proof}
\newpage
\section{Proofs for Metric algorithm}\label{app:randql_metric_proof}

\subsection{Assumptions}

In this section we proof Lemma~\ref{lem:v_lipschitz_from_reparam} and Lemma~\ref{lem:transition_w_from_reparam}.

\begin{proof}[Proof of Lemma~\ref{lem:transition_w_from_reparam}]
    By the dual formula for 1-Wasserstein distance (see e.g. Section~6 of \citet{peyre2019ot}) we have
    \[
        \cW_1(p_h(s,a), p_h(s',a')) = \sup_{f \text{ is } 1-\text{Lipchitz}} \left\{ p_h f(s,a) - p_h f(s',a') \right\}.
    \]
    By Assumption~\ref{ass:reparametrization} we have
    \[
        p_h f(s,a) - p_h f(s',a') = \E_{\xi_h} \left[ f(F_h(s,a,\xi_h)) - f(F_h(s',a',\xi_h)) \right] \leq L_F \rho((s,a),(s',a')).
    \]
\end{proof}

\begin{proof}[Proof of Lemma~\ref{lem:v_lipschitz_from_reparam}]
    Let us proceed by a backward induction over $h$. For $h=H+1$ we have $\Qstar_{H+1}(s,a) = \Vstar_{H+1}(s) = 0$, therefore they are $0$-Lipchitz. 

    Next we assume that have for any $h' > h$ the statement of Lemma~\ref{lem:v_lipschitz_from_reparam} holds. Then by Bellman equations
    \[
        \vert \Qstar_h(s,a) - \Qstar_h(s',a')\vert \leq \vert r_h(s,a) + r_h(s',a') \vert + \vert p_h \Vstar_{h+1}(s,a) - p_h \Vstar_{h+1}(s',a') \vert .
    \]
    By Assumption~\ref{ass:reparametrization} we can represent the action of the transition kernel as follows
    \[
        p_h \Vstar_{h+1}(s,a) - p_h \Vstar_{h+1}(s',a') = \E_{\xi_h}\left[ \Vstar_{h+1}(F_h(s,a, \xi_h)) - \Vstar_{h+1}(F_h(s', a', \xi_h) \right].
    \]
    Since by induction hypothesis $\Vstar_{h+1}$ is $\sum_{h'=h+1}^H L_F^{h'-h} L_r $-Lipschitz and $F_h(\cdot, \xi_h)$ is $L_F$-Lipschitz, therefore
    \begin{align*}
        \vert \Qstar_h(s,a) - \Qstar_h(s',a')\vert &\leq \left( L_r  +  L_F \cdot \sum_{h'=h+1}^H L_F^{h'-h} L_r\right) \rho((s,a), (s',a')) \\
        &\leq \left( \sum_{h'=h}^H L_F^{h'-h} L_r \right) \rho((s,a), (s',a'))
    \end{align*}
    To show that $\Vstar_h$ is also Lipchitz, we have that there is some action $a^\star$ equal to $\pi^\star(s)$ or $\pi^\star(s')$, such that
    \[
        \vert \Vstar_h(s) - \Vstar_h(s') \vert \leq \vert \Qstar_h(s,a^\star) - \Qstar_h(s', a^\star) \vert \leq L_{V,h} \cdot \rho((s,a^\star),(s',a^\star)) \leq L_{V,h} \cdot \rho_{\cS}(s,s'),
    \]
    where in the end we used the sub-additivity assumption on metric over joint space (see Assumption~\ref{ass:metric}).
\end{proof}

\subsection{Algorithm}\label{app:net_staged_randql_description}

\begin{algorithm}
\centering
\caption{Metric \NetStagedRandQL}
\label{alg:NetStagedRandQL}
\begin{algorithmic}[1]
  \STATE {\bfseries Input:}  inflation coefficient $\kappa$, $J$ ensemble size, number of prior transitions $n_0(k)$, prior reward $r_0$, dicretization level $\varepsilon$.
  \STATE {\bfseries Initialize: }  $\varepsilon$-net $\cN_{\varepsilon}$, $\uQ_h(B) = \tQ^j_h(B) = \ur H,$ initialize counters $\tn_h(B) = 0$ for $j,h,B\in[J]\times[H]\times\cN_{\varepsilon}$, stage $q_h(B) = 0$, quantization map $\psi_{\varepsilon} \colon \cS \times \cA \to \cN_{\varepsilon}$.
      \FOR{$t \in[T]$}
      \FOR{$h \in [H]$}
        \STATE Play $a_h \in \argmax_a \uQ_h(\psi_{\varepsilon}(s_h,a))$ and define $B_h = \psi_{\varepsilon}(s_h, a_h)$.
        \STATE Observe reward and next state $s_{h+1}\sim p_h(s_h,a_h)$.
        \STATE Sample learning rates $w_j \sim \Beta(1/\kappa, (\tn+n_0(q_h(B_h))/\kappa)$ for $\tn = \tn_h(B_h)$.
        \STATE Compute value function $\uV_{h+1}(s_{h+1}) = \max_{a \in \cA} \uQ_{h+1}(\psi_{\varepsilon}(s_{h+1},a))$.
        \STATE Update temporary $Q$-values for all $j \in [J]$
        \[
            \tQ^{j}_h(B) := (1 - w_j) \tQ^{j}_h(B) + w_j \left( r_h(s_h,a_h) + \uV_{h+1}(s_{h+1})\right)\,.
        \]
        \vspace{-0.4cm}
        \STATE Update counter $\tn_h(B_h) := \tn_h(B_h) + 1$
        \IF{$\tn_h(B_h) = \lfloor (1 + 1/H)^{q} H \rfloor$ for $q=q_h(B_h)$ is the current stage}
            \STATE Update policy $Q$-values $\uQ_h(B_h) := \max_{j \in [J]} \tQ^{j}_h(B_h)$.
            \STATE Reset temporary  $Q$-values $\tQ^j_h(B_h) := \ur H$.
            \STATE Reset counter $\tn_h(B_h) := 0$ and change stage $k_h(B_h) := k_h(B_h) + 1$.
        \ENDIF
      \ENDFOR
  \ENDFOR
\end{algorithmic}
\end{algorithm}

Next we describe a simple non-adaptive version of our algorithm that works with metric spaces. We assume that for any $\varepsilon > 0$ we can compute a minimal $\varepsilon$-cover of state-action space $\cN_\varepsilon$.\footnote{Remark that the greedy algorithm can easily generate $\varepsilon$-cover of size $N_{\varepsilon/2}$, that will not affect the asymptotic behavior of regret bounds, see \cite{song2019efficient}.}

Next we will use the same notation but with state-action pairs replaces with balls from a fixed cover $\cN_\varepsilon$. To unify the notation, we define $\psi_\varepsilon \colon \cS \times \cA \to \cN_\varepsilon$ that maps any point $(s,a)$ to any ball from $\varepsilon$-cover that contains it.

For any $t,h$ we define $B^t_h = \psi_\varepsilon(s^t_h, a^t_h)$. Next, let $n^t_h(B)$ be a number of visits of ball $B$ before the episode $t$: $n^t_h(B) = \sum_{k=1}^{t-1} \ind\{ B^k_h = B \}$.

Let $e_k = \lfloor (1 + 1/H)^k \cdot H \rfloor$ be length of each stage for any $k \geq 0$ and, by convention, $e_{-1} = 0$. %Define $\cL = \{ \sum_{i=0}^k e_i \}_{k\in\N}$ the sequence of stages starting points.
We will call that in the beginning of episode $t$ a pair $(B,h)$ is in $k$-th stage if $n^t_h(B) \in [\sum_{i=0}^{k-1} e_i, \sum_{i=0}^k e_i )$. %For an episode $t+1$ we will call $k^t_h(s,a)$ the index of the previous stage.

Let $\tn^t_h(B)$ be a number of visits of state-action pair during the current stage in the beginning of episode $t$. Formally, $\tn^t_h(B) = n^t_h(B) - \sum_{i=0}^{k-1} e_i$, where $k$ is an index of current stage. 

%Also we define $\upn^t_h(s,a) = \tn^t_h(s,a) + n_0$ as a pseudo-counts.

Define $\kappa > 0$ be a posterior inflation coefficient, $n_0$ is a number of pseudo-transitions, and $J$ as a number of temporary $Q$-functions. Let $\tQ^{t,j}_h$ be a $j$-th \textit{temporary} Q-value and $\uQ^t_h$ be a \textit{policy} Q-value at the beginning of episode $t$, defined over the $\varepsilon$-cover. We initialize them as follows
\[
    \uQ^1_h(B) = \ur H, \quad \tQ^{1,j}_h(s,a) = \ur H.
\]
Additionally, we define to the value function as follows
\[
    \uV^t_h(s) = \max_{a \in \cA} \uQ^t_h(\psi_{\varepsilon}(s,a)).
\]
Notice that we cannot precomute it as in the tabular setting, however, it is possible to use its values in lazy fashion.

For each transition we preform the following update of temporary Q-values over balls $B \in \cN_{\varepsilon}$
\[
    \tQ^{t+1/2,j}_h(B) = \begin{cases}
        (1- w_{j, \tn}) \cdot \tQ^{t,j}_h(B) + w_{j, \tn} [ r_h(s^t_h, a^t_h) + \uV^{t}_{h+1}(s^t_{h+1})], & B = B^t_h \\
        \tQ^{t,j}_h(B) & \text{otherwise},
    \end{cases}
\]
where  $\tn=\tn^t_h(B)$ is the number of visits of $(B,h)$ in the beginning of episode $t$, and $w_{j, \tn}$ is a sequence of independent beta-distribution random variables $w_{j,\tn} \sim \Beta(1/\kappa, (\tn + n_0) / \kappa)$.

Next we define the stage update as follows
\begin{align*}
    \uQ^{t+1}_h(B) &= \begin{cases}
        \max_{j\in[J]}  \tQ^{t+1/2,j}_h(B) & \tn^t_h(B) = \lfloor (1 + 1/H)^k H \rfloor \\
        \uQ^{t}_h(B) & \text{otherwise}
    \end{cases} \\
    \tQ^{t+1,j}_h(B) &= \begin{cases}
         \ur H  & n^t_h(B) \in \tn^t_h(B) = \lfloor (1 + 1/H)^k H \rfloor \\ 
         \tQ^{t+1/2,j}_h(B) & \text{otherwise}
    \end{cases} \\
    \uV^{t+1}_h(s) &= \min\{ \ur(H-h), \max_{a \in \cA} \uQ^{t+1}_h(\psi_\varepsilon(s,a)) \};\\
    \pi^{t+1}_h(s) &\in \argmax_{a \in \cA} \uQ^{t+1}_h(\psi_\varepsilon(s,a)),
\end{align*}
where $k$ is the current stage.
A detailed description of the algorithm is presented in Algorithm~\ref{alg:NetStagedRandQL}.

For episode $t$ we will call $k^t_h(B)$ the index of stage where $\uQ^t_h(B)$ were updated (and $k^t_h(B) = -1$ if there was no update). For all $t$ we define $\tau^t_h(B) \leq t$ as the episode when the stage update happens. In other words, for any $t$ the following holds
\[
    \uQ^{t+1}_h(B) = \max_{j\in[J]}  \tQ^{\tau^t_h(B)+1/2,j}_h(B),
\]
where $\tau^t_h(B) = 0$ and $e_k = 0$ if there was no updates. To simplify the notation we will omit dependence on $(s,a,h)$ where it is deducible from the context.

We notice that in this case we use $e_k$ samples to compute $\tQ^{\tau^t_h(B)+1/2,j}$ for $k = k^t_h(s,a)$. For this $k$ we define $\ell^i_{k,h}(s,a)$ as the time of $i$-th visit of state-action pair $(s,a)$ during $k$-th stage. Then we have the following decomposition
\begin{equation}\label{eq:temp_q_rollout}
    \tQ^{\tau^t+1/2,j}_h(B) =  \sum_{i=0}^{e_k} W^{i}_{j, e_k} \left( r_h(s^{\ell^i}_h,a^{\ell^i}_h) +  \uV^{\ell^i}_{h+1}(s^{\ell^i}_{h+1}) \right),
\end{equation}
where we drop dependence on $k$ and $(B,h)$ in $\ell^i$ to simplify notations, using the convention $r_h(s^{\ell^0}_h,a^{\ell^0}_h) = \ur$ , $\uV^{\ell^0}_{h+1}(s^{\ell^0}_{h+1}) = \ur (H-1)$ and the following aggregated weights
\[
    W^{0}_{j, n} = \prod_{q=0}^{n-1} ( 1 - w_{j,q}), \quad W^i_{j,n} = w_{j,i-1} \cdot \prod_{q=i}^{n-1} (1 - w_{j,q}),\ i \geq 1.
\]

%\subsection{Sketch of Proof}\label{app:sketch_of_proof_metric}
%\dan{Add about Dirichlet process here}

\subsection{Concentration}
\label{app:concentration_metric}

Let $\betastar\colon (0,1) \times \N \times (0, d_{\max}) \to \R_{+}$ and $\beta^{B}, \beta^{\conc}, \beta\colon (0,1) \times (0, d_{\max}) \to \R_{+}$ be some function defined later on in Lemma \ref{lem:proba_master_event_metric}. We define the following favorable events

\begin{align*}
  \cE^\star(\delta, \varepsilon) &\triangleq \Bigg\{\forall t \in \N, \forall h \in [H], \forall B\in \cN_{\varepsilon}, k = k^t_h(B), (s,a) = \cent(B): \\
  &\qquad
    \Kinf\left( \frac{1}{e_k}\sum_{i=1}^{e_k} \delta_{\Vstar_{h+1}(F_h(s,a,\xi^{\ell^i}_{h+1}))}, p_h \Vstar_{h+1}(s,a) \right) \leq  \frac{\betastar(\delta,e_k,\varepsilon)}{e_k}\Bigg\}\,,\\
    % Next event
    \cE^{B}(\delta, \varepsilon) &\triangleq \Bigg\{ \forall t \in [T], \forall h \in [H], \forall B \in \cN_{\varepsilon}, \forall j \in [J], k = k^t_h(B):  \\
    &\qquad \left| \sum_{i=0}^{e_k} \left( W^i_{j, e_k, k} - \E[W^i_{j, e_k,k}] \right) \left( r_{h}(s^{\ell^i}_h, a^{\ell^i}_h) +  \uV^{\ell^i}_{h+1}(s^{\ell^i}_{h+1}) \right) \right| \\
    &\qquad\qquad\leq 60 \rme^2 \sqrt{\frac{\ur^2 H^2 \kappa \beta^B(\delta, \varepsilon)}{e_k + n_0(k)}} + 1200 \rme \frac{\ur H \kappa \log(e_k + n_0(k)) (\beta^B(\delta, \varepsilon))^2}{e_k + n_0(k)} \bigg\}\,,\\
    % Next event
    \cE^{\conc}(\delta, \varepsilon) &\triangleq \Bigg\{\forall t \in [T], \forall h \in [H], \forall B \in \cN_{\varepsilon}, k = k^t_h(B): \\
    &\qquad\left|\frac{1}{e_k} \sum_{i=1}^{e_k} \Vstar_{h+1}(s^{\ell^i_{k,h}(B)}_{h+1})  - p_h\Vstar_{h+1}(s^{\ell^i_{k,h}(B)}_h, a^{\ell^i_{k,h}(B)}_h) \right| \leq \sqrt{\frac{2\ur^2 H^2 \beta^{\conc}(\delta, \varepsilon)}{e_k}}\Bigg\}\\
    % Next event
    \cE(\delta) &\triangleq \Bigg\{  \sum_{t=1}^T \sum_{h=1}^H (1+1/H)^{H-h}\left| p_h[\Vstar_{h+1} - V^{\pi_t}_{h+1}](s^t_h, a^t_h) - [\Vstar_{h+1} - V^{\pi_t}_{h+1}](s^t_{h+1})\right| \\
    &\qquad\qquad\qquad\qquad\qquad\qquad\qquad\qquad\qquad\qquad\qquad\quad\leq 2\rme \ur H\sqrt{2HT \beta(\delta)}.
\Bigg\}.
\end{align*}
We also introduce the intersection of these events, $\cG(\delta) \triangleq \cE^\star(\delta) \cap \cE^{B}(\delta) \cap \cE^{\conc}(\delta) \cap \cE(\delta)$. We  prove that for the right choice of the functions $\betastar,  \beta^{\KL}, \beta^{\conc}, \beta, \beta^{\Var}$ the above events hold with high probability.
\begin{lemma}
\label{lem:proba_master_event_metric}
For any $\delta \in (0,1)$ and $\varepsilon \in (0, d_{\max})$ and for the following choices of functions $\beta,$
\begin{align*}
    \betastar(\delta,n, \varepsilon) &\triangleq \log(8H/\delta) + \log(N_{\varepsilon}) + 3\log\left(\rme\pi(2n+1)\right)\,,\\
    \beta^B(\delta, \varepsilon) &\triangleq  \log(8H/\delta) + \log(N_{\varepsilon}) +  \log(TJ)\,,\\
    \beta^{\conc}(\delta, \varepsilon) &\triangleq \log(8H/\delta) + \log(N_{\varepsilon}) +  \log(2T) ,\\
    \beta(\delta) &\triangleq \log\left(16/\delta\right),
\end{align*}
it holds that
\begin{align*}
\P[\cE^\star(\delta, \varepsilon)]&\geq 1-\delta/8, \qquad \P[\cE^{B}(\delta, \varepsilon)]\geq 1-\delta/8,  \\\
\P[\cE^\conc(\delta, \varepsilon)] &\geq 1-\delta/8, \qquad \P[\cE(\delta)]\geq 1-\delta/8.
\end{align*}
In particular, $\P[\cG(\delta)] \geq 1-\delta/2$.
\end{lemma}
\begin{proof}
    Let us describe the changes from the similar statement in Lemma~\ref{lem:proba_master_event}.

    Regarding event $\cE^\star(\delta, \varepsilon)$, for any fixed ball $B$ we have exactly the same structure of the problem thanks to Assumption~\ref{ass:reparametrization} and a sequence of i.i.d. random variables $\xi^{\ell^i}_h$. Thus, Theorem~\ref{th:max_ineq_kinf} combined with a union bound over $B \in \cN_{\varepsilon}$ and $H \in [H]$ concludes $\PP{\cE^{\star}(\delta, \varepsilon)} \geq 1 - \delta/8$.

    The proof for the event $\cE^B(\delta,\varepsilon)$ remains the almost the same, with two differences: the predictable weights slightly changed but the upper bound for them remain the same,  and we have  take a union bound not over all state-action pairs $(s,a) \in \cS \times \cA$ but all over balls $B \in \cN_{\varepsilon}$.

    To show that $\PP{\cE^{\conc}(\delta, \varepsilon)} \geq 1 - \delta/8$, let us fix $B \in \cN_{\varepsilon}, h \in [H]$ and $e_k \in [T]$. Then we can define a filtration $\cF_{t,h} = \sigma\left\{ \{ (s^{\ell}_{h'}, a^{\ell}_{h'}, \pi^{\ell}), \ell < t, h' \in [H] \} \cup \{ (s^{t}_{h'}, a^t_{h'}, \pi^{t}), h' \leq h \} \right\}$ and, since $\ell^i_{k,h}(B)$ are stopping times for all $i = 1, \ldots, e_k$, we can define the stopped filtration $\widetilde{\cF_{i}} = \cF_{\ell^i, h}$. Then we notice that $X_i = \Vstar_{h+1}(s^{\ell^i_{k,h}(B)}_{h+1})  - p_h\Vstar_{h+1}(s^{\ell^i_{k,h}(B)}_h, a^{\ell^i_{k,h}(B)}_h)$ forms a martingale-difference sequence with respect to $\widetilde{\cF_{i,h}}$. Thus, by Azuma-Hoeffding inequality and a union bound we have $\PP{\cE^{\conc}(\delta, \varepsilon)} \geq 1 - \delta/8$.

    The proof of $\PP{\cE(\delta)} \geq 1-\delta/8$ remains exactly the same as in Lemma~\ref{lem:proba_master_event}.
\end{proof}

\subsection{Optimism}\label{app:optimism_metric}

 In this section we prove that our estimate of $Q$-function $\uQ^{\,t}_h(s,a)$ is optimistic that is the event
\begin{equation}\label{eq:opt_event_metric}
    \cE_{\opt}(\varepsilon) \triangleq \left\{ \forall t \in [T], h \in [H], (s,a) \in \cS \times \cA:  \uQ^t_h(\psi_{\varepsilon}(s,a)) \geq \Qstar_h(s,a) \right\}.
\end{equation}
holds with high probability on the event $\cE^\star(\delta, \varepsilon)$.

Define constants
\begin{equation}\label{eq:constant_c0_metric}
    c_0 \triangleq \frac{8}{\pi} \left( \frac{4}{\sqrt{\log(17/16)}} + 8 + \frac{49\cdot 4\sqrt{6}}{9} \right)^2 + 1.
\end{equation}
and slightly another constant
\begin{equation}\label{eq:constant_cJ_metric}
    \tilde c_J \triangleq \frac{1}{\log\left( \frac{4}{3 + \Phi(1)} \right)},
\end{equation}
where $\Phi(\cdot)$ is a CDF of a normal distribution.

\begin{proposition}\label{prop:anticonc_metric}
    Define a constant $L = L_r + L_V(1+L_F)$.    
    Assume that $J = \lceil  \tilde{c}_J \cdot (\log(2HT/\delta)  + \log(N_\varepsilon) \rceil$, $\kappa = 2\beta^\star(\delta, T, \varepsilon)$, $\ur = 2$, and a prior count $n_0(k) = \ceil{\tn_0 + \kappa + \frac{\varepsilon L}{H-1} \cdot (e_k + \tn_0 + \kappa)}$ dependent on the stage $k$, where $\tn_0 = (c_0 + 1 + \log_{17/16}(T)) \cdot \kappa$ .
    
    Then on event $\cE^\star(\delta, \varepsilon)$ the following event
    \begin{align*}
        \cE_{\anticonc} &\triangleq \Bigg\{ \forall t \in [T] \ \forall h \in [H] \ \forall B \in \cN_{\varepsilon}: \text{for } k = k^t_h(B), (s,a) = \cent(B):  \\
        &\max_{j \in [J]}\biggl\{ 
            W^0_{j,e_k,k} \ur (H-1) + \sum_{i=1}^{e_k} W^i_{j,e_k,k} \Vstar_{h+1}(F_h(s, a, \xi^{\ell^i}_h))
        \biggr\} \geq p_h \Vstar_{h+1}(s,a) + L \varepsilon \Bigg\}
    \end{align*}
    holds with probability at least $1-\delta/2$.
\end{proposition}
\begin{remark}\label{rm:dirichlet_process}
    We notice that the obtained result is connected to the theory of Dirichlet processes.

    First, let us define the Dirichlet process, following \cite{ferguson1973bayesian}. The stochastic process $G$, indexed by elements $B$ of $\Xset$, is a Dirichlet Process with parameter $\nu$ ($G\sim\mathrm{DP}(\nu)$) if
    \begin{equation*}
    	G\left(B_1\right),\ldots,G\left(B_d\right)\sim\mathrm{Dir}\left(\nu\left(B_1\right), \ldots, \nu\left(B_d\right)\right),
    \end{equation*}
    for any measurable partition $\left(B_1, \ldots, B_d\right)$ of $\Xset$. 

    Let $\hP_n = \frac{1}{n} \sum_{i=1}^n \delta_{Z_i}$ be an empirical measure of an i.i.d. sample $Z_1,\ldots,Z_n \sim P$. Let $\nu$ be a finite (not necessarily probability) measure on $\Xset$ and $\widetilde{P}_n \sim \mathrm{DP}(\nu + n \hP_n)$. Then we have the following representation for the expectations of a function $f \colon \Xset \to \R$ over a sampled measure $\widetilde{P}_n$ (see Theorem 14.37 of \cite{ghosal2017fundamentals} with $\sigma=0$ for a proof)
    \[
        \widetilde{P}_n f = V_n  \cdot Qf + (1 - V_n) \sum_{i=1}^n W_i f(Z_i),
    \]
    where $V_n \sim \Beta(|\nu|, n)$,  $Q \sim \mathrm{DP}(\nu)$, and a vector $(W_1,\ldots,W_n)$ follows uniform Dirichlet distribution $\Dir(1,\ldots,1)$. If we take $\nu = n_0 \cdot \delta_{Z_0}$ for some $Z_0 \in \Xset$ such that $f(Z_0) = \ur(H-1)$\footnote{We can augment the space $\Xset$ with this additional point if needed}, then by a stick-breaking process representation of the Dirichlet distribution we have
    \[
        \widetilde{P}_n f = \tW_0 \ur(H-1) + \sum_{i=1}^n \tW_1 f(Z_i), \quad (\tW_0, \ldots, \tW_1) \sim \Dir(n_0, 1, \ldots, 1).
    \]
    By taking an appropriate $\Xset$ and $f$ we have that Proposition~\ref{prop:anticonc_metric} could be interpret as a deriving a lower bound on the probability of
    $
        \P[\widetilde{P}_n f \geq Pf + \varepsilon L \mid \{ Z_i\}_{i=1}^n ].
    $ 
\end{remark}

\begin{proof}
    First for all, let us fix $t \in [T], h \in [H]$ and $B \in \cN_\varepsilon$ and, consequently, $k = k^t_h(B)$. Also, let fix $j \in [J]$. To simplify the notation in the sequel, define $X_0 =  \ur(H-1)$  and $X_i = \Vstar_{h+1}(F_h(s, a, \xi^{\ell^i}_h))$ for $i > 0$. Notice that $X_i$ for $i>0$ is a sequence of i.i.d. random variables supported on $[0, H-h-1]$.

    By Lemma~\ref{lem:weights_distribution} we have $(W^0_{j,e_k,k},\ldots, W^{e_k}_{j,e_k,k}) \sim \Dir(n_0(k)/\kappa, 1/\kappa, \ldots, 1/\kappa)$. Then we use the aggregation property of Dirichlet distribution: there is a vector $(\widetilde{W}^{-1}_j, \ldots, \widetilde{W}^{e_k}_j) \sim \Dir( (n_0(k) - \tn_0)/\kappa, \tn_0 / \kappa, 1/\kappa, \ldots, 1/\kappa) $ such that
    \[ 
        \sum_{i=0}^{e_k} W^i_{j,e_k,k} X_i = 
        \widetilde{W}^{-1}_j X_0 + \sum_{i=0}^{e_k} \widetilde{W}^i_j  X_i.
    \]
    Next we are going to represent the Dirichlet random vector $\widetilde{W}$ by a stick breaking process (or, equivalently, represent via the generalized Dirichlet distribution)
    \begin{align*}
        \widetilde{W}^{-1}_j &= \xi_j & \xi_j \sim \Beta( (n_0(k) - \tn_0)/\kappa, (e_k + \tn_0)/\kappa), \\
        (\widetilde{W}^0_j, \ldots, \widetilde{W}^{e_k}_j) &= (1 - \xi_j) \cdot (\widehat{W}^0_j, \ldots, \widehat{W}^{e_k}_j), & \widehat{W}_j \sim \Dir(\tn_0/\kappa, 1/\kappa, \ldots, 1/\kappa),
    \end{align*}
    where $\xi_j$ and $\widehat{W}_j$ are independent. Therefore, we have the final decomposition
    \begin{align*}
        \sum_{i=0}^{e_k} W^i_{j,e_k,k} X_i - p_h \Vstar_{h+1}(s,a) - \varepsilon L &= \underbrace{\xi_j \left( \ur (H-1) - p_h \Vstar_{h+1}(s,a) \right) - \varepsilon L}_{T_{\mathrm{approx}}} \\
        &+ (1-\xi_j) \underbrace{\cdot \left( \sum_{i=0}^{e_k} \widehat{W}^i_j X_i - p_h \Vstar_{h+1}(s,a) \right)}_{T_{\mathrm{stoch}}}.
    \end{align*}
    By independence of $\xi_j$ and $\widehat{W}_j$ we have
    \[
        \P\left[\sum_{i=0}^{e_k} W^i_{j,e_k,k} X_i \geq p_h \Vstar_{h+1}(s,a) + \varepsilon L | \{X_i\}_{i=1}^{e_k}\right] \geq \P[T_{\mathrm{approx}} \geq 0] \cdot \P[T_{\mathrm{stoch}} \geq 0].
    \]

    We split our problem to lower bound the two separate probabilities.

    \paragraph{Approximation error}
    To deal with approximation error, we first of all notice that $p_h \Vstar_{h+1}(s,a) \leq H-1$, therefore we have
    \[
        \P[T_{\mathrm{approx}} \geq 0] = \P\left[\xi_j \geq \frac{\varepsilon L}{H-1}\right].
    \]
    Next we assume that $\varepsilon < (H-1)/L$, since $\xi_j \sim \Beta( (n_0(k) - \tn_0)/\kappa, (e_k + \tn_0)/\kappa)$, we may apply \citet[Theorem 1.2'']{alfers1984normal} 
    \[
        \P[T_{\mathrm{approx}} \geq 0] \geq \Phi\left(-\mathrm{sign}(p - \mu) \cdot \sqrt{2 \ualpha \kl(p, \mu)}\right),
    \]
    where $p = (n_0(k) - \tn_0 - \kappa) / (e_k + \tn_0 - \kappa)$ and $\mu = \varepsilon L / (H-1)$. Since $n_0(k) = \ceil{\tn_0 + \kappa + \frac{\varepsilon L}{H-1} \cdot (e_k + \tn_0 + \kappa)}$, we have $\P[T_{\mathrm{approx}} \geq 0] \geq 1/2$.

    \paragraph{Stochastic error}
    Since $X_0 = \ur(H-1)$ is an upper bound on $V$-function, and we have that the weight of the first atom $\alpha_0 \triangleq \tn_0 /\kappa - 1 = c_0 + \log_{17/16}(T) - 1$ for $c_0$ defined in \eqref{eq:constant_c0_metric}.

    Define a measure $\bnu_{e_k} = \frac{\tn_0 - \kappa}{e_k + \tn_0 - \kappa} \delta_{X_0} + \sum_{i=1}^{e_k} \frac{1}{e_k + n_0 - 1} \delta_{X_i}$. Since $p_h \Vstar_{h+1}(s,a) \leq H-h-1$, we can apply Lemma~\ref{lem:lower_bound_dbc_relaxed} with a fixed $\varepsilon = 1/2$ conditioned on independent random variables $X_i$
    \begin{align*}
        \begin{split}
        \P\biggl[ \sum_{i=0}^{e_k} \widehat{W}^i_j X_i &\geq p_h \Vstar_{h+1}(s,a) \mid \{ X_i \}_{i=1}^{e_k} \biggl] \\
        &\geq \frac{1}{2}\left( 1 - \Phi\left(\sqrt{\frac{2 (e_k + n_0 -\kappa) \Kinf\left(\bnu_{e_k}, p_h \Vstar_{h+1}(s,a)\right) }{\kappa}}\right)\right),
        \end{split}
    \end{align*}
    where $\Phi$ is a CDF of a normal distribution. By Lemma~\ref{lem:kinf_prior_remove} and the event $\cE^\star(\delta, \varepsilon)$
    \begin{align*}
        (e_k + n_0 - \kappa) \Kinf\left( \bnu_{e_k}, p_h \Vstar_{h+1}(s,a)\right) 
        \leq e_k \Kinf\left( \hnu_{e_k}, p_h \Vstar_{h+1}(s,a)\right) \leq \beta^\star(\delta, T, \varepsilon),
    \end{align*}
    where $\hnu_{e_k} = \frac{1}{e_k} \sum_{i=1}^{e_k} \delta_{\Vstar_{h+1}(F(s,a, \xi^{\ell^i}_{h+1}))}$, and, as a corollary
    \[
        \P\left[ \sum_{i=0}^{e_k} \widehat{W}^i_j X_i \geq p_h \Vstar_{h+1}(s,a) \mid \cE^\star(\delta, \varepsilon), \{ X_i \}_{i=1}^{e_k}  \right] \geq \frac{1}{2} \left( 1 - \Phi\left( \sqrt{\frac{2\beta^\star(\delta, T, \varepsilon) }{\kappa}} \right) \right).
    \]
    By taking $\kappa = 2\beta^\star(\delta, T, \varepsilon)$ we have a constant probability of being optimistic for stochastic error
    \[
       \P[T_{\mathrm{stoch}} \geq 0 \mid \cE^\star(\delta, \varepsilon)] \geq \frac{1 - \Phi(1)}{2}.
    \]

    Overall, combining two lower bound for approximation and stochastic terms, we have
    \[
        \P\left[\sum_{i=0}^{e_k} W^i_{j,e_k,k} X_i \geq p_h \Vstar_{h+1}(s,a) + \varepsilon L | \cE^\star(\delta, \varepsilon) \right] \geq \frac{1 - \Phi(1)}{4}  = \gamma.
    \]
    
    Next, using a choice $J = \lceil (\log(2HT/\delta) + \log(N_\varepsilon)) / \log(1/(1-\gamma)) \rceil = \lceil \tilde{c}_J \cdot ( \log(2HT/\delta) + \log(N_\varepsilon)) \rceil$ 
    \[
        \P\left[\max_{j \in [J]}\left\{ \sum_{i=0}^{e_k} W^i_{j,e_k,k} X_i \right\} \geq p_h \Vstar_{h+1}(s,a) + \varepsilon L | \cE^{\star}(\delta, \varepsilon)\right] \geq 1 - (1 - \gamma)^{J} \geq 1 - \frac{\delta}{2N_{\varepsilon}HT}.
    \]
    By a union bound we conclude the statement.
\end{proof}

Next we provide a connection between $\cE^{\anticonc}$ and $\cE^{\opt}$.
\begin{proposition}\label{prop:opt_metric}
    It holds $\cE^{\opt} \subseteq \cE^{\anticonc}$.
\end{proposition}
\begin{proof}
    We proceed by a backward induction over $h$. Base of induction $h = H+1$ is trivial. Fix state-action pair $(s,a)$ and let us call $(s',a')$ a center of the ball $\psi_\varepsilon(s,a)$ that is the ball where $(s,a)$ contains.
    
    Next by the update formula for $\uQ^t_h$, and Bellman equations
    \begin{align*}
        \uQ^{t}_h(\psi_{\varepsilon}(s,a)) - \Qstar_h(s,a) &= \max_{j \in [J]} \biggl\{ \sum_{i=0}^{n} W^i_{j,n}  [r_h(s^{\ell^i}_h, a^{\ell^i}_h) - r_h(s',a')] \\
        &+ \sum_{i=0}^{n} W^i_{j,n} \uV^{\ell^i}_{h+1}(s^{\ell^i}_{h+1}) - p_h \Vstar_{h+1}(s',a') \biggl\} + [\Qstar_h(s,a) - \Qstar_h(s',a')],
    \end{align*}
    where $n = e_{k^t_h(B)}$ and we drop dependence on $k,t,h,s,a$ in $\ell^i$. By induction hypothesis we have $\uV^{\ell^i}_{h+1}(s') \geq \uQ^{\ell^i}_{h+1}(\psi_\varepsilon(s', \pi^\star(s'))) \geq \Qstar_{h+1}(s', \pi^\star(s')) = \Vstar_{h+1}(s')$ for any $i$, thus combining it with Lipchitz continuity of reward function and $\Qstar$, and the value of $r_h(s^{\ell^0}, a^{\ell^0}) = \ur > r_h(s,a)$, 
    \begin{align*}
        \uQ^{t}_h(\psi_{\varepsilon}(s,a)) - \Qstar_h(s,a) \geq& \max_{j \in [J]}\biggl\{ 
            W^0_{j,n} \ur (H-1) + \sum_{i=1}^n W^i_{j,n} \Vstar_{h+1}(F_h(s^{\ell^i}_h, a^{\ell^i}_h, \xi^{\ell^i}_h))
        \biggr\}\\
        &- p_h \Vstar_{h+1}(s',a') - (L_r + L_V) \varepsilon.
    \end{align*}
    Next we apply Lipschitz continuity of $F_h$ and $\Vstar_{h+1}$ and obtain 
    \begin{align*}
        \uQ^{t}_h(\psi_{\varepsilon}(s,a)) - \Qstar_h(s,a) \geq& \max_{j \in [J]}\biggl\{ 
            W^0_{j,n} \ur (H-1) + \sum_{i=1}^n W^i_{j,n} \Vstar_{h+1}(F_h(s', a', \xi^{\ell^i}_h))
        \biggr\}\\
        &- p_h \Vstar_{h+1}(s',a') - (L_r + L_V(1 + L_F)) \varepsilon.
    \end{align*}
    
    By the definition of event $\cE^{\anticonc}$ we conclude the statement.
\end{proof}

\begin{proposition}[Optimism]\label{prop:optimism_metric}

Define a constant $L = L_r + L_V(1+L_F)$.    
    Assume that $J = \lceil  \tilde{c}_J \cdot (\log(2HT/\delta)  + \log(N_\varepsilon) \rceil$, $\kappa = 2\beta^\star(\delta, T, \varepsilon)$, $\ur = 2$, and a prior count $n_0(k) = \ceil{\tn_0 + \kappa + \frac{\varepsilon L}{H-1} \cdot (e_k + \tn_0 + \kappa)}$ dependent on the stage $k$, where $\tn_0 = (c_0 + 1 + \log_{17/16}(2e_k)) \cdot \kappa$, $c_0$ is defined in \eqref{eq:constant_c0_metric}, $\tilde{c}_J$ is defined in \eqref{eq:constant_cJ_metric}.  Then $\PP{\cE^{\opt} \mid  \cE^\star(\delta, \varepsilon)} \geq 1-\delta/2$. 
\end{proposition}

\subsection{Regret Bounds}\label{app:regret_bound_metric}
As in the tabular setting, we first connect our algorithm to the algorithm by \cite{song2019efficient}, using the following corollary. Define an event $\cG'(\delta, \varepsilon) = \cG(\delta, \varepsilon) \cap \cE^\opt$.

Let us define the logarithmic term as follows
\[
    \beta^{\max}(\delta, \varepsilon) = \max\{ \kappa, \tn_0 / \kappa, \beta^B(\delta, \varepsilon), \beta(\delta, \varepsilon), \beta^{\conc}(\delta, \varepsilon) \}
\]
that has dependence of order $\cO( \log(TH/\delta) + \log N_{\varepsilon})$.

\begin{corollary}\label{cor:delta_q_bounds_metric}
    Fix $\varepsilon \in (0, L_V/H)$ and assume conditions of Proposition~\ref{prop:optimism_metric}. Let $t \in [T], h\in[H], B \in \cN_{\varepsilon}$. Define $k = k^t_h(B)$ and let $\ell^1 < \ldots < \ell^{e_k}$ be a excursions of $(B,h)$ till the end of the previous stage. Then on the event $\cG'(\delta)$ the following bound holds for $k \geq 0$ and any $(s,a) \in B$
    \[
        0\leq \uQ^t_h(B) - \Qstar_h(s,a) \leq \frac{1}{e_k} \sum_{i=1}^{e_k} [\uV^{\ell^i}_{h+1}(s^{\ell^i}_{h+1}) - \Vstar_{h+1}(s^{\ell^i}_{h+1}) ]  + \cB^t_h(k),
    \]
    where
    \[
        \cB^t_h(k) = 121\rme^2 \cdot \sqrt{\frac{H^2 (\beta^{\max}(\delta, \varepsilon))^2}{e_k}} + 2401\rme \cdot \frac{H (\beta^{\max}(\delta, \varepsilon))^4}{e_k}  + 3(L_r + (1 + L_F) L_V) \varepsilon.
    \]
\end{corollary}
\begin{proof}
    The lower bound follows from the definition of the event $\cE^{\opt}$. For the upper bound we first apply the decomposition for $\uQ^t_h(s,a)$ and the definition of event $\cE^B(\delta, \varepsilon)$ from Lemma~\ref{lem:proba_master_event_metric}
    \begin{align*}
        \uQ^t_h(B) &=  \max_{j \in [J]} \left\{ \sum_{i=0}^{e_k} W^i_{j,n}\left(r_h(s^{\ell^i}_h,a^{\ell^i}_h) +  \uV^{\ell^i}_{h+1}(s^{\ell^i}_{h+1})\right) \right\} \\
        &\leq \frac{1}{e_k + n_0(k)} \sum_{i=1}^{e_k}\left( r_h(s^{\ell^i}_h, a^{\ell^i}_h) + \uV^{\ell^i}_{h+1}(s^{\ell^i}_{h+1}) \right) + \frac{n_0(k) \cdot 2H}{e_k + n_0(k)} \\
        &+ 120 \rme^2 \sqrt{\frac{H^2 \kappa \beta^B(\delta, \varepsilon)}{e_k + n_0(k)}} + 2400\rme \frac{H \kappa \log(n + n_0(k)) (\beta^B(\delta, \varepsilon))^2}{e_k + n_0(k)}.
    \end{align*}
    Additionally, by Bellman equations
    \begin{align*}
        \Qstar_h(s,a) &= \frac{1}{e_k} \sum_{i=1}^{e_k} \Qstar_h(s^{\ell^i}_h,a^{\ell^i}_h) + \frac{1}{e_k} \sum_{i=1}^{e_k} \left(\Qstar_h(s,a) - \Qstar_h(s^{\ell^i}_h,a^{\ell^i}_h)  \right) \\
        &\geq \frac{1}{e_k} \sum_{i=1}^{e_k}\left( r_h(s^{\ell^i}_h,a^{\ell^i}_h) + p_h \Vstar_{h+1}(s^{\ell^i}_h,a^{\ell^i}_h) \right) - 2 \varepsilon L_V.
    \end{align*}
    Combining and using the fact that $n_0(k) \leq \frac{L \varepsilon}{H-1} \cdot (e_k + n_0(k)) + \tn_0 + \kappa$ for $L = L_r + (1+L_F) L_V$
    \begin{align*}
        \uQ^t_h(s,a) - \Qstar_h(s,a) &\leq \frac{1}{e_k} \sum_{i=1}^{e_k}\left[ \uV^{\ell^i}_{h+1} - \Vstar_{h+1} \right](s^{\ell^i}_{h+1}) +  \frac{1}{e_k} \sum_{i=1}^{e_k} \left[ \Vstar_{h+1} (s^{\ell^i}_{h+1})  - p_h \Vstar_{h+1}(s^{\ell^i}_h,a^{\ell^i}_h)\right] \\
        &+ 120\rme^2 \cdot \sqrt{\frac{H^2 (\beta^{\max}(\delta, \varepsilon))^2}{e_k}} + (2400\rme + 2)\frac{H (\beta^{\max}(\delta, \varepsilon))^4}{e_k}  + 3L \varepsilon.
    \end{align*}
    Finally, the applications of event $\cE^{\conc}(\delta, \varepsilon)$ concludes the statement.
\end{proof}

Let us define $\delta^t_h = \uV^{t}_h(s^t_h) - V^{\pi^{t}}_h(s^t_h)$ and $\zeta^{t}_h = \uV^{t}_h(s^t_h) - \Vstar_h(s^t_h)$. 
\begin{lemma}\label{lem:regret_decomp_metric}
    Assume conditions of Proposition~\ref{prop:optimism_metric}. Then on event $\cG'(\delta, \varepsilon) = \cG(\delta, \varepsilon) \cap \cE^{\opt}$, where $\cG(\delta, \varepsilon)$ is defined in Lemma~\ref{lem:proba_master_event_metric}, the following upper bound on regret holds
    \[
        \regret^T \leq 2\rme  H \sum_{t=1}^T \sum_{h=1}^H \ind\{N^t_h = 0\} + \sum_{t=1}^t \sum_{h=1}^H (1+1/H)^{H-h} \xi^t_h + \rme \sum_{t=1}^T \sum_{h=1}^H \cB^t_h.
    \]
    where $\xi^t_h = p_h [\Vstar_{h+1} - V^{\pi^{t}}_{h+1}](s^t_h,a^t_h) - [\Vstar_{h+1} - V^{\pi^{t}}_{h+1}](s^t_{h+1})$ and $\cB^t_h = \cB^t_h(k^t_h(s^t_h, a^t_h)) \cdot \ind\{k^t_h(s^t_h, a^t_h) \geq 0\}$ for $\cB^t_h$ defined in Corollary~\ref{cor:delta_q_bounds_metric}.
\end{lemma}
\begin{proof}
    As in the tabular setting, we notice that on the event $\cE^\opt$ we can upper bound the regret in terms of $\delta^t_1$.
    \begin{equation}\label{eq:regret_bound_deltas_metric}
        \regret^T \leq \sum_{t=1}^T \delta^t_1.
    \end{equation}
    Next we analyze $\delta^t_h$. Since $a^t_h = \argmax_{a\in \cA} \uQ^{t}_h(\psi_{\varepsilon}(s^t_h, a))$, we can use Corollary~\ref{cor:delta_q_bounds_metric} and Bellman equations in the following way
    \begin{align*}
        \delta^t_h &= \uV^{t}_h(s^t_h) - V^{\pi^{t}}_h(s^t_h) =  \uQ^{t}_h(B^t_h) - Q^{\pi^{t}}_h(s^t_h, a^t_h) \\
        &= \uQ^{t}_h(B^t_h) - \Qstar_h(s^t_h, a^t_h) + \Qstar_h(s^t_h, a^t_h) - Q^{\pi^{t}}_h(s^t_h, a^t_h) \\
        &\leq \ur H \ind\{N^t_h = 0\} + \ind\{ N^t_h > 0\}\left( \frac{1}{N^t_h} \sum_{i=1}^{N^t_h} \zeta^{\ell^i_{t,h}}_{h+1} + \cB^t_h(k^t_h) + p_h [\Vstar_{h+1} - V^{\pi^{t}}_{h+1}](s^t_h,a^t_h) \right).
    \end{align*}
    where $k^t_h = k^t_h(B^t_h)$, $N^t_h = e_{k^t_h}$, $\ell^i_{t,h}$ is an $i$-th visitation of the ball $B^t_h$ during an stage $k^t_h$, and additionally by a convention $0/0 = 0$.
    
    Define $\xi^t_h = p_h [\Vstar_{h+1} - V^{\pi^{t}}_{h+1}](s^t_h,a^t_h) - [\Vstar_{h+1} - V^{\pi^{t}}_{h+1}](s^t_{h+1})$ a martingale-difference sequence, and $\cB^t_h = \cB^t_h(k^t_h) \ind\{N^t_h > 0\}$ then 
    \[
        \delta^t_h \leq \ur H \ind\{N^t_h = 0\} + \frac{\ind\{N^t_h > 0\}}{N^t_h} \sum_{i=1}^{N^t_h} \zeta^{\ell^i_{t,h}}_{h+1}  - \zeta^t_{h+1} + \delta^t_{h+1} + \xi^t_h + \cB^t_h.
    \]
    and, as a result
    \begin{align*}
        \sum_{t=1}^T \delta^t_h &\leq \ur H \sum_{t=1}^T \ind\{N^t_h = 0\} + \sum_{t=1}^T \frac{\ind\{ N^t_h > 0\}}{N^t_h}\sum_{i=1}^{N^t_h} \zeta^{\ell^i_{t,h}}_{h+1} \\
        &-\sum_{t=1}^T\zeta^t_{h+1} + \sum_{t=1}^T \delta^t_{h+1} + \sum_{t=1}^T \xi^t_h + \sum_{t=1}^T \cB^t_h.
    \end{align*}

    For the second term we may repeat arguments as in the proof of Lemma~\ref{lem:regret_decomp} and obtain
    \[
        \sum_{q=1}^T \zeta^q_{h+1}  \cdot \sum_{t=1}^T \frac{\ind\{k^t_h \geq 0\}}{N^t_h} \sum_{i=1}^{N^t_h} \ind\{\ell^i_{t,h} = q\} \leq \left( 1 + \frac{1}{H} \right)\sum_{q=1}^T \zeta^q_{h+1}.
    \]

     After a simple algebraic manipulations and using the fact that $\zeta^t_h \leq \delta^t_h$
    \begin{align*}
        \sum_{t=1}^T \delta^t_h &\leq H \sum_{t=1}^T \ind\{N^t_h = 0\} + \sum_{t=1}^T (1 + 1/H) \zeta^t_{h+1} - \sum_{t=1}^T\zeta^t_{h+1} + \sum_{t=1}^T \delta^t_{h+1} + \sum_{t=1}^T \xi^t_h + \sum_{t=1}^T \cB^t_h \\
        &\leq  H \sum_{t=1}^T \ind\{N^t_h = 0\} + \left(1 + \frac{1}{H} \right) \sum_{t=1}^T \delta^t_{h+1} + \sum_{t=1}^T \xi^t_h + \sum_{t=1}^T \cB^t_h.
    \end{align*}
    By rolling out the upper bound on regret \eqref{eq:regret_bound_deltas_metric} we have
    \[
        \regret^T \leq 2\rme H \sum_{t=1}^T \sum_{h=1}^H \ind\{N^t_h = 0\} + \sum_{t=1}^t \sum_{h=1}^H (1+1/H)^{H-h} \xi^t_h + \rme \sum_{t=1}^T \sum_{h=1}^H \cB^t_h.
    \]
    
\end{proof}

\begin{proof}[Proof of Theorem~\ref{th:regret_randql_metric}]
    First, we notice that the event $\cG'(\delta, \varepsilon)$ defined in Lemma~\ref{lem:regret_decomp_metric}, holds with probability at least $1-\delta$ by Lemma~\ref{lem:proba_master_event_metric} and Proposition~\ref{prop:optimism_metric}. Thus, we may assume that $\cG'(\delta, \varepsilon)$ holds for $\varepsilon > 0$ that we will specify later.

    By Lemma~\ref{lem:regret_decomp_metric}
    \[
        \regret^T \leq 2\rme H \sum_{t=1}^T \sum_{h=1}^H \ind\{k^t_h = -1\} + \sum_{t=1}^t \sum_{h=1}^H (1+1/H)^{H-h} \xi^t_h + \rme \sum_{t=1}^T \sum_{h=1}^H \cB^t_h.
    \]
    
    The first term is upper bounded by $2\rme H^3 \cdot N_{\varepsilon}$, since there is no more than $H$ visits of each ball in $\varepsilon$-net before the update for the first stage. The second term is bounded by $\cO(\sqrt{H^3 T \beta^{\max}(\delta, \varepsilon)})$ by a definition of the event $\cE(\delta)$ in Lemma~\ref{lem:proba_master_event_metric}. 
    
    To analyze the last term, consider the following sum
    \[
        \sum_{t=1}^T \sum_{h=1}^H \frac{\ind\{ e_{k^t_h(B^t_h)} > 0 \}}{\sqrt{e_{k^t_h(B^t_h)}}} \leq \sum_{(B,h) \in \cN_{\varepsilon}\times [H]} \sum_{k=0}^{k^T_h(B)} \frac{e_{k+1}}{\sqrt{e_k}},
    \]
    where
    \[
        e_{k} = \left\lfloor \left(1+\frac{1}{H} \right)^k H \right\rfloor \Rightarrow \frac{e_{k+1}}{\sqrt{e_k}} \leq 2 \sqrt{H} \left(1 + \frac{1}{H} \right)^{k/2},
    \]
    therefore
    \begin{equation}\label{eq:upper_bound_sum_ek}
        \sum_{h=0}^{k^T_h(B)} \frac{e_{k+1}}{\sqrt{e_k}} \leq 4\sqrt{H} \frac{(1+1/H)^{(k^T_h(B)+1)/2} }{\sqrt{1+1/H}-1} = 4H \sqrt{e^{k^T_h(B)+1}}.
    \end{equation}
    Notice that 
    \[
        N^{T+1}_h(B) \geq \sum_{k=0}^{k^T_h(B)} e^{k} = H (e^{k^T_h(B)+1} - 1) \Rightarrow  e^{k^T_h(B)+1} \leq \frac{N^{T+1}_h(B) + 1}{H}
    \]
    thus from the Cauchy-Schwarz inequality
    \begin{align*}
        \sum_{t=1}^T \sum_{h=1}^H \frac{\ind\{ e_{k^t_h(B^t_h) > 0} \}}{\sqrt{e_{k^t_h(B^t_h)}}} &\leq 4\sqrt{H} \sum_{(B,h) \in \cN_{\varepsilon} \times [H]} \sqrt{N^{T+1}_h(B) + 1} \\
        &\leq 4\sqrt{SAH^2} \sqrt{\sum_{(B,h)} (N^{T+1}_h(B) + 1)} \leq 4 \sqrt{H^3T \cdot N_{\varepsilon}} + 4 N_{\varepsilon}H^2.
    \end{align*}
    By the similar arguments we have
    \[
        \sum_{t=1}^T \sum_{h=1}^H \frac{\ind\{ e_{k^t_h(B^t_h)} > 0 \}}{e_{k^t_h(B^t_h)}} \leq \cO\left( H N_{\varepsilon} \log(T)\right).
    \]
    
    Using this upper bound, we have for $L = L_r + (1+L_F)L_V$
    \begin{align*}
        \sum_{t=1}^T \sum_{h=1}^H \cB^t_h &= \cO\left( H \beta^{\max}(\delta, \varepsilon) \sum_{t=1}^T \sum_{h=1}^H \frac{\ind\{ e_{k^t_h(s^t_h,a^t_h)} > 0 \}}{\sqrt{e_{k^t_h(s^t_h, a^t_h)}}}\right)  \\
        &+ \cO\left(H (\beta^{\max}(\delta, \varepsilon))^4 \sum_{t=1}^T \sum_{h=1}^H \frac{\ind\{ e_{k^t_h(s^t_h,a^t_h)} > 0 \}}{\sqrt{e_{k^t_h(s^t_h, a^t_h)}}} \right) \\
        &+ \cO\left( LTH \varepsilon \right) \\
        &\leq \cO\left( \sqrt{H^5 T N_{\varepsilon} \cdot (\beta^{\max}(\delta, \varepsilon))^2} + H^3 N_\varepsilon (\beta^{\max}(\delta, \varepsilon))^4 + LTH\varepsilon \right).
    \end{align*}
    Overall, for any fixed $\varepsilon > 0$ we have
    \[
        \regret^T \leq \cO\left( \sqrt{H^5 T N_{\varepsilon} \cdot (\beta^{\max}(\delta, \varepsilon))^2} + H^3 N_\varepsilon (\beta^{\max}(\delta, \varepsilon))^4 + LTH\varepsilon + \sqrt{H^3 T}\right).
    \]

    Next we finally use that $\cS \times \cA$ have covering dimension $d_c$ that means$N_{\varepsilon} \leq C_N \cdot \varepsilon^{-d_c}$, thus our regret bound transforms as follows
    \begin{align*}
        \regret^T &\leq \cO\biggl( \sqrt{H^5 T C_N \varepsilon^{-d_c} \cdot (\log(TC_NH/\delta)  + d_c \log(1/\varepsilon) )^2} \\
        &\quad + H^3 C_N \varepsilon^{-d_c} (\log(TC_NH/\delta) + d_c \log(1/\varepsilon) )^4 + LTH\varepsilon \biggl).
    \end{align*}
    By taking $\varepsilon = T^{-1/(d_c+2)}$ we conclude the statement
    
\end{proof}

\newpage
\section{Adaptive \texorpdfstring{\RandQL}{RandQL}}\label{app:adaptive_randql_proof}
In this section we describe how to improve the dependence in our algorithm from covering dimension to zooming dimension, and describe all required notation.

\subsection{Additional Notation}

In this section we introduce an additional notation that is needed for introducing an adaptive version of \RandQL algorithm for metric spaces.

\paragraph{Hierarchical partition}

Next, we define all required notation to describe an adaptive partition, as \cite{sinclair2019adaptive,sinclair2022adaptive}. Finally, we define the following general framework of hierarchical partition. Instead of balls, we will use a more general notion of regions that will induce a better structure from the computational point of view. We recall for any compact set $A \subseteq \cS \times \cA$ we call $\diam(A) = \max_{x,y \in A} \rho(x,y)$.

\begin{definition}
    A hierarchical partition of $\cS \times \cA$ of a depth $d > 0$ is a collection of regions $\cP_{d}$ and their centers such that
    \begin{itemize}
        \item Each region $B \in \cP_d$  is of the form $\cS(B) \times \cA(B)$, where $\cS(B) \subseteq \cS, \cA(B) \subseteq \cA$;
        \item $\cP_d$ is a cover of $\cS \times \cA$: $\bigcup_{B \in \cP_{d}} B = \cS \times \cA$;
        \item For every $B \in \cP_d,$ we have $\diam(B) \leq d_{\max} \cdot 2^{-d}$;
        \item Let $B_1, B_2 \in \cP_d.$ If $B_1 \not = B_2$ then $\rho(\cent(B_1), \cent(B_2)) \geq d_{\max} \cdot 2^{-d} $;
        \item For any $B \in \cP_d,$ there exists a unique $A \in \cP_{d-1}$ (called the parent of $B$) such that $B \subseteq A$.
    \end{itemize}
    and, for $d=0$ we define it as $\cP_0 = \{ \cS \times \cA \}$.
\end{definition}
We call the tree generated by the structure of $\cT = \{ \cP_d\}_{d \geq 0}$ a tree of this hierarchical partition.
The main example of this partition is the dyadic partition of $\cS \times \cA$ in the case of $\cS = [0,1]^{d_{\cS}}, \cA = [0,1]^{d_{\cA}}$ and the metric induced by the infinity norm $\rho((s,a), (s',a')) = \max\{ \norm{s-s'}_\infty, \norm{a-a'}_\infty \}$.
For  examples we refer to \citep{sinclair2022adaptive}. 

\subsection{Algorithm}

In this section we describe two algorithms: \AdaptiveRandQL which is an adaptive metric counterpart of \RandQL, and \AdaptiveStagedRandQL which is an adaptive metric counterpart of \StagedRandQL. First, we start from the notation and algorithmic parts that will be common for both algorithms.

Algorithms maintain an adaptive partition $\cP^t_h$ of $\cS \times \cA$, that is a sub-tree of an (infinite) tree of the hierarchical partition $\cT = \{ \cP_d \}_{d \geq 0}$.  We initialize $\cP^1_h = \{ \cP_0 \}$, and then we refine the tree $\cP^t_h$ be adding new nodes that corresponding to nodes of $\cT$. The leaf nodes of $\cP^t_h$ represent the active balls, and for $B \in \cP^t_h$ the set of its inactive parent balls is defined as $\{ B' \in \cP^t_h \mid B \subset B' \}$. For any $B \in \cP^t_h$ we define $d(B)$ as a depth of $B$ in the tree under a convention $d(\cS \times \cA) = 0$.

Additionally, we need to define so-called \textit{selection rule} and \textit{splitting rule}. For any state $s \in \cS$ we define the set of all relevant balls as $\cR^t_h(s) = \{ $ active $ b \in \cP^t_h \mid (s,a) \in B $ for some $a \in \cA \}$. Then for the current state $s^t_h$ we define the current ball as $B^t_h = \argmax_{B \in \cR^t_h(s^t_h)}  \uQ^t_h(B)$ and the corresponding action as $a^t_h$. To define the splitting rule we maintain the counters $n^t_h(B)$ for all $B \in \cP^t_h$ as a number of visits of a node $B$ and all its parent nodes. Then we will perform splitting of the current ball $B^t_h$ if $\sqrt{ d^2_{\max} / n^t_h(B^t_h)} \leq \diam(B^t_h)$. During splitting, we extend $\cP^{t+1}_h$ by its child nodes in the hierarchical partition tree $\cT$. For more details we refer to \citep{sinclair2022adaptive}, up to small changes in notation. In particular, their constant $\tilde{C}$ is equal to $d_{\max}$ in our setting to make the construction exactly the same for both \AdaptiveRandQL and \AdaptiveStagedRandQL algorithms.

\paragraph{\AdaptiveRandQL}

\begin{algorithm}[h!]
\centering
\caption{\AdaptiveRandQL}
\label{alg:AdaptiveRandQL}
\begin{algorithmic}[1]
  \STATE {\bfseries Input:} ensemble size $J$, number of prior transitions $n_0$, prior reward $r_0$.
  \STATE {\bfseries Initialize: }  $\uQ_h(B) = \tQ^j_h(B) =\ur H,$ initialize counters $n_h(s,a) = 0$ for $h,s,a\in[H]\times\cS\times\cA$.
      \FOR{$t \in[T]$}
      \FOR{$h \in [H]$}
        \STATE Compute $B_h = \argmax_{B \in \cR^t_h(s_h)} \uQ_h(B)$ and play $a_h$ for $(s_h,a_h) = \cent(B_h)$;
        \STATE Observe reward and next state $s_{h+1}\sim p_h(s_h,a_h)$.
        \STATE Sample $\rw_j \sim \Beta(n, n_0)$ for $n = n_h(B_h)$.
        \STATE Compute value $\uV_{h+1}(s_{h+1}) = \max_{B \in \cR^t_h(s_{h+1})} \uQ_{h+1}(B)$.
        \STATE Build targets for all $j \in [J]$
        \[
        \rQ_{h}^{j} = \rw_{j} [r_h(s_h,a_h) + \uV_{h+1}(s_{h+1})] + (1-\rw_{j}) \ur H
        \,.\]
        \vspace{-0.4cm}
        \STATE Sample learning rates $w_j \sim \Beta(H, n)$.
        \STATE Update ensemble $Q$-functions for all $j \in [J]$
        \[
            \tQ^{j}_h(B_h) := (1 - w_j) \tQ^{j}_h(s_h,a_h) + w_j \rQ_{h}^{j}\,.
        \]
        \vspace{-0.4cm}
        \STATE Update policy $Q$-function $\uQ_h(s_h,a_h) := \max_{j \in [J]} \tQ^{j}_h(s_h,a_h)$.
        \STATE Update counters $n_h(B_h) := n_h(B_h) + 1$;
        \STATE If $\sqrt{d^2_{\max} / n_h(B_h)} \leq \diam(B_h)$, then refine partition $B_h$ (see \cite{sinclair2022adaptive}).
      \ENDFOR
  \ENDFOR
\end{algorithmic}
\end{algorithm}

This algorithm is an adaptive metric version of \RandQL algorithm.  We recall that for $B \in \cP^t_h$ we define $n^t_h(B) = \sum_{i=1}^{t-1} \ind\{ (B^i_h) \text{ is a parent of } B \}$ is the number of visits of the ball $B$ and its parent balls at step $h$ before episode $t$. We start by initializing the ensemble of Q-values, the policy Q-values, and values to an optimistic value $\tQ_h^{t,j}(B) = \uQ_h^{1}(B) = \uV^1_h(B) = r_0 H$ for all $(j,h)\in[J]\times[H]$ and the unique ball in the partition $B = \cS \times \cA$ and $r_0>0$ some pseudo-rewards.

At episode $t$ we update the ensemble of Q-values as follows, denoting by $n=n^t_h(B)$ the count, $w_{j,n} \sim \Beta(H, n)$ the independent learning rates, 
\[
\tQ^{t+1,j}_h(B) = \begin{cases}
        (1- w_{j,n}) \tQ^{t,j}_h(B) + w_{j,n} \rQ_{h}^{t,j}(s^t_h, a^t_h), & B = B^t_h \\
        \tQ^{t,j}_h(B) & \text{otherwise},
    \end{cases}
\]
where we defined the target $\rQ_{h}^{t,j}(s^t_h, a^t_h)$ as a mixture between the usual target and some prior target with mixture coefficient $\rw_{n,j}\sim\Beta(n, n_0)$ and $n_0$ the number of prior samples,
\[
\rQ_{h}^{t,j}(s^t_h, a^t_h) = \rw_{j,n} [r_h(s^t_h, a^t_h) + \uV^{t}_{h+1}(s^t_{h+1})] + (1-\rw_{j,n}) r_0 H\,.
\]
For a discussion on prior re-injection we refer to Appendix~\ref{app:description_randQL}. The value function is computed on-flight by the rule $\uV^t_h(s) = \max_{B \in \cR^t_h} \uQ^t_h(B)$.

The policy Q-values are obtained by taking the maximum among the ensemble of Q-values
\[
\uQ_h^{t+1}(B) = \max_{j\in[J]} \tQ_h^{t+1,j}(B)\,.
\]
The policy is then greedy with respect to the policy Q-values and selection rule  $(s,\pi_h^{t+1}(s)) = \cent(B)$, where $B =  \argmax_{B \in \cR^{t+1}_h} \uQ_h^{t+1}(B)$. After the update of Q-values, algorithm verifies the splitting rule. If the splitting rule is triggered, then all new balls are defined by counter and Q-values of its parent. We notice that all Q-values could be efficiently computed on the nodes of the adaptive partition. The complete and detailed description is presented in Algorithm~\ref{alg:AdaptiveStagedRandQL}.

\paragraph{\AdaptiveStagedRandQL}

\begin{algorithm}
\centering
\caption{\AdaptiveStagedRandQL}
\label{alg:AdaptiveStagedRandQL}
\begin{algorithmic}[1]
  \STATE {\bfseries Input:}  inflation coefficient $\kappa$, $J$ ensemble size, number of prior transitions $n_0(k)$, prior reward $r_0$.
  \STATE {\bfseries Initialize: } $\uQ_h(B) = \tQ^j_h(B) = \ur H,$ initialize counters $\tn_h(B) = 0$ for $j,h,B\in[J]\times[H]\times\cN_{\varepsilon}$, stage $q_h(B) = 0$.
      \FOR{$t \in[T]$}
      \FOR{$h \in [H]$}
        \STATE Compute $B_h = \argmax_{B \in \cR^t_h(s_h)} \uQ_h(B)$ and play $a_h$ for $(s_h,a_h) = \cent(B_h)$;
        \STATE Observe reward and next state $s_{h+1}\sim p_h(s_h,a_h)$.
        \STATE Sample learning rates $w_j \sim \Beta(1/\kappa, (\tn+n_0(q_h(B_h))/\kappa)$ for $\tn = \tn_h(B_h)$.
        \STATE Compute value  $\uV_{h+1}(s_{h+1}) = \max_{B \in \cR^t_h(s_{h+1})} \uQ_{h+1}(B)$.
        \STATE Update temporary $Q$-values for all $j \in [J]$
        \[
            \tQ^{j}_h(B) := (1 - w_j) \tQ^{j}_h(B) + w_j \left( r_h(s_h,a_h) + \uV_{h+1}(s_{h+1})\right)\,.
        \]
        \vspace{-0.4cm}
        \STATE Update counters $\tn_h(B_h) := \tn_h(B_h) + 1$ and $n_h(B_h) := n_h(B_h) + 1$.
        \IF{$\tn_h(B_h) = \lfloor (1 + 1/H)^{q} H \rfloor$ for $q=q_h(B_h)$ is the current stage}
            \STATE Update policy $Q$-values $\uQ_h(B_h) := \max_{j \in [J]} \tQ^{j}_h(B_h)$.
            \STATE Reset temporary  $Q$-values $\tQ^j_h(B_h) := \ur H$.
            \STATE Reset counter $\tn_h(B_h) := 0$ and change stage $k_h(B_h) := k_h(B_h) + 1$.
        \ENDIF
        \STATE If $\sqrt{d^2_{\max} / n_h(B_h)} \leq \diam(B_h)$, then refine partition $B_h$ (see \cite{sinclair2022adaptive}).
      \ENDFOR
  \ENDFOR
\end{algorithmic}
\end{algorithm}

The notation for this algorithm is very close to \NetStagedRandQL and we describe only differences between them. The main difference is a way to compute value $\uV^t_h(s) = \max_{B \in \cR^t_h(s)} \uQ^t_h(B)$ and policy $(s,\pi^t_h(s)) = \cent(B)$ for $B = \argmax_{B \in \cR^t_h(s)} \uQ^t_h(B)$. Additionally, all counters including temporary will move to the child nodes after splitting, as it performed in \AdaptiveRandQL. The detailed description is presented in Algorithm~\ref{alg:AdaptiveStagedRandQL}.

\subsection{Regret Bound}

In this section we state the regret bounds for \AdaptiveStagedRandQL and derive a proof. The given proof shares a lot of similarities with the proof of \NetStagedRandQL in the first half and to the proof of \AdaptiveQL by \cite{sinclair2022adaptive} in the second half.

We fix $\delta\in(0,1),$  and the number of posterior samples 
\begin{equation}
\label{eq:def_J_adaptive}
    J \triangleq \lceil \tilde{c}_J \cdot ( \log(2C_NHT/\delta) + d_c \log_2(8T/d_{\max})  ) \rceil,
\end{equation}
where $\tilde{c}_J = 1/\log(4/(3 + \Phi(1)))$ and $\Phi(\cdot)$ is the  cumulative distribution function (CDF) of a normal distribution.

Additionally we select
\[
    n_0(k) = \left\lceil \tn_0 + \kappa + \frac{L \cdot d_{\max}}{H-1} \cdot \frac{e_k + \tn_0 + \kappa}{\sqrt{H e_k - k - H^2}} \right\rceil, \quad \tn_0 = (c_0 + 1 + \log_{17/16}(T)) \cdot \kappa
\]
where  $c_{0}$ is an absolute constant defined in \eqref{eq:constant_c0} (see Appendix~\ref{app:optimism_tabular}), $\kappa$ is the posterior inflation coefficient and $L = L_r + (1+L_F)L_V$ is a  constant. Next we restate the regret bound result for \AdaptiveStagedRandQL algorithm.
\begin{theorem*}[Restatement of Theorem~\ref{th:regret_randql_adaptive}]
    Consider a parameter $\delta \in (0,1)$.
    Let $\kappa \triangleq 2(\log(8HC_N/\delta) + d_c \log_2(8T/d_{\max}) + 3\log(\rme\pi(2T+1)))$, $\ur \triangleq 2$. Then it holds for \AdaptiveStagedRandQL, with probability at least $1-\delta$, 
    \[
        \regret^T = \tcO\biggl( L H^{3/2} \sum_{h=1}^H  T^{\frac{d_{z,h}+1}{d_{z,h}+2}}\biggl),
    \]
    where $d_{z,h}$ is the step-$h$ zooming dimension and we ignore all multiplicative factors in the covering dimension $d_c$.
\end{theorem*}

\begin{proof}
    We divide the proof to four main parts, a little bit different proof of \StagedRandQL and \NetStagedRandQL since we also need to apply clipping techniques.

    \paragraph{Concentration events}

    We can define (almost) the same set of events as in Appendix~\ref{app:concentration_metric}, where union bound over balls is taken over all the hierarchical partition tree up to depth $D$ that we define as $\cT_D$. 

    \begin{align*}
  \cE^\star(\delta) &\triangleq \Bigg\{\forall t \in \N, \forall h \in [H], \forall B\in \cT_D, k = k^t_h(B), (s,a) = \cent(B): \\
  &\qquad
    \Kinf\left( \frac{1}{e_k}\sum_{i=1}^{e_k} \delta_{\Vstar_{h+1}(F_h(s,a,\xi^{\ell^i}_{h+1}))}, p_h \Vstar_{h+1}(s,a) \right) \leq  \frac{\betastar(\delta,e_k,\varepsilon)}{e_k}\Bigg\}\,,\\
    % Next event
    \cE^{B}(\delta, T) &\triangleq \Bigg\{ \forall t \in [T], \forall h \in [H], \forall B \in \cT_D, \forall j \in [J], k = k^t_h(B):  \\
    &\qquad \left| \sum_{i=0}^{e_k} \left( W^i_{j, e_k, k} - \E[W^i_{j, e_k,k}] \right) \left( r_{h}(s^{\ell^i}_h, a^{\ell^i}_h) +  \uV^{\ell^i}_{h+1}(s^{\ell^i}_{h+1}) \right) \right| \\
    &\qquad\qquad\leq 60 \rme^2 \sqrt{\frac{\ur^2 H^2 \kappa \beta^B(\delta, \varepsilon)}{e_k + n_0(k)}} + 1200 \rme \frac{\ur H \kappa \log(e_k + n_0(k)) (\beta^B(\delta, \varepsilon))^2}{e_k + n_0(k)} \bigg\}\,,\\
    % Next event
    \cE^{\conc}(\delta, T) &\triangleq \Bigg\{\forall t \in [T], \forall h \in [H], \forall B \in \cT_D, k = k^t_h(B): \\
    &\qquad\left|\frac{1}{e_k} \sum_{i=1}^{e_k} \Vstar_{h+1}(s^{\ell^i_{k,h}(B)}_{h+1})  - p_h\Vstar_{h+1}(s^{\ell^i_{k,h}(B)}_h, a^{\ell^i_{k,h}(B)}_h) \right| \leq \sqrt{\frac{2\ur^2 H^2 \beta^{\conc}(\delta, \varepsilon)}{e_k}}\Bigg\},\\
    % Next event
    \cE(\delta) &\triangleq \Bigg\{  \sum_{t=1}^T \sum_{h=1}^H (1+3/H)^{H-h}\left| p_h[\Vstar_{h+1} - V^{\pi_t}_{h+1}](s^t_h, a^t_h) - [\Vstar_{h+1} - V^{\pi_t}_{h+1}](s^t_{h+1})\right| \\
    &\qquad\qquad\qquad\qquad\qquad\qquad\qquad\qquad\qquad\qquad\qquad\quad\leq 2\rme^3 \ur H\sqrt{2HT \beta(\delta)}.
    \end{align*}
    
    To apply the union bound argument, we have to bound the size of $\cT_D$. First, we notice that relation between centers of balls in each layer $\cP_d$ implies that there at least $\vert \cP_d \vert$ non-intersecting balls of radius $d_{\max} \cdot 2^{-d-2}$. Thus, the size of this sub-tree could be bounded as
    \[
        |\cT_D| \leq \sum_{d=0}^D N_{d_{\max} 2^{-d-2}} \leq C_N \sum_{d=0}^D \left( 2^{d+2}/d_{\max}\right)^{d_c} \leq (8/d_{\max})^{d_c} C_N \cdot 2^{d_c \cdot D}.
    \] 
    using the relation between covering and packing numbers, see e.g. Lemma 4.2.8 by \cite{vershynin2018high}. The only undefined quantity here is $D$, that can be upper-bounded given budget $T$. To do it, we apply Lemma B.2 by \cite{sinclair2022adaptive} for any $B \in \cP^t_h$
    \begin{equation}\label{eq:diam_through_n}
        \left( \frac{d_{\max}}{2\cdot\diam(B)} \right)^2 \leq n^t_h(B) \leq \left( \frac{d_{\max}}{\diam(B)} \right)^2.
    \end{equation}
    Our goal is to find a value $D$ such that $\cP^{T+1}_h \subseteq \cT_D$ for any MDPs and correct interactions. To do it, we notice that it is equivalent to show that $\diam(B) \geq d_{\max} 2^{-D}$, that could be guaranteed since
    \[
        \diam(B) \geq \frac{d_{\max}}{2 \sqrt{n^{T+1}_h(B)}} \geq \frac{d_{\max}}{2T},
    \]
    which implies that $D = 1+\log_2(T)$ is enough. Finally, since for the value of interest
    \[
        \log | \cT_D | \leq d_c \log_2(T) + 
        \log C_N + d_c \log(8/d_{\max}),
    \]
    we can define the $\beta$-functions as follows follows

    \begin{align*}
    \betastar(\delta) &\triangleq \log(8C_N H/\delta) + d_c \log_2(8T/d_{\max})  + 3\log\left(\rme\pi(2n+1)\right)\,,\\
    \beta^B(\delta, T) &\triangleq  \log(8C_N H/\delta) + d_c \log_2(8T/d_{\max})  +  \log(TJ)\,,\\
    \beta^{\conc}(\delta, T) &\triangleq \log(8C_N H/\delta) + d_c \log_2(8T/d_{\max})  +  \log(2T) ,\\
    \beta(\delta) &\triangleq \log(16C_N H/\delta) + d_c \log_2(8T/d_{\max}),
\end{align*}
    and following line-by-line the proof of Lemma~\ref{lem:proba_master_event_metric}, for an event $\cG(\delta) = \cE^{\star}(\delta)\cap \cE^{B}(\delta, T) \cap \cE^{\conc}(\delta, T) \cap \cE(\delta) $ we have $\PP{\cG(\delta)} \geq 1 - \delta/2$.

    \paragraph{Optimism}

    Next, we state the required analog of Proposition~\ref{prop:anticonc_metric}. We can show that with probability at least $1-\delta/2$  on the event $\cE^\star(\delta)$ the following event
    \begin{align*}
        &\cE_{\anticonc} \triangleq \Bigg\{ \forall t \in [T] \ \forall h \in [H] \ \forall B \in \cT_D: \text{for } k = k^t_h(B), (s,a) = \cent(B):  \\
        &\max_{j \in [J]}\biggl\{ 
            W^0_{j,e_k,k} \ur (H-1) + \sum_{i=1}^{e_k} W^i_{j,e_k,k} \Vstar_{h+1}(F_h(s, a, \xi^{\ell^i}_h))
        \biggr\} \geq p_h \Vstar_{h+1}(s,a) + L \cdot \diam(B^t_h) \Bigg\}
    \end{align*}
    under the choice $J = \lceil  \tilde{c}_J \cdot (\log(2HT/\delta)  + \log(\vert \cT_D \vert)) \rceil$, $\kappa = 2\beta^\star(\delta, T)$, $\ur = 2$, and a prior count 
    \[
        n_0(k) = \ceil{\tn_0 + \kappa + \frac{L \cdot d_{\max}}{H-1} \cdot \frac{e_k + \tn_0 + \kappa}{\sqrt{H e_k - k - H^2}}}
    \] dependent on the stage $k$, where $\tn_0 = (c_0 + 1 + \log_{17/16}(T)) \cdot \kappa$, $L = L_r + L_V(1+L_F)$. In particular, the proof exactly the same as the proof of Proposition~\ref{prop:anticonc_metric} for $\varepsilon$ dependent on $k$.

    At the same time, it is possible to show that $\cE_{\anticonc}$ implies
    \begin{equation}\label{eq:opt_event_adaptive}
    \cE_{\opt} \triangleq \left\{ \forall t \in [T], h \in [H], \forall B \in \cP^t_h, \forall (s,a) \in B:  \uQ^t_h(B) \geq \Qstar_h(s,a) \right\}.
\end{equation}

    Indeed, in the proof of Proposition~\ref{prop:opt_metric} we actively uses the bound $\rho((s^{\ell^i}_h, a^{\ell^i}_h), (s,a)) \leq \varepsilon$. In the adaptive setting, we have to, at first, use an upper bound $\rho((s^{\ell^i}_h, a^{\ell^i}_h), (s,a)) \leq \diam(B^{\ell^i}_h)$ by a construction $B \subseteq B^{\ell^i}_h$, and then apply Lemma B.2 by \cite{sinclair2022adaptive} by defining an upper bound 
    \[
        \diam(B^{\ell^i}_h) \leq \frac{d_{\max}}{\sqrt{n^{\ell^i_h}_h(B^{\ell^i}_h)}} \leq \frac{d_{\max}}{\sqrt{\sum_{i=0}^{k-1} e_i}} \leq \frac{d_{\max}}{\sqrt{H\sum_{i=0}^{k-1} (1+1/H)^i - k}} \leq \frac{d_{\max}}{\sqrt{ H e_k - k - H^2}}
    \]
    for $k = k^t_h(B)$ for a particular ball $B \in \cP^t_h$ in the case $H e_k - k - H^2 \geq 0$.

    By combining event $\cE_{\opt}$ and the event $\cE^B(\delta)$ we can prove the same statement as Corollary~\ref{cor:delta_q_bounds_metric}.

    Let $t \in [T], h\in[H], B \in \cP^t_h$. Define $k = k^t_h(B)$ and let $\ell^1 < \ldots < \ell^{e_k}$ be a excursions of $(B,h)$ till the end of the previous stage. Then on the event $\cG'(\delta) = \cG(\delta) \cap \cE_{\opt}$ the following bound holds for $k \geq 0$ and for any $(s,a) \in B$
    \begin{equation}\label{eq:q_value_bonus_bound}
        0\leq \uQ^t_h(B) - \Qstar_h(s,a) \leq H \ind\{He_{k}/2 \leq k + H^2 \} + \frac{1}{e_k} \sum_{i=1}^{e_k} [\uV^{\ell^i}_{h+1}(s^{\ell^i}_{h+1}) - \Vstar_{h+1}(s^{\ell^i}_{h+1}) ]  + \cB^t_h,
    \end{equation}
    where
    \begin{equation}\label{eq:bonus_value_adaptive}
        \cB^t_h = 121\rme^2 \cdot \sqrt{\frac{H^2 (\beta^{\max}(\delta, T))^2}{e_k}} + 2401\rme \cdot \frac{H (\beta^{\max}(\delta, T))^4}{e_k}  + \frac{5 L \cdot d_{\max}}{\sqrt{He_k}}
    \end{equation}
    where $k= k^t_h(B^t_h)$ and $\beta^{\max}(\delta,T) = \max\{ \betastar(\delta, T), \beta^{B}(\delta), \beta^{\conc}(\delta), \beta(\delta)\}$. Also we can express this bound in terms of a diameter of $B^t_h$ as follows
    \begin{align*}
        \diam(B^t_h) &\geq \frac{d_{\max}}{2\sqrt{n^t_h(B^t_h)}} \geq \frac{d_{\max}}{2\sqrt{\sum_{i=0}^{k} e_i}} \geq \frac{d_{\max}}{2\sqrt{H \sum_{i=0}^{k} (1+1/H)^i}} \geq \frac{d_{\max}}{2\sqrt{H}} \\
        &\geq \frac{d_{\max}}{2\sqrt{H^2 (1+1/H)^{k+1}}} \geq \frac{d_{\max}}{2\sqrt{2He_{k}}},
    \end{align*}
    thus
    \[
        \frac{1}{\sqrt{H e_k}} \leq \frac{3\diam(B^t_h)}{d_{\max}},
    \]
    and we have
    \begin{equation}\label{eq:bonus_value_adaptive_ub}
        \begin{split}
            \cB^t_h &\leq 7566 \rme^2 H^{3/2} (\beta^{\max}(\delta,T))^4 \diam(B^t_h) / {d_{\max}} + 15 L \diam(B^t_h) \\
            &\leq \rho(H, \delta, L) \cdot \diam(B^t_h),
        \end{split}
    \end{equation}
    where we define $\rho(H,\delta,L) \triangleq 7566 \rme^2 H^{3/2} (\beta^{\max}(\delta,T))^4 / d_{\max} + 15L$.

    As a additional corollary, we have for all $t\in [T], h \in [H]$
    \begin{equation}\label{eq:value_optimism}
        \uV^t_h(s) = \max_{B \in \cR^t_h(s)} \uQ^t_h(B) = \uQ^t_h(B^\star) \geq \Qstar_h(s, \pistar(s)) = \Vstar_h(s),
    \end{equation}
    where $B^\star$ is a ball that contains a pair $(s, \pistar(s))$.

    This upper and lower bound have the similar structure as Lemma D.2 by \cite{sinclair2022adaptive} and the rest of the proof directly follows \citep{sinclair2022adaptive}.

    \paragraph{Clipping techniques}

    Next we introduce the required clipping techniques developed by \cite{simchowitz2019nonasymptotic,cao2020provably}. Definition~\ref{def:value_gap} introduces the quantity $\gap_h(s,a) = \Vstar_h(s) - \Qstar_h(s,a)$, and for any compact set $B  \subseteq \cS \times \cA$ we define $\gap_h(B) = \min_{(s,a) \in B} \gap_h(s,a)$. Finally, we define clipping operator for any $\mu, \nu \in \R$
    \begin{equation}\label{eq:def_clip}
        \clip(\mu | \nu) = \mu \ind\{ \mu \leq \nu \}.
    \end{equation}
    In particular, this operator satisfies the following important property
    \begin{lemma}[Lemma E.2. of \cite{sinclair2022adaptive}]\label{lem:clip_ub}
        Suppose that $\gap_h(B) \leq \psi \leq \mu_1 + \mu_2$ for any $\psi, \mu_1, \mu_2$. Then 
        \[
            \psi \leq \clip\left[\mu_1 \bigg| \frac{\gap_h(B)}{H+1} \right] + \left(1 + \frac{1}{H}\right) \mu_2
        \]
    \end{lemma}

    Now we apply this lemma to our update rules, producing a result similar to Lemma~E.3 of \cite{sinclair2022adaptive}. We notice that
    \begin{align*}
        \gap_h(B^t_h) &\leq \gap_h(s^t_h,a^t_h) = \Vstar_h(s^t_h) - \Qstar_h(s^t_h,a^t_h) \\
        &\leq \uV^t_h(s^t_h) - \Qstar_h(s^t_h,a^t_h) = \uQ^t_h(B^t_h) - \Qstar_h(s^t_h, a^t_h).
    \end{align*}
    Thus, denoting $\psi =  \uQ^t_h(B^t_h) - \Qstar_h(s^t_h, a^t_h)$ and, by \eqref{eq:q_value_bonus_bound},
    \[
        \mu_1 = H \ind\{He_{k^t_h}/2 > k^t_h + H^2 \}  + \cB^t_h, \quad \mu_2 = \frac{1}{e_k} \sum_{i=1}^{e_k} [\uV^{\ell^i}_{h+1}(s^{\ell^i}_{h+1}) - \Vstar_{h+1}(s^{\ell^i}_{h+1}) ]
    \]
    we apply Lemma~\ref{lem:clip_ub} and obtain
    \begin{align}\label{eq:value_ub_clip}
        \begin{split}
        \uV^t_h(s^t_h) - \Qstar_h(s^t_h,a^t_h) &\leq \clip\left[ H \ind\{He_{k^t_h}/2 \leq k^t_h + H^2 \}  + \cB^t_h | \frac{\gap_h(B^t_h)}{H+1} \right] \\
        &+ \left( 1 + \frac{1}{H}\right)\frac{1}{e_k} \sum_{i=1}^{e_k} [\uV^{\ell^i}_{h+1}(s^{\ell^i}_{h+1}) - \Vstar_{h+1}(s^{\ell^i}_{h+1}) ]
        \end{split}
    \end{align}
    for $k^t_h = k^t_h(B^t_h)$ and $\cB^t_h$ defined in \eqref{eq:bonus_value_adaptive}.

    \paragraph{Regret decomposition} The rest of the analysis we preform conditionally on event $\cG'(\delta) = \cG(\delta) \cap \cE_{\opt}$ that holds with probability at least $1-\delta$.

    By defining $\delta^t_h = \uV^t_h(s^t_h) - V^{\pi^t}(s^t_h)$ and $\zeta^t_h = \uV^t_h(s^t_h) - \Vstar_h(s^t_h)$ we have
    \[
        \regret^T = \sum_{t=1}^T \Vstar_1(s^t_1) - V^{\pi^t}_1(s^t_1) \leq \sum_{t=1}^T \delta^t_1,
    \]
    and, at the same time, by Bellman equations
    \begin{align*}
        \delta^t_h &= \uV^t_h(s^t_h) - Q^{\pi^t_h}(s^t_h, a^t_h) =  \uV^t_h(s^t_h) - \Qstar_h(s^t_h, a^t_h) + \Qstar_h(s^t_h, a^t_h) -   Q^{\pi^t}(s^t_h, a^t_h)\\
        &=  \uV^t_h(s^t_h) - \Qstar_h(s^t_h, a^t_h) + \Vstar_{h+1}(s^t_{h+1}) - V^{\pi^t}_h(s^t_{h+1}) + \xi^t_h \\
        &= \uV^t_h(s^t_h) - \Qstar_h(s^t_h, a^t_h) + \delta^t_{h+1} - \zeta^t_{h+1} + \xi^t_h,
    \end{align*}
    where $\xi^t_h = p_h [\Vstar_{h+1} - V^{\pi^{t}}_{h+1}](s^t_h,a^t_h) - [\Vstar_{h+1} - V^{\pi^{t}}_{h+1}](s^t_{h+1})$ is a martingale-difference sequence.  By \eqref{eq:value_ub_clip} we have
    \begin{align*}
        \sum_{t=1}^T \delta^t_h &= \sum_{t=1}^T \uV^t_h(s^t_h) - \Qstar_h(s^t_h, a^t_h) + \delta^t_{h+1} - \zeta^t_{h+1} + \xi^t_h \\
        &\leq \left(1 + \frac{1}{H} \right) \sum_{t=1}^T \frac{1}{e_{k^t_h}}\sum_{i=1}^{e_{k^t_h}} \zeta^{\ell^i_{k^t_h}}_{h+1} + \sum_{t=1}^T \delta^t_{h+1} - \sum_{t=1}^T \zeta^t_{h+1} + \sum_{t=1}^T \xi^t_h \\
        &+ \sum_{t=1}^T \clip\left[ H \ind\{He_{k^t_h}/2 > k^t_h + H^2 \}  + \cB^t_h(k^t_h) \big| \frac{\gap_h(B^t_h)}{H+1} \right]
    \end{align*}
    where $k^t_h = k^t_h(B^t_h)$. Repeating argument of Lemma~\ref{lem:regret_decomp} and \cite{zhang2020advantage}
    \[
        \left(1 + \frac{1}{H} \right) \sum_{t=1}^T \frac{1}{e_{k^t_h}}\sum_{i=1}^{e_{k^t_h}} \zeta^{\ell^i_{k^t_h}}_{h+1} \leq  \left(1 + \frac{1}{H} \right)^2 \sum_{t=1}^T \zeta^t_{h+1} \leq \left(1 + \frac{3}{H} \right) \sum_{t=1}^T \zeta^t_{h+1}.
    \]
    Using an upper bound $\zeta^t_h \leq \delta^t_h$ we have for any $h \geq 1$
    \begin{align*}
        \sum_{t=1}^T \delta^t_h &\leq \left(1 + \frac{3}{H} \right) \sum_{t=1}^T \delta^t_{h+1}   + \sum_{t=1}^T \xi^t_h \\
        &+ \sum_{t=1}^T \clip\left[ H \ind\{He_{k^t_h}/2 \leq k^t_h + H^2 \}  + \cB^t_h \big| \frac{\gap_h(B^t_h)}{H+1} \right],
    \end{align*}
    and, rolling out starting with $h=1$ we have the following regret decomposition
    \begin{align*}
        \regret^T &\leq \rme^3 \sum_{t=1}^T \sum_{h=1}^H H \ind\{He_{k^t_h}/2 \leq k^t_h + H^2 \} &=\termA\\
        &+ \rme^3 \sum_{t=1}^T \sum_{h=1}^H \clip\left[ \cB^t_h \big| \frac{\gap_h(B^t_h)}{H+1} \right] & =\termB \\
        &+ \sum_{t=1}^T \sum_{h=1}^H (1 + 3/H)^{H-h} \xi^t_h. &=\termC
    \end{align*}

    \paragraph{Term $\termA$} 

    For this term we notice that for any fixed $h$ the following event
    \[
        H e_{k^t_h} \leq 2(k^t_h + H^2) \iff H \lfloor H \left(1 + 1/H \right)^{k^t_h} \rfloor \leq 2(k^t_h + H^2),
    \]
    that is guaranteed if
    \[
        \left(1 + 1/H \right)^{k^t_h} \leq 2 T + 3 \iff k^t_h \log(1+1/H) \leq \log(2T/H^2 + 3).
    \]
    Thus, indicator can be equal to $1$ no more than $H \log(2T+3)$ times for any $t \in [T]$. As a result,
    \[
        \termA \leq \rme^2 H^3 \log(2T+3).
    \]

    \paragraph{Term $\termB$}
    Let us rewrite this term using a definition of clipping operator and use the definition of near-optimal set (see Definition~\ref{def:near_optimal_set})
    \begin{align*}
      \termB &= \rme^3 \sum_{t=1}^T \sum_{h=1}^H \cB^t_h \ind\left\{  (H+1) \cB^t_h \geq \gap_h(B^t_h) \right\} \leq \rme^3 \sum_{t=1}^T \sum_{h=1}^H \cB^t_h  \ind\{\cent(B^t_h) \in Z^{ \cB^t_h}_h \}.
    \end{align*}

    Next we consider the summation for a fixed $h$. Here we follow Theorem F.3 by \cite{sinclair2022adaptive} and obtain
    \begin{align*}
        \sum_{t=1}^T \cB^t_h \ind\{\cent(B^t_h) \in Z^{ \cB^t_h }_h \} = \sum_{r} \sum_{B: \diam(B) = r} \sum_{t: B^t_h = B} \cB^t_h \ind\{\cent(B) \in Z^{ \cB^t_h}_h \},
    \end{align*}
    where we applied an additional rescaling by a function $\rho$ defined in \eqref{eq:bonus_value_adaptive_ub}.

    Next we fix a constant $r_0 > 0$ and break a summation into two parts: $r \geq r_0$ and $r \leq r_0$.

    \begin{enumerate}
        \item Case $r \leq r_0$. In this situation we have can apply \eqref{eq:bonus_value_adaptive_ub}
        \begin{align*}
            \sum_{r \leq r_0} \sum_{B: \diam(B) = r} &\sum_{t: B^t_h = B} \cB^t_h\ind\{\cent(B) \in Z^{ \cB^t_h}_h \}\\
            &= \cO\left( T r_0 \rho(H, \delta, L)\right).
        \end{align*}
        \item Case $r \geq r_0$. In this situation we also apply \eqref{eq:bonus_value_adaptive_ub} under the indicator function
        \begin{align*}
            \sum_{r \geq r_0} \sum_{B: \diam(B) = r} &\sum_{t: B^t_h = B} \cB^t_h \ind\{\cent(B) \in Z^{\cB^t_h}_h \}\\
            \leq 
            \sum_{r \geq r_0} \sum_{B: \diam(B) = r \cdot  \rho(H, \delta, L)} &\ind\{\cent(B) \in Z^{ \rho(H,\delta,L) \cdot r }_h\} \sum_{t: B^t_h = B} \cB^t_h.
        \end{align*}
        To upper bound the last sum we repeat the argument of \eqref{eq:upper_bound_sum_ek} and apply \eqref{eq:diam_through_n}, using the fact that $\diam(B) = r \cdot  \rho(H, \delta, L)$
        \begin{align*}
            \sum_{t: B^t_h = B} \frac{1}{\sqrt{e_k}} &\leq \sum_{k=0}^{k^T_h(B)} \frac{e_{k+1}}{\sqrt{e_k}} \leq 4H \sqrt{e^{k^T_h(B) + 1} } \\
            &\leq 4 \sqrt{H (n^{T+1}_h(B) + 1)} \leq 4 \sqrt{2H} \cdot \frac{d_{\max}}{\diam(B)} = \frac{\sqrt{32H} \cdot d_{\max}}{r}.
        \end{align*}
        As a result, we have by \eqref{eq:bonus_value_adaptive}
        \[
            \sum_{t: B^t_h = B} \cB^t_h \leq \frac{\sqrt{32H} \cdot d_{\max}}{r} \cdot \left( 2522 \rme^2 H (\beta^{\max}(\delta, T))^4 + 5L {d_{\max}} / \sqrt{H} \right)
        \]
        and
        \begin{align*}
            \sum_{r \geq r_0} \sum_{B: \diam(B) = r} &\sum_{t: B^t_h = B} \cB^t_h \ind\{\cent(B) \in Z^{\cB^t_h}_h \} \\
            &= \cO\left(\sum_{r \geq r_0} N_{r}(Z^{\rho(H, \delta,L) \cdot r}_h) \cdot \frac{H^{3/2} d_{\max}(\beta^{\max}(\delta, T))^4 + L d^2_{\max}}{r} \right). 
        \end{align*}
        Finally, by an arbitrary choice of $r_0$ and a definition of zooming dimension with a scaling $\rho = \rho(H,\delta,L)$ (Definition~\ref{def:zooming_dim})
        \[
            \termB = \cO\left( (H^{3/2} d_{\max} (\beta^{\max}(\delta, T))^4 + L d^2_{\max}) \cdot \sum_{h=1}^H  \inf_{\tilde r_0} \left\{ T r_0 + \sum_{r \geq r_0}  \frac{C_{N,h}}{\tilde r^{d_{z,h} + 1}}  \right\}\right).
        \]
        
    \end{enumerate}

    \paragraph{Term $\termC$}

    For this term we just apply definition of the main event $\cG(\delta) \supseteq \cE(\delta)$ and obtain
    \[
        \termC = \cO\left( \sqrt{H^3 T \beta^{\max}(\delta, T)}\right).
    \]

    \paragraph{Final regret bound}    
    First, we notice that $\beta^{\max(\delta,T)} = \tcO\left( d_c \right)$, therefore we have
    \[
        \regret^T = \tcO\left( H^3 d_c + (H^{3/2} d_c^4 + L) \sum_{h=1}^H \inf_{r_0 > 0} \left\{ T r_0  + \sum_{r \geq r_0}  \frac{C_{N,h}}{r^{d_{z,h} + 1}} \right\} + \sqrt{H^3 T d_c}   \right).
    \]
    Taking $r_0 = K^{-d_{z,h} + 1/2}$ for each $h$ and summing the geometric series we conclude the statement.
    
\end{proof}

\newpage
% Technical part
\newpage
\section{Deviation and Anti-Concentration Inequalities}
\label{app:deviation_ineq}

\subsection{Deviation inequality for $\Kinf$}

For a measure $\nu \in \Pens([0,b])$ supported on a segment $[0,b]$ (equipped with a Borel $\sigma$-algebra) and a number $\mu \in [0,b]$ we recall the definition of the minimum Kullback-Leibler divergence
 \begin{equation*}
    \Kinf(\nu, \mu) \triangleq \inf\left\{  \KL(\nu,\eta): \eta \in \Pens([0,b]), \nu \ll \eta, \E_{X \sim \eta}[X] \geq \mu \right\}\,.
 \end{equation*}

As the Kullback-Leibler divergence this quantity admits a variational formula.
\begin{lemma}[Lemma 18 by \citealp{garivier2018kl}]
\label{lem:var_form_Kinf} For all $\nu \in \Pens([0,b])$, $u\in [0,b)$,
\[
\Kinf(\nu,u) = \max_{\lambda \in[0,1]} \E_{X\sim \nu}\left[ \log\left( 1-\lambda \frac{X-u}{b-u}\right)\right]\,,
 \]
 moreover if we denote by $\lambda^\star$ the value at which the above maximum is reached, then
 \[
   \E_{X\sim \nu} \left[\frac{1}{1-\lambda^\star\frac{X-u}{b-u}}\right] \leq 1\,.
 \]
\end{lemma}
\begin{remark} Contrary to \citet{garivier2018kl} we allow that $u=0$ but in this case Lemma~\ref{lem:var_form_Kinf} is trivially true, indeed
  \[
  \Kinf(\nu, 0) =  0  = \max_{\lambda \in[0,1]} \E_{X\sim \nu}\left[ \log\left( 1-\lambda \frac{X}{b}\right)\right]\,.
   \]
\end{remark}

 Let $(X_t)_{t\in\N^\star}$ be i.i.d.\,samples from a measure $\nu$ supported on $[0,b]$. We denote by $\hnu_n \in \Pens([0,b])$ the empirical measure $\hnu_n = \sum_{i=1}^n \delta_{X_i}$, where $\delta_{X_i}$ is a Dirac measure on $X_i \in [0,b]$.

We are now ready to state the deviation inequality for the $\Kinf$ by \citet{tiapkin2022dirichlet} which is a self-normalized version of Proposition~13 by \citet{garivier2018kl}. Notice that this inequality is stated in terms of slightly less general definition of $\Kinf$, however, the proof remains completely the same.
 \begin{theorem} \label{th:max_ineq_kinf}
 For all $\nu \in \Pens([0,b])$ and for all $\delta\in[0,1]$,
 \begin{align*}
     \P\big(\exists n\in \N^\star,\, n\Kinf(\hnu_n, \E_{X \sim \nu}[X]) > \log(1/\delta) + 3\log(e\pi(1+2n))\big)\leq \delta.
 \end{align*}
\end{theorem}
%\dan{If time remains, rewrite proof here.}

\subsection{Anti-concentration Inequality for Dirichlet Weighted Sums}

In this section we state anti-concentration inequality by \cite{tiapkin2022optimistic} in terms of slightly different definition of $\Kinf$.

\begin{equation}\label{eq:c0_eps}
    c_0(\varepsilon) = \left(\frac{4}{\sqrt{\log(17/16)}} + 8 + \frac{49 \cdot 4 \sqrt{6}}{9} \right)^2 \frac{2}{\pi \cdot \varepsilon^2} + \log_{17/16}\left( \frac{5}{32 \cdot \varepsilon^2}\right).
\end{equation}
\begin{theorem}[Lower bound]\label{thm:lower_bound_dbc} 
    For any $\alpha = (\alpha_0+1, \alpha_1, \ldots, \alpha_m) \in \R_{++}^{m+1}$ define  $\up \in \simplex_{m}$ such that $\up(\ell) = \alpha_\ell/\ualpha, \ell = 0, \ldots, m$, where $\ualpha = \sum_{j=0}^m \alpha_j$. Let $\varepsilon \in (0,1)$. Assume that $\alpha_0 \geq c_0(\varepsilon) + \log_{17/16}(\ualpha)$ for $c_0(\varepsilon)$ defined in \eqref{eq:c0_eps}, and $\ualpha \geq 2\alpha_0$. Then for any $f \colon \{0,\ldots,m\} \to [0,\ub]$ such that $f(0) = \ub$, $f(j) \leq b < \ub/2, j \in \{1,\ldots,m\}$ and $\mu \in (\up f,  \ub)$ 
    \[
        \P_{w \sim \Dir(\alpha)}[wf \geq \mu] \geq (1 - \varepsilon)\P_{g \sim \cN(0,1)}\left[g \geq \sqrt{2 \ualpha \Kinf\left( \sum_{i=0}^m \up(i) \cdot \delta_{f(i)}, \mu\right)}\right].
    \]
\end{theorem}

%\dan{If time remains, write the comparison between this version and what is really stated in \cite{tiapkin2022optimistic}.}

Next we formulate a simple corollary of Theorem~\ref{thm:lower_bound_dbc}, that slightly relaxes assumptions of this theorem under assumption $\mu < b \leq \ub/2$.

\begin{lemma}\label{lem:lower_bound_dbc_relaxed}
    For any $\alpha = (\alpha_0+1, \alpha_1, \ldots, \alpha_m) \in \R_{++}^{m+1}$ define  $\up \in \simplex_{m}$ such that $\up(\ell) = \alpha_\ell/\ualpha, \ell = 0, \ldots, m$, where $\ualpha = \sum_{j=0}^m \alpha_j$. Also define a measure $\bnu = \sum_{i=0}^m \up(i) \cdot \delta_{f(i)}$.
    
    Let $\varepsilon \in (0,1)$. Assume that $\alpha_0 \geq c_0(\varepsilon) + \log_{17/16}(2(\ualpha - \alpha_0))$ for $c_0(\varepsilon)$ defined in \eqref{eq:c0_eps}. The for any $f \colon \{0, \ldots, m\} \to [0,\ub]$ such that $f(0) = \ub, f(j) \leq b \leq \ub/2, j \in [m]$, and any $\mu \in (0, b)$
    \[
        \P_{w \sim \Dir(\alpha)}[wf \geq \mu] \geq (1 - \varepsilon)\P_{g \sim \cN(0,1)}\left[g \geq \sqrt{2 \ualpha \Kinf\left( \bnu, \mu\right)}\right].
    \]
 \end{lemma}
 \begin{proof}
    Assume that assumption $\ualpha \geq 2\alpha_0$ holds.
 
     Then we show that the Theorem~\ref{thm:lower_bound_dbc} also holds for $\mu \leq \up f$. First, we notice that for any $\gamma > 0$
     \[
        \P_{w \sim \Dir(\alpha)}[wf \geq \mu] \geq \P_{w \sim \Dir(\alpha)}[wf \geq \up f + \gamma ] \geq  (1 - \varepsilon)\P_{g \sim \cN(0,1)}\left[g \geq \sqrt{2 \ualpha \Kinf\left( \bnu, \up f + \gamma\right)}\right].
     \]
     By continuity of $\Kinf$ in its second argument (see Theorem~7 by \cite{honda2010asymptotically}) we can tend $\gamma$ to zero, and then use an equality $\Kinf\left( \bnu, \up f\right) = \Kinf\left( \bnu, \mu \right) = 0$.

     Next, assume $\ualpha \leq 2\alpha_0$. In this case we have $\up f \geq b$, thus for any $0 \leq \mu \leq b$
    \begin{align*}
        \P_{w \sim \Dir(\alpha)}\left[ wf \geq \mu \right] &\geq \P_{\xi \sim \Beta(\alpha_0+1, \ualpha - \alpha_0)}\left[ \ub \xi \geq \mu \right] \geq \P_{\xi \sim \Beta(\alpha_0+1, \ualpha - \alpha_0)}\left[\xi \geq \frac{1}{2} \right],
    \end{align*}
    where we first apply a lower bound $f(j) \geq 0$  for all $j > 0$ and $f(0) = \ub$, and second apply a bound $\mu \leq \ub/2$. Here we may apply the result of~\citet[Theorem 1.2'']{alfers1984normal} and obtain the following lower bound 
    \[
        \P_{w \sim \Dir(\alpha)}\left[wf \geq \mu\right] \geq \Phi\left(-\mathrm{sign}(\alpha_0/\ualpha - 1/2) \cdot \sqrt{2 \ualpha \kl(\alpha_0/\ualpha, 1/2)}\right) \geq (1-\varepsilon)\P_{g \sim \cN(0,1)}\left[ g \geq 0\right]
    \]
    where we used $\alpha_0 / \ualpha > 1/2$.
 \end{proof}

\subsection{Rosenthal-type inequality}

In this section we state Rosenthal-type inequality for martingale differences by \cite[Theorem 4.1]{pinelis1994optimum}. The exact constants could be derived from the proof.

\begin{theorem}\label{th:rosenthal}
    Let $X_1,\ldots,X_n$ be a martingale-difference sequence adapted to a filtration $\{\cF_i\}_{i=1,\ldots,n}$: $\E[X_i | \cF_i] = 0$. Define $\cV_i = \E[X_i^2 | \cF_{i-1}]$. Then for any $p \geq 2$ the following holds
    \[
        \E^{1/p}\left[ \left| \sum_{i=1}^n X_i \right|^p \right] \leq C_1 p^{1/2}  \E^{1/p}\left[ \left| \sum_{i=1}^n \cV_i \right|^{p/2}  \right] + 2C_2 p  \E^{1/p}\left[\max_{i\in[n]} \left| X_i \right|^p\right],
    \]
    where $C_1 = 60\rme, C_2 = 60$.
\end{theorem}

Additionally, we need some additional lemma to use this inequality in our setting.
\begin{definition}
    A random variable $X$ is called \textit{sub-exponential} with parameters $(\sigma^2, b)$ if the following tail condition holds for any $t > 0$
    \[
        \P[\vert X - \E[X] \vert \geq t ] \leq 2\exp\left( - \frac{t^2}{2\sigma^2 + 2bt} \right).
    \]
\end{definition}
By Theorem~1 of \cite{skorski2023bernstein} we have for any $\xi \in B(\alpha, \beta)$ with $\beta \geq \alpha$ and any $t > 0$
    \[
        \P\left[ \vert \xi - \E[\xi] \vert \geq t \right] \leq 2\exp\left( - \frac{t^2}{2(v + ct/3)} \right),
    \]
    where 
    \[
        v =  \frac{\alpha \beta}{(\alpha + \beta)^2 (\alpha + \beta + 1)} \leq \frac{\alpha}{(\alpha + \beta)^2}, \quad c =  \frac{2(\beta - \alpha)}{(\alpha + \beta)(\alpha + \beta + 2)} \leq \frac{2}{\alpha + \beta},
    \]
so $\xi$ is $(\alpha/(\alpha + \beta)^2, 2/(3(\alpha + \beta)))$ sub-exponential.

\begin{lemma}\label{lm:max_subexp}
    Let $X_1,\ldots,X_n$ be a sequence of centred $(\sigma^2,b)$ sub-exponential random variables, not necessarily independent. Then for any $p \geq 2$
    \[ 
        \E\left[ \max_{\ell\in[n]} \left| X_{\ell} \right|^p\right] \leq \max\{\sqrt{8\sigma^2 \log n}, 8b\log n\}^p + \rme (2\sigma)^p p^{p/2} + 2\rme (8b)^p p^p.
    \]
\end{lemma}
\begin{proof}
    By Fubini theorem we have for any $\eta \geq 0$: $\E[\eta^p] = p \int_0^\infty u^{p-1} \P[\eta \geq u] \rmd u$, thus for any $a > 0$ the following holds
    \begin{align*}
        \E\left[\max_{\ell \in [n]} \vert X_{\ell}\vert^p \right] &= p \int_0^\infty u^{p-1} \P\left[\max_{\ell \in [n]} \vert X_{\ell} - \E[X_{\ell}] \vert \geq u\right] \rmd u \\
        &\leq  a^{p} + p \int_a^\infty u^{p-1} \P\left[\exists \ell \in [n] : \vert X_{\ell}\vert \geq u\right] \rmd u \\
        &\leq  a^{p} + 2p \int_a^\infty u^{p-1} n \exp\left( - \frac{u^2}{2(\sigma^2 + bu)} \right) \rmd u.
    \end{align*}
    By selecting $a = \max\{\sqrt{8\sigma^2 \log n}, 8b\log n\}$ we have 
    \[
        n \exp\left( - \frac{u^2}{2(\sigma^2 + bu)} \right) \leq \exp\left( - \frac{u^2}{4(\sigma^2 + bu)} \right) \leq \exp\left( - \frac{u^2}{8\sigma^2} \right) + \exp\left( - \frac{u}{8b} \right)
    \] 
    for any $u \geq a$, thus
    \begin{align*}
        \E\left[\max_{\ell \in [n]} \vert X_{\ell} \vert^p \right] &\leq  \max\{\sqrt{8\sigma^2 \log n}, 8b\log n\}^p \\
        &+ 2 p \int_a^\infty u^{p-1} \exp\left( - \frac{u^2}{8\sigma^2} \right) \rmd u + 2 p \int_a^\infty u^{p-1} \exp\left( - \frac{u}{8b} \right) \rmd u \\
        &\leq \max\{\sqrt{8\sigma^2 \log n}, 8b\log n\}^p + p (2\sqrt{2}\sigma)^{p} \Gamma(p/2) + 2p (8b)^p \Gamma(p).
    \end{align*}
    By the bounds on Gamma-function we have 
    \[
        p\Gamma(p/2) = \Gamma(p/2+1) \leq (p+1)^{(p+1)/2} 2^{-(p+1)/2} \rme^{1-p/2} \leq \rme p^{p/2} 2^{-p/2}
    \]and $p \Gamma(p) = \Gamma(p+1) \leq (p+1/2)^{p+1/2} \rme^{1-p} \leq \rme p^p$ (see \cite{guo2007necessary}), thus
    \[
        \E\left[\max_{\ell \in [n]} \vert X_{\ell} \vert^p \right] \leq \max\{\sqrt{8\sigma^2 \log n}, 8b\log n\}^p + \rme (2\sigma)^p p^{p/2} + 2\rme (8b)^p p^p.
    \]
\end{proof}

\begin{proposition}\label{prop:beta_martingale_bound}
    Let $(Y_t, w_t)_{t=1,\ldots,n}$ be a sequence of random variables, where $w_t$ are $\Beta(1/\kappa, (t+t_0)/\kappa)$, and let $(\cF_t)_{i=t,\ldots,n}$ be a filtration such that 1) $Y_t$ are $\cF_{t-1}$-measurable (i.e., predictable), 2) $w_t$ is $\cF_{t}$-measurable (i.e., adapted to $\cF_t$), 3) $\E[w_t | \cF_{t-1}] = \E[w_{t}]$, 4) $Y_t \in [0,1]$ almost surely.

    Then, consider the following two sequences:
    \[
        X_{t} \triangleq (1 - w_t) X_{t-1} + w_t \cdot Y_t, \qquad \bar{X}_{t} \triangleq (1 - \bar{w}_t) \bar{X}_{t-1} + \bar{w}_t \cdot Y_t\,,
    \]
    where $\bar{w}_t = \E[w_t]$ and $X_0 \equiv \bar{X}_0 \equiv 1$. Then, with probability at least $1-\delta$, the following holds
    \[
        |X_n - \bar{X}_n| \leq 60 \rme^2  \sqrt{\frac{\kappa \log(1/\delta)}{n+t_0}} + 1200\rme \frac{ \kappa \log(n) \log^2(1/\delta)}{n+t_0}\,.
    \]
\end{proposition}
\begin{proof}
    First, we notice that
    \begin{align*}
        X_{t} - \bar{X}_t &= (1 - w_{t}) X_{t-1} + w_{t} Y_{t} - (1-\bar{w}_t) \bar{X}_{t-1} + \bar{w}_t Y_t \\
        &= (1 - \bar{w}_t) (X_t - \bar{X}_t)  + (w_t - \bar{w}_t) (Y_t - X_{t-1})\,,
    \end{align*}
    wherefore we have $\E[X_{t} - \bar{X}_t|\cF_{t-1}] = (1 - \bar{w_t}) (X_{t-1} - \bar{X}_{t-1})$. Notice that it is \emph{not} a martingale, but instead we can consider the following martingale
    \[
        Z_t \triangleq \frac{X_t - \bar{X}_t}{P_t}, \qquad P_t \triangleq \prod_{j=1}^t (1 - \bar{w}_t)\,.
    \]
    It is easy to check that $Z_t$ is a martingale: $\E[Z_t] = \E[(X_t - \bar{X}_t)/P_t | \cF_{t-1}] = (1 - \bar{w}_t) (X_t - \bar{X}_t) / P_t = (X_t - \bar{X}_t) / P_{t-1}$. Thus, defining $\Delta_t = Z_t - Z_{t-1}$, we apply Theorem~\ref{th:rosenthal}:
    \begin{equation}\label{eq:rosenthal_beta_recursion}
        \E^{1/p}\left[ \left| Z_n \right|^p \right] \leq 60\rme \sqrt{p} \cdot \E^{1/p}\left[ \left| \sum_{t=1}^n \cV_t \right|^{p/2} \right] + 120p \cdot \E^{1/p}\left[\max_{t\in[n]} \vert \Delta_t \vert^p\right]\,,
    \end{equation}
    where $\cV_t = \E[ (\Delta_t)^2 | \cF_{t-1}]$. Next, we compute $\Delta_t$ and $P_t$ as follows
    \[
        \Delta_t = \frac{X_t - \bar{X}_t}{P_t} - \frac{X_{t-1} - \bar{X}_{t-1}}{P_{t-1}} = \frac{(w_t - \bar{w_t})(Y_t - \bar{X}_{t-1}) }{P_{t}}\,,\quad P_t = \prod_{j=1}^t \left( 1 - \frac{1}{j + t_0}\right) = \frac{t_0}{t+t_0}\,.
    \]
    Next, we are going to bound all the terms in the \eqref{eq:rosenthal_beta_recursion}. We start from the variance term by noting that
    \[
        \cV_t = \frac{(Y_t - \bar{X}_{t-1})^2}{P_t^2} \Var(w_t) \leq \frac{(t+t_0)^2}{t_0^2} \cdot \frac{\kappa}{(t+t_0 + 1)^2} \leq \frac{\kappa}{t_0^2}\,,
    \]
    and, as a result, 
    \[
        \E^{1/p}\left[ \left| \sum_{t=1}^n \cV_t \right|^{p/2} \right] \leq \sqrt{\frac{n \cdot \kappa}{t_0^2}}\,.
    \]
    For the second term, we first upper-bound $\Delta_t$ as 
    $
        | \Delta_t | \leq \frac{|w_t - \bar{w_t}|}{P_t}\,,
    $
    and then we notice that $|w_t - \bar{w_t}|$ is sub-exponential with parameter $(\kappa/(t+t_0)^2, 2\kappa /3(t+t_0))$ and, as a result, $|w_t - \bar{w_t}|/P_t$ are $(\kappa/t_0^2, 2\kappa/(3t_0))$-subexponential for any $t$. Therefore, Lemma~\ref{lm:max_subexp} implies
    \begin{align*}
        \E\left[ \max_{t \in [n]} | \Delta_t |^p \right] &\leq \max\{ \sqrt{8 \kappa / t_0^2 \log(n)}, 16/3 \cdot \kappa/t_0 \cdot \log(n) \}^p \\
        &+ \rme (2\kappa/t_0^2)^p \cdot p^{p/2} + 2\rme (16 \kappa / 3t_0)^p \cdot p^p \leq (20 \cdot \kappa/t_0 p \cdot \log(n))^p\,.
    \end{align*}
    Thus, we have
    \[
        \E^{1/p}[|Z_n|^p] \leq \frac{60 \rme}{t_0} \sqrt{n \cdot \kappa \cdot p} + 1200 \cdot \frac{\kappa \cdot p^2 \cdot \log(n)}{t_0}\,.
    \]
    Next, we plug-in into this inequality $Z_n = (X_n - \bar{X_n})/P_n$ and achieve the following bound
    \[
        \E^{1/p}[|X_n - \bar{X}_n|^p] \leq 60 \rme \sqrt{ \frac{\kappa \cdot p}{n+t_0}} + 1200 \cdot \frac{\kappa \cdot p^2 \cdot \log(n)}{n+t_0}\,.
    \]
    Next we turn from moments to tails. By Markov inequality with $p = \log(1/\delta)$
    \begin{align*}
        \P\left[ \left|X_n - \bar{X}_n \right| \geq t\right] &\leq \left( \frac{\E^{1/p}\left[ \left|X_n - \bar{X}_n \right|^p \right]}{t} \right)^p \\
        &\leq \left( \frac{60 \rme B \sqrt{\frac{\kappa \log(1/\delta)}{n+t_0}} + 1200 \log^2(1/\delta) B\kappa \frac{\log(n)}{n+t_0} }{t} \right)^{\log(1/\delta)}.
    \end{align*}
    Taking $t = 60 \rme^2 B \sqrt{\frac{\kappa \log(1/\delta)}{n+t_0}} + 1200\rme B \frac{ \kappa \log(n) \log^2(1/\delta)}{n+t_0}$ we conclude the statement.
    
\end{proof}

\section{Technical Lemmas}

\begin{lemma}\label{lem:kinf_prior_remove}
    Let $\nu \in \Pens([0,b])$ be a probability measure over the segment $[0,b]$ and let $\bnu = \alpha \delta_{\ub} + (1-\alpha) \cdot \nu$ be a mixture between $\nu$ and a Dirac measure on $\ub > b$. Then for any $\mu \in (0, b)$
    \[
        \Kinf(\bnu, \mu) \leq (1-\alpha) \Kinf(\nu, \mu).
    \]
\end{lemma}
\begin{proof}
    By a variational formula for $\Kinf$ (see Lemma~\ref{lem:var_form_Kinf})
    \[
        \Kinf(\bnu, \mu) = \max_{\lambda \in [0, 1/(\ub-\mu)]} \E_{X \sim \bnu}\left[ \log\left( 1 - \lambda (X-\mu) \right) \right].
    \]
    Since $\bnu$ is a mixture, we have for any $\lambda \in [0, 1/(\ub- \mu)]$
    \[
        \E_{X \sim \bnu}\left[ \log\left( 1 - \lambda (X-\mu) \right) \right] = (1-\alpha) \E_{X \sim \bnu}\left[ \log\left( 1 - \lambda (X-\mu) \right) \right] + \alpha  \log\left( 1 - \lambda (\ub-\mu) \right).
    \]
    Notice that $\max_{\lambda > 0} \log(1-\lambda(\ub-\mu)) = 0$. Thus, maximizing each term separately over $\lambda$, we have
    \begin{align*}
        \Kinf(\bnu, \mu) &\leq (1-\alpha) \max_{\lambda \in [0, 1/(\ub-\mu)]} \E_{X \sim \bnu}\left[ \log\left( 1 - \lambda (X-\mu) \right) \right] \\
        &\leq (1-\alpha) \max_{\lambda \in [0, 1/(b-\mu)]} \E_{X \sim \bnu}\left[ \log\left( 1 - \lambda (X-\mu) \right) \right]  = (1-\alpha) \Kinf(\nu, \mu).
    \end{align*}
\end{proof}

%\newpage
% Experimental part
\section{Experimental details}
\label{app:experiment_details}
In this section, we detail the experiments we conducted for tabular and non-tabular environments. For all experiments, we used 2 CPUs (Intel Xeon CPU $2.20$GHz), and no GPU was used. Each experiment took approximately one hour.

\subsection{Tabular experiments}

In our initial experiment, we investigated a simple grid-world environment.

\paragraph{Environments} For tabular experiments we use two environments.

The first one is a grid-world environment with $100$ states $(i, j) \in [10]\times[10]$ and $4$ actions (left, right, up and down). The horizon is set to $H=50$. When taking an action, the agent moves in the corresponding direction with probability $1-\epsilon$, and moves to a neighbor state at random with probability $\epsilon=0.2$. The agent starts at position $(1, 1)$. The reward equals to $1$ at the state $(10, 10)$ and is zero elsewhere. 

The second one is a chain environment described by \cite{osband2016deep} with $L=15$ states and $2$ actions (left or right). The horizon is equal to $30$, and the probability of moving in the wrong direction is equal to $0.1$. The agent starts in the leftmost state with reward $0.05$, and the largest reward is equal to $1$ in the rightmost state.

\begin{figure}
    
    \centering
    \includegraphics[width=0.6\linewidth]{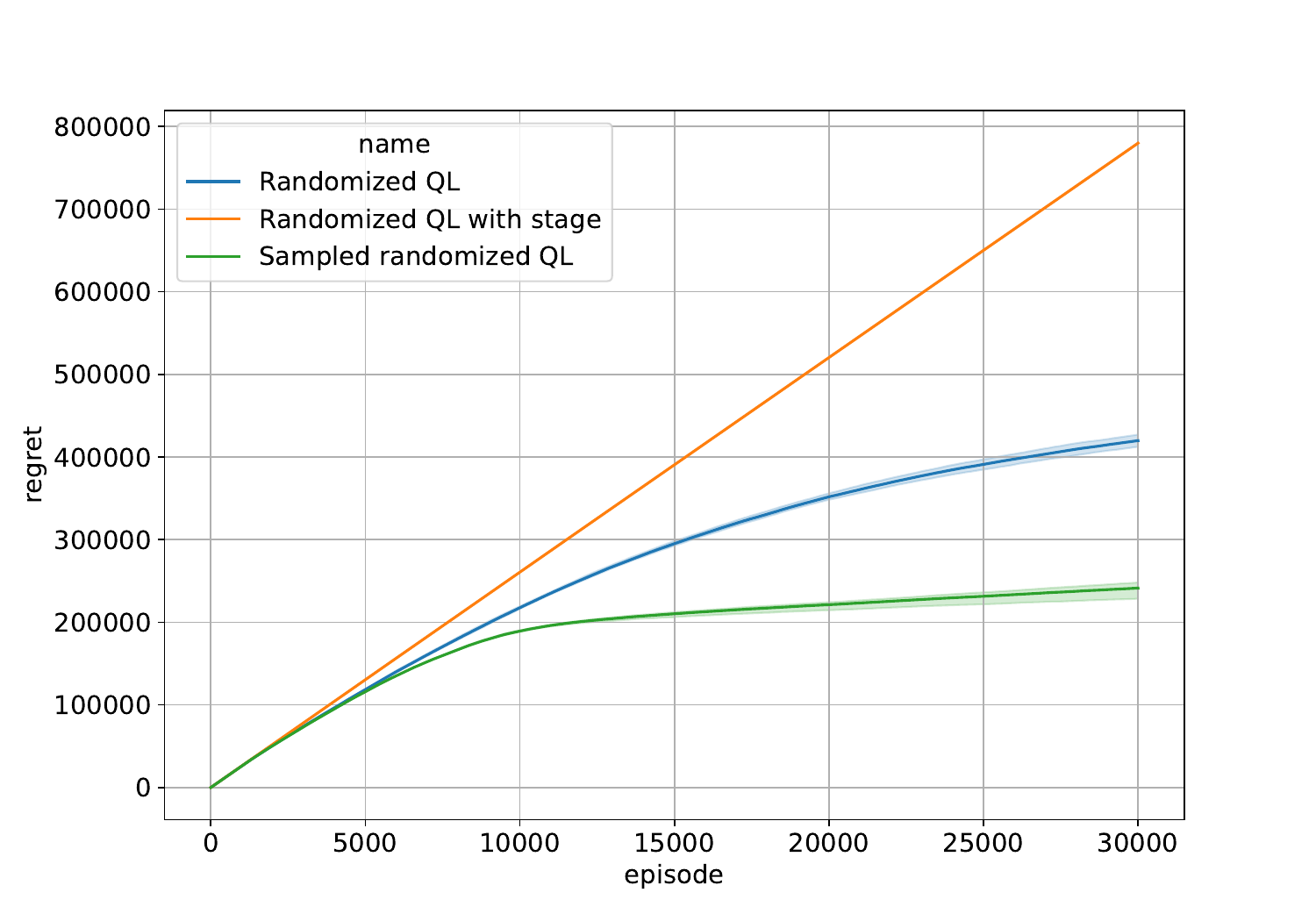}
    \caption{Regret curves of \RandQL, \StagedRandQL and \SampledRandQL on a grid-world environment with $100$ states and $4$ actions for $H = 50$ an transitions noise $0.2$. We show average over 4 seeds.}
    \label{fig:randql_vs_variant}
\end{figure}

\paragraph{Variations of randomized Q-learning} First, we compare the different variations of randomized Q-learning on the grid-world environment. Precisely, we consider:
\begin{itemize}
    \item \RandQL a randomized version of \OptQL, detailed in Appendix~\ref{app:description_randQL}.
    \item \StagedRandQL a staged version of \RandQL, described in Section~\ref{sec:algorithm}.
    \item \SampledRandQL a version of \RandQL which samples one Q-value function in the ensemble to act, described in Appendix~\ref{app:description_randQL}.
\end{itemize}
For these algorithms, we used the same parameters: posterior inflation $\kappa=1.0$, $n_0=1/S$ prior sample (same as \PSRL, see below), ensemble size  $J=10$. We use a similar ensemble size as the one used for the experiments with \OPSRL by \citet{tiapkin2022optimistic}. For \StagedRandQL we use stage of sizes  $\big((1+1/H)^k\big)_{k\geq1}$ without the $H$ factor, in order to have several epochs per state-action pair even for few episodes.

The comparison is presented in Figure~\ref{fig:randql_vs_variant}. We observe that \RandQL and \SampledRandQL behave similarly, with slightly better performance for \SampledRandQL. This is coherent with the experiment on the comparison between \OPSRL and \PSRL \citet{tiapkin2022optimistic} where the optimistic version performs worse than the fully randomized algorithm. We also note that even with the aggressive stage schedule, \StagedRandQL needs more episodes to converge. We conclude that, despite simplifying the analysis, it artificially slows down the learning in practice.

To ease the comparison with the baselines, for the rest of the experiments, we only use \RandQL because of its similarity with \OptQL.

\paragraph{Baselines}
We compare \RandQL algorithm to the following baselines:
\begin{itemize}
    \item \OptQL\citep{jin2018is} a model-free optimistic Q-learning.
    \item \UCBVI \citep{azar2017minimax} a model-based optimistic dynamic programming.
    \item \GreedyUCBVI \citep{efroni2019tight} optimistic real-time dynamic programming.
    \item \PSRL \citep{osband2013more} model-based posterior sampling.
    \item  \RLSVI \citep{russo2019worst} model-based randomized dynamic programming.
\end{itemize}

The selection of parameters can have a significant impact on the empirical regrets of an algorithm. For example, adjusting the multiplicative constants in the bonus of \UCBVI or the scale of the noise in \RLSVI can result in vastly different regrets. To ensure a fair comparison between algorithms, we have made the following parameter choices:
\begin{itemize}

\item For bonus-based algorithms, \UCBVI, \OptQL, we use simplified bonuses from an idealized Hoeffding inequality of the form
\begin{align}
\label{eq:simplified_bonus}
	\beta_h^t(s,a) \triangleq
	\min\left(
	\sqrt{\frac{1}{n_h^t(s,a)}} + \frac{H-h+1}{n_h^t(s,a)}, H-h+1
	\right)\,.
\end{align}
As explained by \citet{menard2021ucb}, this bonus does not necessarily result in a true upper-confidence bound on the optimal Q-value. However, it is a valid upper-confidence bound for $n_h^t(s,a)= 0$, which is important in order to discover new state-action pairs. 
\item For \RLSVI we use the variance of Gaussian noise equal to simplified Hoeffding bonuses described above in \eqref{eq:simplified_bonus}.
\item For \PSRL, we use a Dirichlet prior on the transition probability distribution with parameter $(1/S,\ldots,1/S)$ and for the rewards a Beta prior with parameter $(1,1)$. Note that since the reward $r$ is not necessarily in $\{0,1\}$, we just sample a new randomized reward $r'\sim\Ber(r)$ accordingly to a Bernoulli distribution of parameter $r$, to update the posterior, see \citet{agrawal2013further}. 
\end{itemize}

\begin{figure}
    \centering
    \includegraphics[width=0.45\linewidth]{figures/randql_vs_baseline.pdf}
    \includegraphics[width=0.45\linewidth]{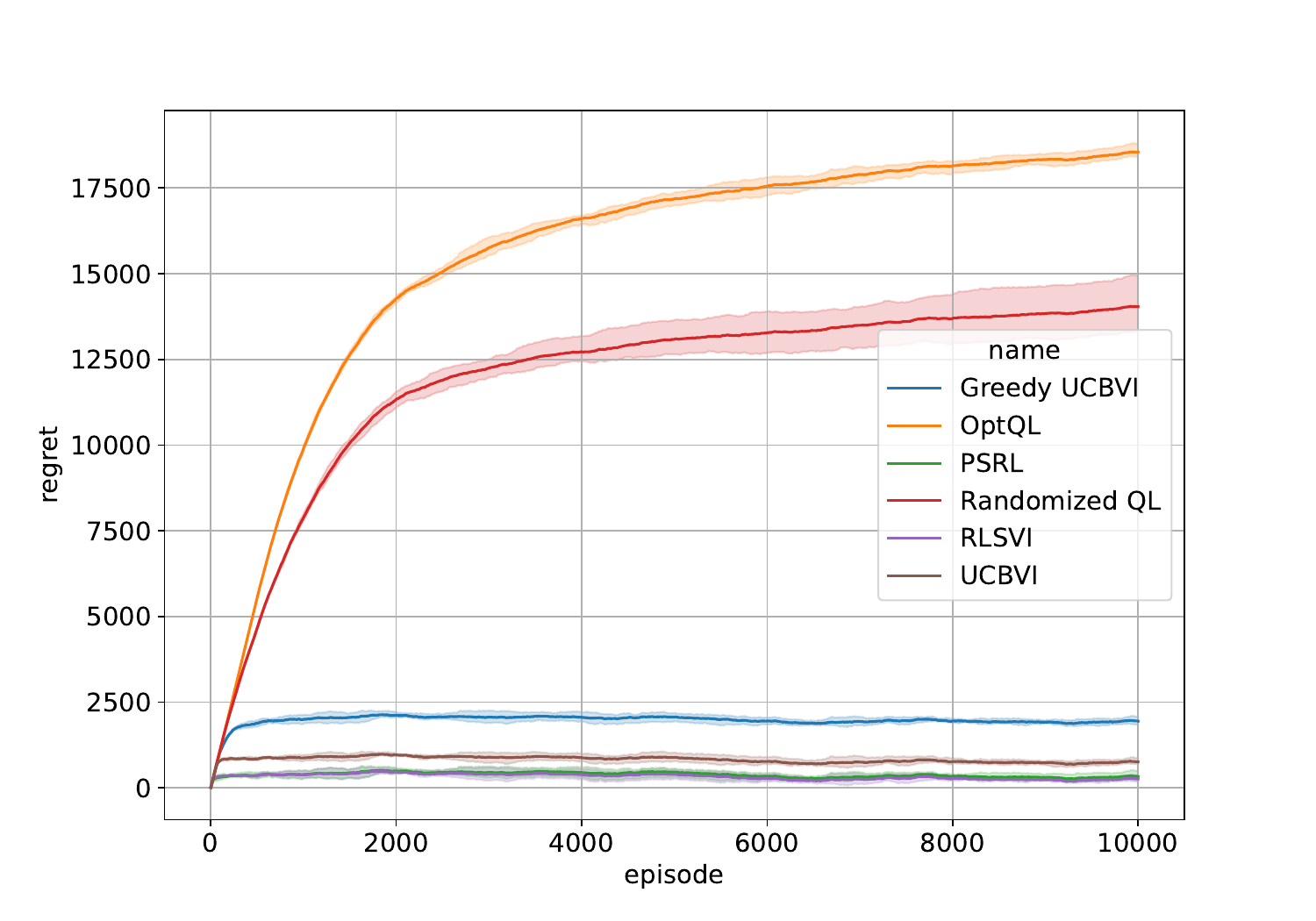}
    \caption{Regret curves of \RandQL and several baselines in (left) a grid-world environment with $100$ states and $4$ actions for $H = 50$ an transitions noise $0.2$, and (right) in a chain environment of length $L=15$, $2$ actions for $H=30$ with transition noise $0.1$: smaller is better. We show average and error bars over 4 seeds.}
    \label{fig:randql_vs_baseline}
\end{figure}

\paragraph{Results} Figure~\ref{fig:randql_vs_baseline} shows the result of the experiments. Overall, we see that \RandQL outperforms \OptQL algorithm on the tabular environment, but still degrades in comparison to model-based approaches, which is usual for model-free algorithms in tabular environments. Indeed, as explained by \citet{menard2021ucb}, using a model and backward induction allows new information to be more quickly propagated. For example, \UCBVI needs only one episode to propagate information about the last step $h=H$ to the first step $h=1$, whereas \OptQL or \RandQL need at least $H$ episodes. But as a counterpart,  \RandQL has a better time-complexity and space-complexity than model-based algorithms, see Table~\ref{tab:time_space_complexity}.

\begin{table}[t]
\setlength{\tabcolsep}{4pt}
\centering
\begin{tabular}{@{}lcl}\toprule
\textbf{Algorithm} & \textbf{Time-complexity (per episode)} & \textbf{Space complexity} \\
\midrule
\UCBVI~{\tiny{\citep{azar2017minimax}}} &  \multirow{3}{*}{$\tcO(HS^2A)$} & \multirow{4}{*}{$\tcO(HS^2A)$}\\
\PSRL~{\tiny{\citep{osband2013more}}} &  & \\
\RLSVI~{\tiny{\citep{russo2019worst}}} & & \\ 
\GreedyUCBVI~{\tiny{\citep{efroni2019tight}}} &  $\tcO(HSA)$ & \\
\OptQL~{\tiny{\citep{jin2018is}}} &  $\tcO(H)$ &  $\tcO(HSA)$\\
\midrule 
 \rowcolor[gray]{.90} \RandQL~{\tiny (this paper)} & $\tcO(H)$  &  $\tcO(HSA)$\\
\bottomrule
\end{tabular}
\caption{Time- and space-complexity of several tabular algorithms.}
\label{tab:time_space_complexity}
\end{table}

\subsection{Non-tabular experiments}

The second experiment was performed on a set of two-dimensional continuous environments \citep{rlberry} with levels of increasing exploration difficulty.

\paragraph{Environment} We use a ball environment with the 2-dimensional unit Euclidean ball as state-space $\cS = \{s\in\R^2, \norm{s}_2\leq 1\}$ and of horizon $H=30$. The action space is a list of 2-dimensional vectors $\cA = \{[0.0, 0.0], [-0.05, 0.0], [0.05, 0.0], [0.0, 0.05], [0.0, -0.05]\}$ that can be associated with the action of staying at the same place, moving left, right, up or down. Given a state $s_h$ and an action $a_h$ the next state is 
\[
s_{h+1} = \mathrm{proj}_{\cS}(s_{h} + a_h + \sigma z_h)
\]
where $z_h\sim \cN([0,0] , I_2)$ is some independent Gaussian noise with zero mean and identity covariance matrix and $\mathrm{proj}_B$ is the Euclidean projection on the unit ball $\cS$. The initial position $s_1 = \sigma_1 z_1$ with $z_1\sim \cN([0,0] , I_2)$ and $\sigma_1=0.001$, is sampled at random from a Gaussian distribution. The reward function is independent of the action and the step 
\[
r_h(s,a) = \max( 0,  1 - \norm{ s - s'}/c )
\]
where $s'=[0.5,0.5]\in\cS$ is the reward center and $c >0$ is some smoothness parameter. We distinguish $3$ levels by increasing exploration difficulty:
\begin{itemize}
    \item Level $1$, dense reward and small noise. The smoothness parameter is $c=0.5\cdot \sqrt{2}\approx 0.71$ and the transition standard deviation is $\sigma = 0.01$.
    \item Level $2$, sparse reward, and small noise.
    The smoothness parameter is $c=0.2$ and the transition standard deviation is $\sigma = 0.01$.
    \item Level $3$, sparse reward, and large noise.
    The smoothness parameter is $c=0.2$ and the transition standard deviation is $\sigma = 0.025$.
\end{itemize}

\paragraph{\RandQL algorithm} Among the different versions of \RandQL for continuous state-action space, see Section~\ref{sec:metric_space}, we pick the \AdaptiveRandQL algorithm, described in Appendix~\ref{app:adaptive_randql_proof}, as it is the closest version to the \AdaptiveQL algorithm. It combines the \RandQL algorithm and adaptive discretization. For \AdaptiveRandQL we used an ensemble of size $J=10 \approx \log(T)$, $\kappa = 10 \approx \log(T)$ and a prior number of samples of $n_0 =  0.33$. Note that we increased the number of prior samples in comparison to the tabular case as explained in Section~\ref{sec:metric_space}.

\paragraph{Baselines} We compare \AdaptiveRandQL algorithm to the following baselines:
\begin{itemize}
    \item  \AdaptiveQL \citep{sinclair2019adaptive,sinclair2022adaptive}, an adaptation of \OptQL algorithm to continuous state-space thanks to adaptive discretization;
    \item  \KernelUCBVI \citep{domingues2021kernel}, a kernel-based version of the \UCBVI algorithm;
    \item \DQN \citep{mnih2013playing}, a deep RL algorithm;
    \item \BootDQN \citep{osband2015bootstrap}, a deep RL algorithm with an additional exploration given by bootstraping several Q-networks;
\end{itemize}
For \AdaptiveQL and \KernelUCBVI baselines, we employ the same simplified bonuses \eqref{eq:simplified_bonus} used for the tabular experiments. For \KernelUCBVI we used a Gaussian kernel of bandwidth $0.025$ and the representative states technique, with $300$ representative states, described by \citet{domingues2021kernel}.

For \DQN and \BootDQN we use as a network a 2-layer multilayer perceptron (MLP) with a hidden layer size equal to $64$. For exploration, \DQN utilizes $\varepsilon$-greedy exploration with coefficient annealing from $1.0$ to $0.1$ during the first $10,000$ steps. For \BootDQN we use an ensemble of $10$ heads and do not use $\varepsilon$-greedy exploration.

\begin{figure}
    \centering
    \includegraphics[width=0.48\linewidth]{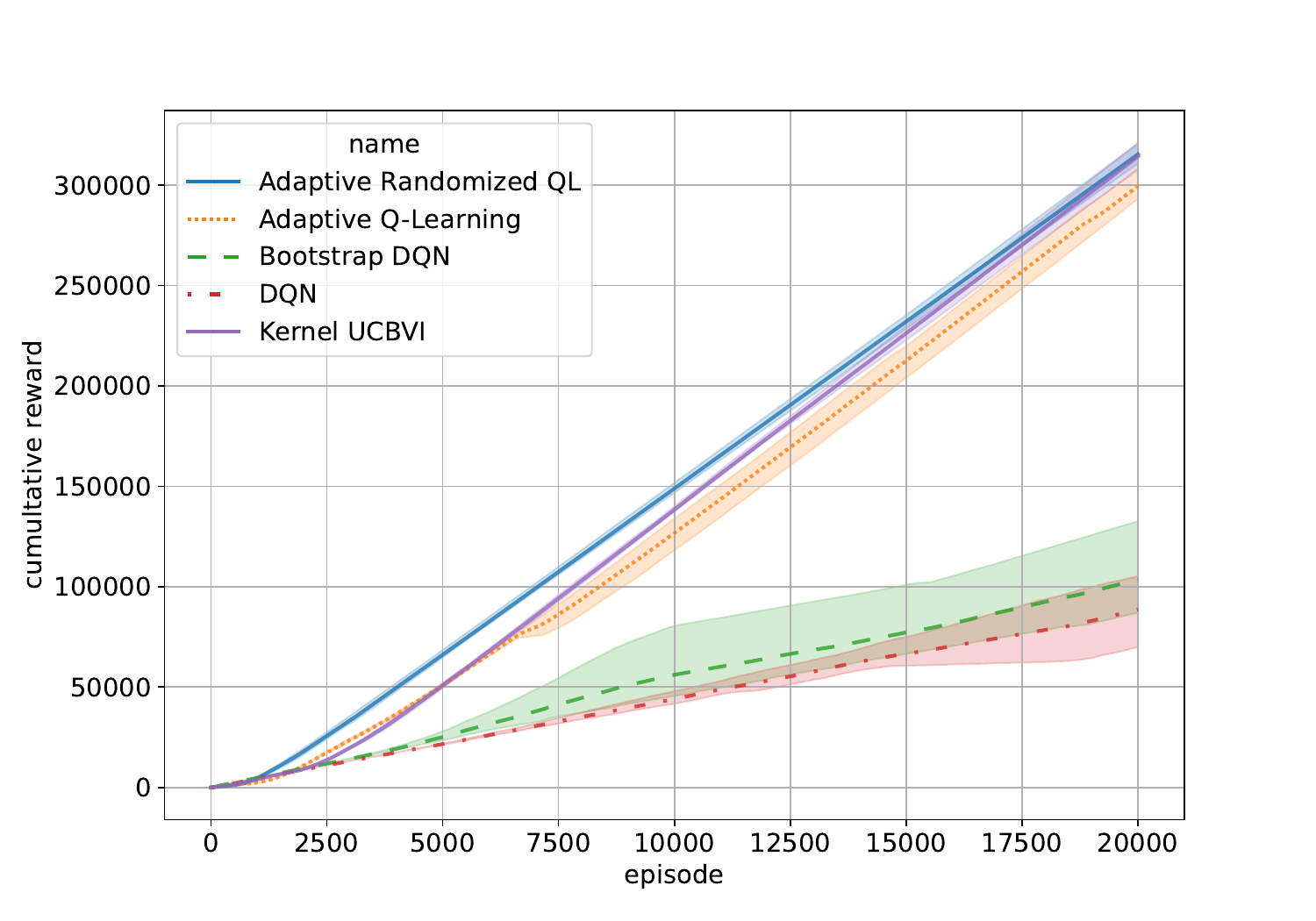}
    \includegraphics[width=0.48\linewidth]{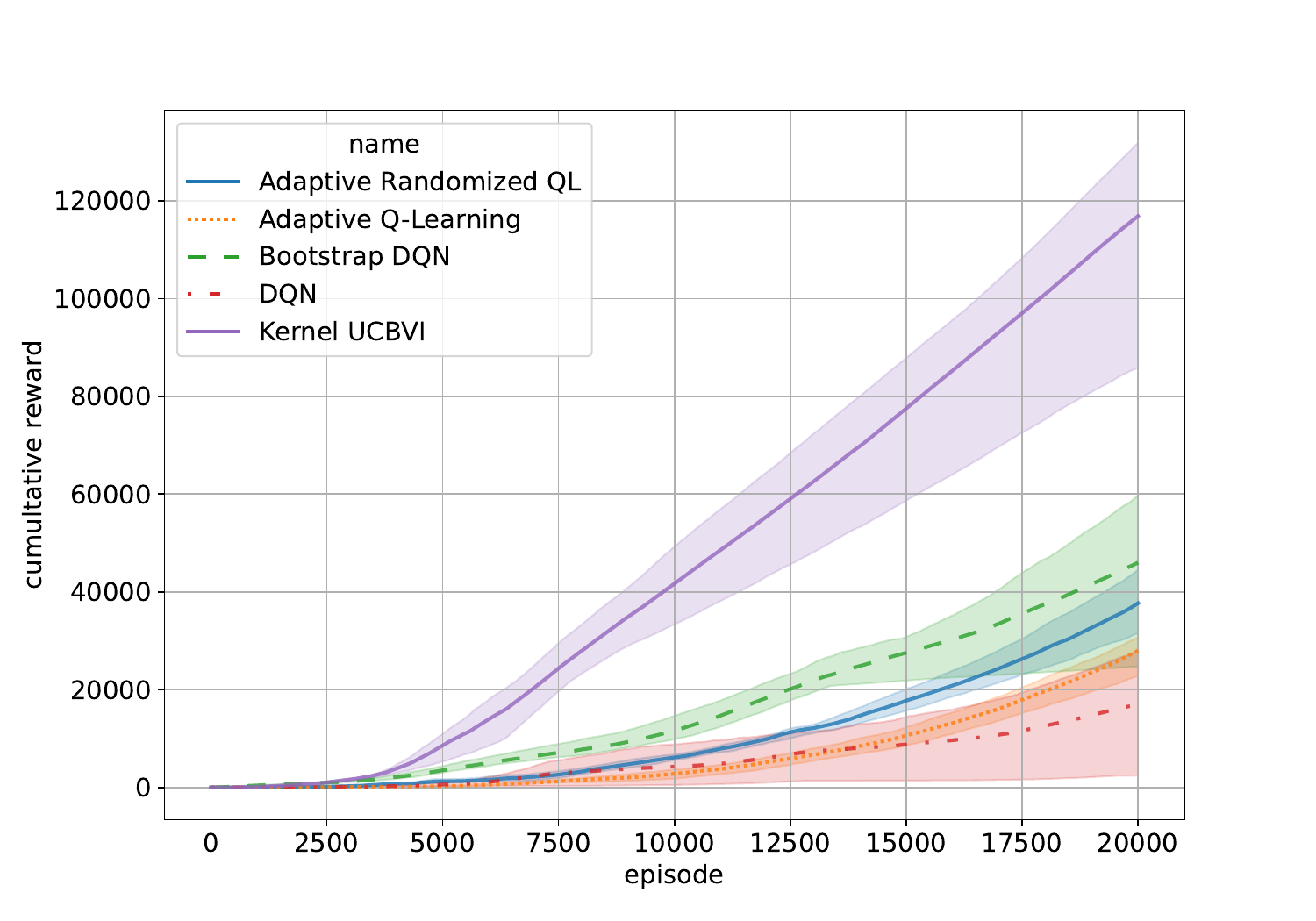}
    \includegraphics[width=0.48\linewidth]{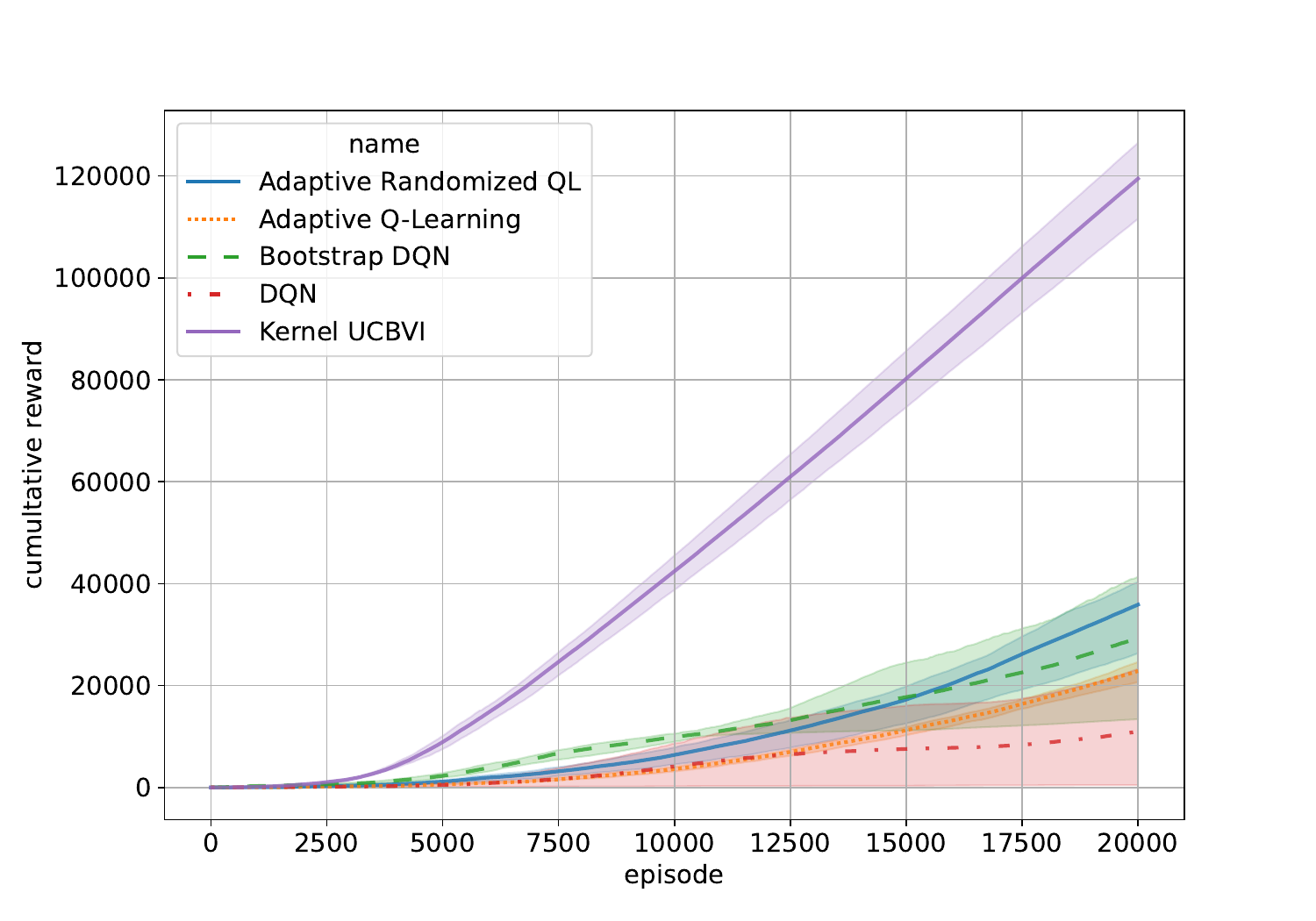}
    \caption{Cumulative rewards (higher is better) of \AdaptiveRandQL and several baselines in ball environments with increasing exploration difficulty: Upper Left displays Level 1, Upper Right shows Level 2, Down shows Level 3.   We show average and error bars over 4 seeds.}
    \label{fig:pball_randql}
\end{figure}

\paragraph{Results} Figure~\ref{fig:pball_randql} shows the results of non-tabular experiments.  Overall, we see that \AdaptiveRandQL outperforms \AdaptiveQL in all environments, especially in the sparse reward setting.  However, we see that the model-based algorithm is much more sample efficient than the model-free algorithm, as it was shown by \cite{domingues2021kernel}. This is connected to a low dimension of the presented environment, where the difference in theoretical regret bounds is not so large. However, this performance come at the price of $3$-times larger time complexity, see Table~\ref{tab:time_complexity}.

Regarding the comparison to neural-network-based algorithms, we see that approaches based on adaptive discretization always outperform \DQN and \BootDQN on an environment with non-sparse rewards. We connect this phenomenon to the fact that neural network algorithms are solving two problems at the same time: exploration and optimization, whereas discretization-based approaches solve only the exploration problem. 

In the setup of sparse rewards, it turns out that neural network-based approaches are competitive with \AdaptiveQL and \AdaptiveRandQL. Notably, \DQN shows itself as the worst one, whereas \AdaptiveRandQL and \BootDQN show similar performance, additionally justifying the exploration effect of ensemble learning and randomized exploration.

\begin{table}[t]
\setlength{\tabcolsep}{4pt}
\centering
\begin{tabular}{@{}lcl}\toprule
\textbf{Algorithm} & \textbf{Episode time (second)}\\
\midrule
\AdaptiveRandQL &  $5.780e^{-02}$ \\
\AdaptiveQL &  $4.213e^{-02}$\\
\KernelUCBVI &  $1.523e^{-01}$\\
\bottomrule
\end{tabular}
\caption{Average time of one episode in second (averaged over $20000$ episodes).}
\label{tab:time_complexity}
\end{table}

\end{document}